\documentclass{article}

\usepackage[preprint]{neurips_2026}
\AtBeginDocument{\newgeometry{
  textwidth=6in,
  textheight=9in,
  top=1in,
  headheight=12pt,
  headsep=25pt,
  footskip=30pt
}}
\makeatletter
\AtBeginDocument{\renewcommand{\@noticestring}{Preprint. Under review.}}
\makeatother

\usepackage[utf8]{inputenc}
\usepackage[T1]{fontenc}
\usepackage{hyperref}
\usepackage{url}
\usepackage{booktabs}
\usepackage{amsmath,amssymb,amsthm}
\usepackage{nicefrac}
\usepackage{microtype}
\usepackage{xcolor}
\definecolor{ourblue}{RGB}{0,102,204}
\definecolor{unpolishedorange}{RGB}{204,102,0}
\usepackage{algorithm}
\usepackage{algpseudocode}
\usepackage{float}
\usepackage{setspace}
\usepackage{mathtools}
\usepackage{enumitem}
\usepackage{graphicx}
\usepackage{multirow}

\newtheoremstyle{prominent}% name
  {10pt}%   space above
  {10pt}%   space below
  {\itshape}% body font
  {}%       indent amount
  {\bfseries}% theorem head font
  {.}%      punctuation after theorem head
  {0.5em}%  space after theorem head
  {}%       theorem head spec (can be left empty, meaning 'normal')

\theoremstyle{prominent}
\newtheorem{claim}{Claim}
\newtheorem{proposition}[claim]{Proposition}
\newtheorem{corollary}{Corollary}[claim]
\newtheorem{lemma}{Lemma}
\newtheorem{theorem}{Theorem}

\theoremstyle{definition}

\theoremstyle{remark}
\newtheorem*{remark}{Remark}

\newcommand{\R}{\mathbb{R}}
\newcommand{\Spp}{\mathbb{S}_{++}}
\newcommand{\E}{\mathbb{E}}
\newcommand{\Tr}{\mathrm{Tr}}
\newcommand{\AM}{\mathrm{AM}}
\newcommand{\GM}{\mathrm{GM}}

\newcommand{\Diag}{\mathrm{Diag}}
\newcommand{\vol}{\mathrm{vec}}
\newcommand{\KL}{D_{\mathrm{KL}}}

\newcommand{\polar}{\mathrm{polar}}
\newcommand{\qr}{\mathrm{qr}}
\newcommand{\St}{\mathrm{St}}
\DeclareMathOperator*{\argmin}{arg\,min}

\newcommand{\Gtil}{\widetilde{G}}

\newcommand{\diag}{\mathrm{diag}}
\newcommand{\Lr}{{L^*_{\mathrm{restr}}}}

\usepackage[most]{tcolorbox}
\newcommand{\aside}[1]{{\footnotesize\color{gray}#1}}
\newcommand{\highlight}[1]{\textcolor{ourblue}{#1}}
\tcbset{
  keyeq/.style={enhanced, colback=white, colframe=ourblue!55!black,
    left=6pt, right=6pt, top=4pt, bottom=4pt,
    boxrule=0.5pt, arc=1.5pt,
    attach boxed title to top left={xshift=0.6cm, yshift=-\tcboxedtitleheight/2},
    boxed title style={size=small, colback=white, opacityback=1,
      opacityframe=0, boxrule=0pt},
    fonttitle=\bfseries\footnotesize,
    coltitle=ourblue!55!black,
    title={#1}
  }
}

\definecolor{editedpink}{RGB}{220,40,120}

\definecolor{srtgreen}{RGB}{0,140,0}

\title{Pro-KLShampoo: Projected KL-Shampoo\\with Whitening Recovered by Orthogonalization}

\author{%
  Ruotong Sun\textsuperscript{$\dagger$}, Ermin Wei\textsuperscript{$*\dagger$}\\[2pt]
  \textsuperscript{$*$}Department of Electrical \& Computer Engineering, Northwestern University\\
  \textsuperscript{$\dagger$}Department of Industrial Engineering \& Management Sciences, Northwestern University\\[2pt]
  \texttt{ruotongsun2030@u.northwestern.edu},\quad \texttt{ermin.wei@northwestern.edu}
}

\begin{document}

\maketitle

% ============================================================================
\begin{abstract}
% ============================================================================
Optimizers that exploit the matrix structure of gradients are central to modern LLM pre-training, with two distinct frontiers: explicit Kronecker-factored preconditioning---most recently KL-Shampoo, which estimates the preconditioner via KL divergence minimization---and orthogonalization of the gradient momentum, exemplified by Muon and analyzed as steepest descent under the spectral norm. The two routes are typically developed in isolation.
We make a structural observation about KL-Shampoo's Kronecker preconditioners: their eigenvalue spectra exhibit a \emph{spike-and-flat} shape---a few dominant eigenvalues followed by an approximately uniform tail---across layers and training stages, holding exactly under a rank-$\rho$ signal-plus-noise gradient model.
We exploit this structure by restricting one of KL-Shampoo's Kronecker factors to a parametric family aligned with the spike-and-flat shape: full spectral structure on a tracked $r$-dimensional subspace, single shared eigenvalue across the remaining $n-r$ directions. On these directions, we apply orthogonalization. An identity shows that this orthogonalization recovers the algebraic form of full KL-Shampoo's preconditioner.
On four pre-training scales (GPT-2 124M / 350M, LLaMA 134M / 450M), Pro-KLShampoo consistently outperforms KL-Shampoo at every subspace rank we test in validation loss, peak per-GPU memory, and wallclock time to reach each loss level.
\end{abstract}
% ============================================================================
\section{Introduction}
\label{sec:intro}
% ============================================================================

Optimizers that exploit the matrix structure of gradients now rival AdamW~\citep{kingma2014adam,loshchilov2017decoupled} for LLM pre-training. The Kronecker-factored preconditioning approach, originating with Shampoo~\citep{gupta2018shampoo}, has seen rapid recent development. SOAP~\citep{vyas2024soap} reduces per-iteration runtime by running Adam in Shampoo's eigenbasis. KL-Shampoo~\citep{lin2025understanding} instead replaces the underlying estimation principle---recasting it as KL divergence minimization---and outperforms SOAP and Shampoo on LLM pretraining tasks. A separate approach, exemplified by Muon~\citep{modded_nanogpt_2024}, orthogonalizes the gradient momentum, realizing steepest descent under the spectral norm~\citep{bernstein2024old}.

In this work, we make a structural observation about KL-Shampoo's Kronecker preconditioners: their eigenvalue spectra exhibit a \emph{spike-and-flat} shape---a few dominant eigenvalues followed by an approximately uniform tail across layers and training stages (Figure~\ref{fig:spike_flat}), echoing low-rank gradient and spiked Hessian phenomena reported in neural network training~\citep{gur2018gradient,ghorbani2019investigation,jaiswal2024low}. Under a rank-$\rho$ signal-plus-noise gradient model, the derived preconditioners in KL-Shampoo exhibit the spike-and-flat structure exactly: at least $n - \rho$ bottom eigenvalues coincide. This motivates restricting one of the Kronecker preconditioners to a parametric family aligned with the spike-and-flat shape.

Concretely, the parametric family takes the following form: full spectral structure on a tracked $r$-dimensional subspace, and a single shared eigenvalue across the remaining $n-r$ directions. This scalar tail update is the optimal choice in KL-Shampoo when spike-and-flat holds exactly. In practice the tail is only approximately uniform, and the scalar tail update loses information. Pro-KLShampoo retains the spike-and-flat parametric family but \emph{replaces the scalar tail update with orthogonalization}. An identity (Eq.~\eqref{eq:polar_identity}) shows this replacement is not heuristic: orthogonalization on the remaining directions recovers the algebraic form of full KL-Shampoo's preconditioner there. Our contributions:
\begin{enumerate}[leftmargin=*,itemsep=1pt]
\item We identify a spike-and-flat eigenvalue structure in KL-Shampoo's preconditioners---empirically robust across layers and training stages, theoretically exact under a low-rank gradient model---and exploit it by restricting the KL objective to a parametric family aligned with this structure.

\item On the directions outside the tracked subspace, we show that orthogonalization recovers the algebraic form of full KL-Shampoo's preconditioner, connecting Muon-style orthogonalization to the KL estimation framework. A nonconvex convergence guarantee at rate $O(1/\sqrt{T})$ follows for the idealized algorithm under operator-norm smoothness.

\item We evaluate Pro-KLShampoo on GPT-2 (124M / 350M) and LLaMA (134M / 450M) at multiple subspace ranks. Pro-KLShampoo improves on KL-Shampoo in validation loss, peak memory, and wallclock time to reach matched loss levels in every configuration; an ablation study isolates the contribution of each design component.
\end{enumerate}

% ============================================================================
\section{Background}
\label{sec:background}
% ============================================================================

We use $\Spp^k$ for the cone of $k \times k$ symmetric positive-definite (SPD) matrices, and $\St(n,r) \coloneqq \{U \in \R^{n \times r}: U^\top U = I_r\}$ for the Stiefel manifold. We write $\|\cdot\|_F$, $\|\cdot\|_{\mathrm{op}}$, and $\|\cdot\|_*$ for the Frobenius, operator (spectral), and nuclear norms, respectively.

Consider a weight matrix $W \in \R^{m \times n}$ and stochastic gradient $G \in \R^{m \times n}$.
Following~\citet{gupta2018shampoo}, we use row-major vectorization: $\vol(G) \coloneqq (g_1^\top, \ldots, g_m^\top)^\top \in \R^{mn}$ where $g_i^\top$ is the $i$-th row of $G$.
We write $\Sigma \coloneqq \E[\vol(G)\vol(G)^\top] \in \Spp^{mn}$ for the gradient second moment.

\subsection{Kronecker-factored preconditioning: K-FAC, Shampoo, and KL-Shampoo}
Kronecker product approximation of the preconditioner has seen substantial development in neural-network training. \textbf{K-FAC}~\citep{martens2015optimizing} approximates the Fisher information matrix as a Kronecker product of two factors, enabling efficient natural-gradient updates. \textbf{Shampoo}~\citep{gupta2018shampoo} adopts the same Kronecker structure for the gradient second moment, applying the update $\Delta W \propto L^{-p}\,G\,R^{-p}$ where $L \in \Spp^m, R \in \Spp^n$ are estimated as $L \propto \E[GG^\top]$ and $R \propto \E[G^\top G]$. The original Shampoo~\citep{gupta2018shampoo} uses $p=1/4$ for the regret guarantee; subsequent theoretical and empirical work supports $p=1/2$, and we follow this convention throughout.

\textbf{KL-Shampoo}~\citep{lin2025understanding} reformulates Shampoo's estimation as a KL divergence minimization. Treating $\Sigma$ as the covariance of a zero-mean Gaussian, the goal is to find the best Kronecker-structured covariance $L \otimes R$ that approximates $\Sigma$:
\begin{equation}\label{eq:full_kl}
    \min_{L \in \Spp^m,\; R \in \Spp^n}\; \KL\!\big(\mathcal{N}(0,\Sigma) \;\|\; \mathcal{N}(0, L \otimes R)\big).
\end{equation}
Unlike Shampoo's Frobenius-based estimation, the KL divergence naturally respects the SPD constraint on $L$ and $R$. The stationarity conditions of~\eqref{eq:full_kl} (\citealt{lin2025understanding}, Claim~2) couple $L, R$:
\begin{equation}\label{eq:full_stat}
    L^* = \tfrac{1}{n}\,\E\!\left[G\,(R^*)^{-1}G^\top\right], \qquad
    R^* = \tfrac{1}{m}\,\E\!\left[G^\top\,(L^*)^{-1}\,G\right].
\end{equation}
Each side is the gradient second moment whitened by the other side---a coupling absent in Shampoo. This KL-based estimation removes the need for step-size grafting, which vanilla Shampoo typically requires. In practice, $L$ and $R$ are maintained via exponential moving averages (EMA).
\subsection{Orthogonalization}
For a matrix $M = U \Sigma V^\top$ (SVD), define $\polar(M) \coloneqq UV^\top = M(M^\top M)^{\dagger/2}$; this operation is called \emph{orthogonalization}. The gradient step $W \leftarrow W - \eta\polar(\nabla f(W))$ realizes steepest descent under the spectral norm~\citep{bernstein2024old, bernstein2024modular, large2024scalable}. Muon~\citep{modded_nanogpt_2024} applies this update to the momentum, with $\polar$ approximated by Newton--Schulz iteration.

% ============================================================================
\section{Pro-KLShampoo: Exploiting Spike-and-Flat Structure}
\label{sec:method}
% ============================================================================

In this section, we present \textbf{Pro-KLShampoo: Projected KL-Shampoo with Whitening Recovered by Orthogonalization}, by restricting KL-Shampoo's preconditioner to a parametric family with a spike-and-flat eigenvalue structure. The motivation comes from two parts: an empirical observation and an argument from a low-rank gradient model.

\paragraph{Empirical observation.}
Figure~\ref{fig:spike_flat} visualizes the eigenvalue spectra of KL-Shampoo's Kronecker preconditioners. Both sides exhibit a \emph{spike-and-flat} pattern: a small number of dominant eigenvalues followed by a nearly uniform tail. This is a robust empirical pattern across the layers, depths, and training stages we measure. The same structure also holds on LLaMA (Appendix~\ref{app:llama_spike}). Related spectral observations---low-rank gradient subspaces~\citep{jaiswal2024low, gur2018gradient} and spiked Hessian spectra~\citep{ghorbani2019investigation}---suggest that low-dimensional spectral structure is pervasive in deep network optimization.

\begin{figure}[t]
    \centering
    \includegraphics[width=\linewidth]{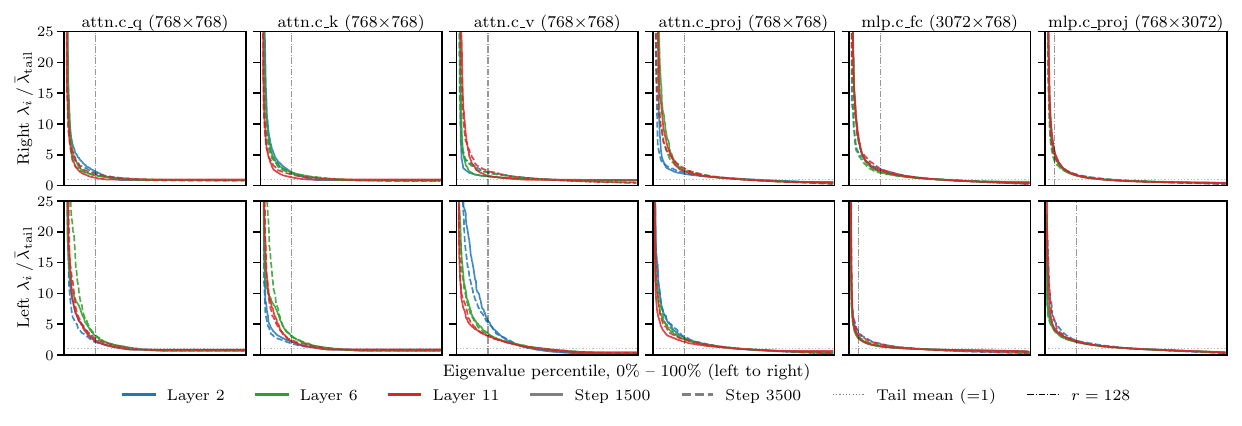}
    \caption{Eigenvalue spectra of practical version of KL-Shampoo's Kronecker preconditioners on GPT-2 (124M), normalized by the tail mean with rank $r = 128$ (vertical dashed line). Both the right-side preconditioner $R$ (top row) and the left-side preconditioner $L$ (bottom row) show a consistent spike-and-flat pattern across all layer types (by panel), three depths (by color), and two training stages (line style). In each row, the first four panels are attention (square) layers, the last two are MLP (rectangular) layers. For square layers, the structure is especially pronounced on the right side, most clearly in \texttt{c\_v}. For the rectangular layers, $r = 128$ still captures the spike for $R$ (the larger side), although the tail occupies a larger fraction of the spectrum and exhibits more variation.}
    \label{fig:spike_flat}
\end{figure}

\paragraph{Low rank gradient supports the spike-and-flat.}
The rank-$\rho$ signal-plus-noise gradient model is a fundamental building block of spiked covariance theory~\citep{johnstone2001distribution, paul2007asymptotics} and is consistent with low-rank gradient phenomena reported in transformer training~\citep{zhao2024galore,jaiswal2024low}. Under this model, the KL stationary points $L^*$ and $R^*$ have at least $n-\rho$ and $m-\rho$ equal bottom eigenvalues, respectively (Lemma~\ref{lem:spike_exact}; Appendix~\ref{app:spike_proof_remarks}).

\paragraph{Structural restriction.}
Together, these motivate restricting KL-Shampoo's Kronecker preconditioner to a parametric family that encodes the spike-and-flat structure explicitly: the preconditioner need not be maintained in full. We apply this restriction to the right-side preconditioner $R$ only, orienting rectangular layers so that $R$ is the larger side ($m \leq n$).\footnote{For square layers, Figure~\ref{fig:spike_flat} shows the spike on $R$ to be more pronounced than on $L$; for rectangular layers, the larger side has a much more expensive QR decomposition (e.g., GPT-2's MLP has a larger side $4\times$ the smaller, making its QR $64\times$ as expensive; LLaMA's SwiGLU MLP has ratio $8/3$).} We choose an $r$-dimensional projected subspace (spanned by a particular $U \in \St(n,r)$) in which $R$ retains its full spectral structure, and collapse the remaining $n - r$ dimensions to a single scalar:
\begin{equation}\label{eq:spike_flat}
    \hat{R} \;=\; USU^\top + \mu_\perp P_\perp.
\end{equation}
Here $S \in \Spp^r$ is the preconditioner maintained in the tracked subspace $U$, and $\mu_\perp > 0$ is the scalar approximating the uniform tail of the spectrum on the complement, with $P_\perp \coloneqq I_n - UU^\top$. Ideally, $r$ captures the dominant eigenvalues of $R$ (Figure~\ref{fig:spike_flat}). When $r \ll n$, the right-side storage drops from $O(n^2)$ to $O(nr)$ and its QR cost from $O(n^3)$ to $O(r^3)$; the left-side is unchanged.

\begin{tcolorbox}[keyeq={Spike-and-flat restriction of the KL objective}]
\begin{equation}\label{eq:restr}
    \mathcal{J}^{\mathrm{restr}} \;\coloneqq\; \min_{\substack{L \in \Spp^m,\; U \in \St(n,r) \\ S \in \Spp^r,\; \mu_\perp > 0}} \KL\!\left(\mathcal{N}(0,\Sigma) \;\|\; \mathcal{N}\left(0, L \otimes (USU^\top + \mu_\perp P_\perp)\right)\right)
\end{equation}
\end{tcolorbox}

% \paragraph{Projected KL-Shampoo with Whitening Recovered by Orthogonalization (Pro-KLShampoo)}

% \unpolishedstart
% Algorithm~\ref{alg:ideal} displays the method alongside KL-Shampoo (\textcolor{ourblue}{blue} = our additions).
% The restricted optimum decomposes cleanly: on the subspace, we recover KL-Shampoo's coupled estimation (Claim~\ref{claim:stat}); on the complement, the update reduces to scalar-whitened gradient descent, which we sharpen with an orthogonalized update (\S\ref{sec:polar}).
% The subspace itself is tracked by power iteration on the full-preconditioner estimate---the optimal direction under an interlacing argument (\S\ref{sec:opt_sub}).
% \unpolishedend

Restricting $\hat R$ to spike-and-flat structure discards information whenever the bottom $n-r$ eigenvalues of the full optimum $R^*$ are not exactly uniform. Intuitively, the flatter the tail, the more accurate the approximation. The following claim quantifies the intuition by the $\AM/\GM$ ratio (the ratio of arithmetic to geometric mean) of the tail eigenvalues.

\begin{claim}[Approximation gap]\label{claim:gap}
Denote the optimal full KL objective in~\eqref{eq:full_kl} as $\mathcal{J}^{\mathrm{full}}$. Let $(L^*, R^*)$ be the full KL optimum with $R^*$ having eigenvalues $\mu_1^* \geq \cdots \geq \mu_n^* > 0$. Then:
\begin{equation}\label{eq:gap_bound}
    0 \;\leq\; \mathcal{J}^{\mathrm{restr}} - \mathcal{J}^{\mathrm{full}} \;\leq\; \frac{m(n-r)}{2}\,\log\frac{\AM(\mu_{r+1}^*,\ldots,\mu_n^*)}{\GM(\mu_{r+1}^*,\ldots,\mu_n^*)}.
\end{equation}
When $\mu_{r+1}^* = \cdots = \mu_n^*$\ , the gap is zero.
\end{claim}

The $\AM/\GM$ ratio is commonly used in majorization theory to represent the spread of a positive sequence~\citep{marshall1979inequalities}.
% Appendix~\ref{app:proof_gap} gives the proof.

\paragraph{Decomposition of the preconditioned gradient.}
Since $USU^\top$ and $\mu_\perp P_\perp$ have orthogonal ranges, $\hat{R}^{-1/2}$ (using $p=1/2$ for the preconditioner exponent) decomposes as $U S^{-1/2}U^\top + \mu_\perp^{-1/2}P_\perp$, and the preconditioned gradient splits into
\begin{equation}\label{eq:restr_update}
\begin{aligned}
L^{-1/2}\, G\, \hat{R}^{-1/2}
&=
\underbrace{
    \vphantom{\mu_\perp^{-1/2}}
    L^{-1/2}\, G U\, S^{-1/2}\, U^\top
}_{
    \substack{
    \text{full KL-Shampoo in the subspace}
    }}
\quad +
\underbrace{
    \mu_\perp^{-1/2}\, L^{-1/2}\, G P_\perp
}_{
    \substack{
    \text{one-sided KL-Shampoo in the complement}\\
    \text{with the other side approximated by a scalar tail}
    }}
\end{aligned}
\end{equation}

We refer to this preconditioned gradient (and the resulting algorithm) as \textbf{Smok-Hop}.\footnote{Smok-Hop is the Pro-KLShampoo algorithm without orthogonalization, introduced in \S\ref{sec:polar} (denoted $\polar(\cdot)$). The name spells out the construction: the letters of ``Pro-KLShampoo'' minus those of ``polar.''}
On the projected subspace, preconditioning is two-sided (as in KL-Shampoo), while the complement is left-preconditioned and uniformly scaled by a scalar on the right side. The scalar tail discards information when the bottom eigenvalues are not exactly uniform; relatedly, two-sided preconditioning is known to outperform one-sided variants more generally~\citep{eschenhagen2026clarifying}. We improve the complement of the preconditioned gradient by introducing orthogonalization (see details in \S\ref{sec:polar}). The overall Pro-KLShampoo update reads:
\begin{equation}\label{eq:update}
\begin{aligned}
\Delta W
&= -\alpha_{\mathrm{kl}} L^{-1/2}\, G U\, S^{-1/2}\, U^\top - c_a \polar \big(\mu_\perp^{-1/2}\, L^{-1/2}\, G P_\perp \big),
\end{aligned}
\end{equation}
where $c_a = \sqrt{\max(1,m/n)}$ is an aspect-ratio scaling for the orthogonalized complement, and $\alpha_{\mathrm{kl}}$ is added since orthogonalization changes the scale of the complement. We characterize the magnitude of $\alpha_{\mathrm{kl}}$ in section~\ref{sec:polar}.

% ---------------------------------------------------------------------------
\subsection{Stationarity of the restricted problem}
\label{sec:stationarity}

We fix $U$ for now (the choice of $U$ is addressed in \S\ref{sec:opt_sub}) and derive the optimal $L, S, \mu_\perp$. Note that once $U$ is given, the gradient decomposes into the projected subspace and the complement component, i.e., $G = \Gtil U^\top + G_\perp$ where $\Gtil \coloneqq GU \in \R^{m \times r}$ and $G_\perp \coloneqq GP_\perp \in \R^{m \times n}$.

\begin{claim}[Restricted stationarity]\label{claim:stat}
The optimal solution of the restricted KL objective~\eqref{eq:restr} with respect to $(L, S, \mu_\perp)$ for fixed $U$ should satisfy:
\begin{align}
&S^*
= \tfrac{1}{m}\,\E\bigl[\Gtil^\top (\Lr)^{-1} \Gtil\bigr],
\quad \text{\aside{(Covariance estimation of the subspace preconditioner)}} \label{eq:restr_S} \\
&\mu_\perp^*
= \tfrac{1}{m(n-r)}\,\Tr\!\bigl(\E\bigl[G_\perp^\top (\Lr)^{-1} G_\perp\bigr]\bigr),
\quad \text{\aside{(Scalar estimation of the complement's flat tail)}} \label{eq:restr_mu} \\
&\Lr
= \tfrac{1}{n}\,\E\bigl[G\,(\hat{R}^*)^{-1}G^\top\bigr],
\quad \text{\aside{(Covariance estimation of the unrestricted-side preconditioner)}} \label{eq:restr_L}
\end{align}

where $\hat{R}^* = US^*U^\top + \mu_\perp^* P_\perp$.
The proof is in Appendix~\ref{app:proof_stat}.
\end{claim}

The optimal solutions are coupled (as in KL-Shampoo): $\Lr$ and $\hat R^*$ are gradient second moments mutually whitened by each other's inverse, with $\hat R^*$ decomposing into $S^*$ on the projected subspace and the scalar $\mu_\perp^*$ averaging the whitened second moment on the complement. Explicit characterization is infeasible---$\Lr$ is needed to solve for $S^*$ but itself unknown---so we use the EMA scheme of~\citet{lin2025understanding}, interpreted as a stochastic proximal gradient step for the KL objective.

% ---------------------------------------------------------------------------
\subsection{Optimal subspace}
\label{sec:opt_sub}

Claim~\ref{claim:stat} gives the optimal $L, S, \mu_\perp$ for any fixed $U$. Plugging this back, we characterize the minimizer of the resulting objective in $U$, which depends on the matrix $\Phi_L \coloneqq \E[G^\top L^{-1}G] \in \Spp^n$; recall that under full KL-Shampoo stationarity~\eqref{eq:full_stat}, $\Phi_{L^*} = m R^*$.

\begin{claim}[Optimal subspace]\label{claim:opt_sub}
Fix $L \in \Spp^m$ and let $\Phi_L$ have eigenvalues $\phi_1 \geq \cdots \geq \phi_n > 0$. Every global minimizer $U^*$ of~\eqref{eq:restr} over $U \in \St(n,r)$ is an eigenspace of $\Phi_L$. The optimal index set $I^*$ is the size-$r$ subset whose complement has the smallest AM/GM: $I^* \in \argmin_{|I|=r}\; \frac{\AM(\phi_j : j \notin I)}{\GM(\phi_j : j \notin I)}$. When the bottom-$(n{-}r)$ eigenvalues form the flattest subset (spike-and-flat shape), $I^* = \{1,\ldots,r\}$.
\end{claim}

% \begin{enumerate}[label=(\roman*),leftmargin=*,itemsep=2pt]
%     \item \unpolished{Substituting $S^*$ and $\mu_\perp^*$ from Claim~\ref{claim:stat} into $\mathcal{J}$, the reduced objective (up to constants independent of $U$) is}
%     \begin{equation}\label{eq:reduced_obj}
%         f(U) = \frac{m}{2}\!\left[\log\det(U^\top\Phi_L U) + (n{-}r)\log\Tr(P_\perp\Phi_L)\right].
%     \end{equation}

%     \item \unpolished{The global minimizer of $f$ over $\St(n,r)$ is achieved at an eigenspace of $\Phi_L$: there exists $I^* \subset \{1,\ldots,n\}$ with $|I^*| = r$ such that $U^* = [v_{i_1^*}, \ldots, v_{i_r^*}]$ is optimal.}
%     \item \unpolished{Among all eigenspace choices, the optimal one minimizes the complement's AM/GM ratio:}
%     \begin{equation}\label{eq:opt_AMGM}
%         I^* \in \argmin_{|I|=r}\; \frac{\AM(\phi_j : j \notin I)}{\GM(\phi_j : j \notin I)}.
%     \end{equation}

%     \item \unpolished{When the bottom-$(n{-}r)$ eigenvalues have the smallest AM/GM ratio among all $(n{-}r)$-subsets (the typical spike-and-flat shape), $I^* = \{1,\ldots,r\}$---the top-$r$ eigenspace.}
% \end{enumerate}

See the proof in Appendix~\ref{app:proof_opt_sub}. In general, the optimal subspace need not be the top-$r$ eigenspace. For example, $\Phi_L=\Diag(10, 9, 1)$ with $r=1$: the top-1 choice gives complement $\AM/\GM \approx 1.67$, while capturing the \emph{bottom} eigenvalue gives $\AM/\GM \approx 1.00$. Note that the stationary condition of full KL-Shampoo~\eqref{eq:full_stat} gives $R^*{=}\tfrac{1}{m}\Phi_{L^*}$, so $R^*$ and $\Phi_{L^*}$ share the same spike-and-flat structure (Figure~\ref{fig:spike_flat}); since we choose $r$ to capture the dominant eigenvalues of $R^*$, it also captures those of $\Phi_{L^*}$.
We thus conjecture that $\Phi_{\Lr}$ has similar structure with its dominant eigenvalues captured within the top $r$; we verify this empirically (Figure~\ref{fig:phi_L_spike_flat}, Appendix~\ref{app:phi_L}). Under this structure, Claim~\ref{claim:opt_sub} gives the top-$r$ eigenspace as optimal. Since $\Lr$ is unknown, in practice we run subspace iteration on $\Phi_L$, where $L$ is the algorithm's running EMA estimate of $\Lr$.

\vspace{5pt}

\begin{remark}
Tracking the top-$r$ eigenspace is also used in prior subspace-based optimizers~\citep{zhao2024galore,liu2025cosmos}; under the restricted KL objective and the spike-and-flat conjecture on $\Phi_{\Lr}$, this construction emerges as the KL-optimal choice (Claim~\ref{claim:opt_sub}).
\end{remark}

% ---------------------------------------------------------------------------
\subsection{Recover per-direction whitening by orthogonalization}
\label{sec:polar}

The update decomposition~\eqref{eq:restr_update} applies the same scalar $\mu_\perp^{-1/2}$ to all $n{-}r$ directions in the complement subspace, rather than a preconditioner matrix scaling each direction. We call the latter \emph{per-direction whitening}. We argue below that orthogonalization recovers the algebraic form of KL-Shampoo-style per-direction whitening on the complement subspace.

\paragraph{Decomposition of full KL-Shampoo's update.}
Full KL-Shampoo applies the update $\Delta W^{\mathrm{full}} = -L^{*-1/2}\, G\, R^{*-1/2}$.
Let $U_f$ be the top-$r$ eigenspace of $R^*$ and $U_{f,\perp}$ be an orthonormal basis for $U_f$'s orthogonal complement.
Since $R^*$ is block-diagonal in the basis $[U_f, U_{f,\perp}]$, the full-KL update decomposes as
\[
    \Delta W^{\mathrm{full}} = -L^{*-1/2}\, G\, U_f\, (U_f^\top R^* U_f)^{-1/2}\, U_f^\top \;-\; \underbrace{L^{*-1/2}\, G\, U_{f,\perp}\, (U_{f,\perp}^\top R^* U_{f,\perp})^{-1/2}\, U_{f,\perp}^\top}_{\Delta^{\mathrm{full}}_\perp}. 
\]
Substituting the full-KL stationarity condition~\eqref{eq:full_stat} into $\Delta^{\mathrm{full}}_\perp$:
\begin{equation}\label{eq:fullcomp}
    \Delta^{\mathrm{full}}_\perp = \frac{1}{m}\, L^{*-1/2}\, G\, U_{f,\perp}\, \bigl(U_{f,\perp}^\top \highlight{\boldsymbol{\E[G^\top L^{*-1} G]}}\, U_{f,\perp}\bigr)^{-1/2}\, U_{f,\perp}^\top.
\end{equation}

\paragraph{Pro-KLShampoo's orthogonalization on the complement.}
Denote $U_\perp$ as an orthonormal basis for $U$'s orthogonal complement, equivalently $P_\perp = U_\perp U_\perp^\top$. The complement update of~\eqref{eq:update} reads (the aspect-ratio scaling $c_a$ is absorbed into the learning rate):

\begin{tcolorbox}[keyeq={Orthogonalization recovers the algebraic form of complement whitening}]
\begin{align}
    \polar(\mu_\perp^{-1/2} L^{-1/2} G\, P_\perp)
    &= \polar(L^{-1/2} G\, U_\perp)\, U_\perp^\top \notag \\
    &= L^{-1/2} G\, U_\perp\, \bigl(U_\perp^\top \highlight{\boldsymbol{G^\top L^{-1} G}}\, U_\perp\bigr)^{\dagger/2}\, U_\perp^\top, \label{eq:polar_identity}
\end{align}
\end{tcolorbox}

% using $\polar(\alpha M) = \polar(M)$, $\polar(M) = M(M^\top M)^{\dagger/2}$ and $\polar(MW^\top) = \polar(M)\,W^\top$ for column-orthonormal $W$.

This update matches $\Delta^{\mathrm{full}}_\perp$ in algebraic form: both~\eqref{eq:fullcomp} and~\eqref{eq:polar_identity} have the same structure, with full KL-Shampoo's stationary quantities replaced by Pro-KLShampoo's running ones, and the expectation $\E[G^\top L^{*-1} G]$ replaced by the instantaneous $G^\top L^{-1} G$. Pro-KLShampoo's $L, U_\perp$ track its own stationary point---of the restricted KL objective (Claims~\ref{claim:stat} and~\ref{claim:opt_sub})---which coincides with full KL-Shampoo's when $\Lr = L^*$: then $\Phi_{\Lr} = m R^*$ by~\eqref{eq:full_stat}, so $\Phi_{\Lr}$ and $R^*$ have the same top-$r$ eigenspace, and $U_\perp$ tracks the subspace spanned by $U_{f,\perp}$.

Identity~\eqref{eq:polar_identity} is therefore not just a cosmetic match: it is the algebraic mechanism by which orthogonalization recovers full KL-Shampoo's complement form on Pro-KLShampoo's running estimates. Beyond Pro-KLShampoo, identity~\eqref{eq:polar_identity} supports the broader idea of composing Muon-style orthogonalization with Shampoo-style whitening.
\paragraph{Calibrating $\alpha_{\mathrm{kl}}$.}
The mixing weight $\alpha_{\mathrm{kl}}$ is set by matching the operator-norm of the orthogonalized complement to that of the original complement update at the restricted KL stationary state. This yields a per-layer principled range, which we intersect across layers to obtain $\alpha_{\mathrm{kl}} \in [10^{-3},\, 2{\times} 10^{-2}]$ (Appendix~\ref{app:lambda_star}). In experiments we sweep $\alpha_{\mathrm{kl}} \in \{0.005, 0.01, 0.015\}$ within this range.

\label{sec:ideal_alg}

\begin{algorithm}[H]
\setstretch{1.2}
\caption{Pro-KLShampoo (idealized)}
\label{alg:ideal}
\begin{algorithmic}[0]
\vspace{3pt}
\State \textbf{1: Gradient projection}
\Statex \quad $G \coloneqq \mathrm{Mat}(\nabla\ell(\theta)) \in \R^{m\times n}$;\;\; $\Gtil \coloneqq GU \in \R^{m\times r}$;\;\; $G_\perp \coloneqq G - \Gtil U^\top$
\State \textbf{2: Covariance estimation} (each iter)
\Statex \quad $\begin{aligned} \begin{pmatrix} L \\ S \end{pmatrix} &\gets \beta_2\begin{pmatrix} L \\ S \end{pmatrix} + (1{-}\beta_2)\begin{pmatrix} \Delta_L \\ \Delta_S \end{pmatrix} \\ \mu_\perp &\gets \beta_2\,\mu_\perp + (1{-}\beta_2)\,\delta_\perp \end{aligned} \quad \begin{aligned} \Delta_L &= \tfrac{1}{n}\bigl(\Gtil\,Q_S\Diag(\lambda_S^{\odot-1})Q_S^\top\Gtil^\top + \mu_\perp^{-1}\,G_\perp G_\perp^\top\bigr) \\ \Delta_S &= \tfrac{1}{m}\,\Gtil^\top Q_L\Diag(\lambda_L^{\odot-1})Q_L^\top\Gtil \\ \delta_\perp &= \tfrac{1}{m(n-r)}\,\Tr\!\bigl(G_\perp^\top Q_L\Diag(\lambda_L^{\odot-1})Q_L^\top G_\perp\bigr) \end{aligned}$
\State \textbf{3: Subspace tracking} (each iter)\label{line:track}
\Statex \quad $U_+ \gets \qr\!\bigl(\beta_2\,U S + (1{-}\beta_2)\,\tfrac{1}{m}\,G^\top Q_L\Diag(\lambda_L^{\odot-1})Q_L^\top G\,U\bigr)$
\Statex \quad $T_U \coloneqq U^\top U_+$;\;\; $S \gets T_U^\top S\,T_U$;\;\; $Q_S \gets T_U^\top Q_S$;\;\; $U \gets U_+$
\State \begin{minipage}[t]{0.48\linewidth}\textbf{4a: Eigenvalue EMA} (each iter)\\[3pt]
\hspace*{1em}$\lambda_L \gets \beta_2\,\lambda_L + (1{-}\beta_2)\,\diag(Q_L^\top \Delta_L\, Q_L)$\\[3pt]
\hspace*{1em}$\lambda_S \gets \beta_2\,\lambda_S + (1{-}\beta_2)\,\diag(Q_S^\top \Delta_S\, Q_S)$
\end{minipage}\hfill\begin{minipage}[t]{0.48\linewidth}\textbf{4b: Infrequent QR} (every $\tau$ iters)\\[3pt]
\hspace*{1em}$Q_L \gets \qr(L\, Q_L)$\\[3pt]
\hspace*{1em}$Q_S \gets \qr(S\, Q_S)$
\end{minipage}
\vspace{3pt}
\State \textbf{5: Update}
\Statex \quad $W \gets W - \eta\,\bigl[\alpha_{\mathrm{kl}}\, L^{-1/2}\,G\,U\,S^{-1/2}\,U^\top + c_a\,\polar\!\bigl(\mu_\perp^{-1/2}\,L^{-1/2}\,G\,P_\perp\bigr)\bigr]$\label{line:comp}
\end{algorithmic}
\end{algorithm}

Algorithm~\ref{alg:ideal} completes the specification of Pro-KLShampoo. We track the top-$r$ eigenspace of $\Phi_L \coloneqq \E[G^\top L^{-1} G]$ via one step of power iteration (Line~\ref{line:track})---under the spike-and-flat structure of $\Phi_{\Lr}$, the KL-optimal subspace (Claim~\ref{claim:opt_sub}). We estimate $\Phi_L U$ by EMA over a fresh minibatch sample $\tfrac{1}{m}\,G^\top Q_L\Diag(\lambda_L^{\odot-1})Q_L^\top G\,U$ and a running estimate $U S$, exact at the restricted KL stationary point because $L = \Lr$ (Claim~\ref{claim:stat}) and $U$ is an eigenspace of $\Phi_{\Lr}$ (Claim~\ref{claim:opt_sub}) together give $\Phi_{\Lr} U = m\,U S$. Because $S$ and $Q_S$ are expressed in the old $U$ basis, the second line of Line~\ref{line:track} rotates them into the new basis via $T_U \coloneqq U^\top U_+$. The remaining steps---EMA covariance estimation (Step~2), decoupled eigenvalue EMA (Step~4a), and infrequent eigenbasis QR (Step~4b)---follow~\citet{lin2025understanding}. All EMAs in Algorithm~\ref{alg:ideal} share a single coefficient $\beta_2$.

The practical implementation, which adds Nesterov momentum, Newton--Schulz iteration for orthogonalization, and eigenvalue clipping, is given in Appendix~\ref{app:practical}.

% \begin{corollary}[KL coefficient]\label{cor:lambda_star}
% \begin{equation}\label{eq:lambda_star}
%     \alpha_{\mathrm{kl}}^*\;=\;\frac{c_a\sqrt{k}}{\sqrt{m(n-r)}}.
% \end{equation}
% \unpolishedstart
% For square layers ($m=n$), this simplifies to $\alpha_{\mathrm{kl}}^* = 1/\sqrt{n}$.
% \unpolishedend
% \end{corollary}

% \begin{table}[t]
% \centering
% \small
% \begin{tabular}{lcccccc}
% \toprule
% Layer & $m$ & $n$ & $r$ & $c_a$ & $k$ & $\alpha_{\mathrm{kl}}^*$ \\
% \midrule
% Attention QKV / out & 768 & 768 & 128 & $1$   & 640 & $0.036$ \\
% MLP up   ($4n\times n$) & 3072 & 768 & 128 & $2$   & 640 & $0.036$ \\
% MLP down ($n\times 4n$) & 768 & 3072 & 128 & $1$   & 768 & $0.018$ \\
% \bottomrule
% \end{tabular}
% \caption{$\alpha_{\mathrm{kl}}^*$ predicted by Corollary~\ref{cor:lambda_star} for GPT-2 (124M) at $r=128$.
% Empirically we use $\alpha_{\mathrm{kl}} = 0.02$; the prediction matches within $2\times$ across all layer shapes and within $10\%$ for MLP-down.}
% \label{tab:lambda_star}
% \end{table}

% ---------------------------------------------------------------------------
\subsection{Convergence analysis}
The idealized Pro-KLShampoo update~\eqref{eq:update} admits an $O(1/\sqrt T)$ nonconvex convergence guarantee under operator-norm smoothness~\citep{bernstein2024old,large2024scalable} and standard nonconvex stochastic assumptions (Theorem~\ref{thm:polar_conv}). Here we assume a uniform lower bound on preconditioner eigenvalues (Lemma~\ref{lem:clamp_lb}), which is enforced in practice by eigenvalue clipping~\citep{lin2025understanding}. Let $f:\R^{m\times n}\to\R$ denote the loss as a function of the matrix-shaped weights $W$, and let $\mathcal F_t \coloneqq \sigma(G_0, \ldots, G_{t-1})$ denote the natural filtration; the algorithm state $(L_t, S_t, U_t, \mu_{\perp,t})$ at step $t$ is $\mathcal F_t$-measurable, and conditional expectations $\E[\cdot \mid \mathcal F_t]$ are taken over the next gradient sample $G_t$.

\begin{theorem}[Convergence of Pro-KLShampoo]\label{thm:polar_conv}
Assume $f$ is $L_{\mathrm{op}}$-operator-norm smooth, $G_t$ is unbiased with variance $\sigma_F^2$ and $\|G_t\|_{\mathrm{op}}\leq G_{\max}$ a.s., and the preconditioner eigenvalues are uniformly lower bounded (Lemma~\ref{lem:clamp_lb}). Define $\sigma_{kl}^2 \;\coloneqq\; \sup_{t\geq 0}\,\operatorname{ess\,sup}\,\E\!\left[\|L_t^{-1/2}G_t U_t S_t^{-1/2}\|_{\mathrm{op}}^2 \;\big|\; \mathcal F_t\right]$, finite and deterministic under the preceding assumptions. Then Algorithm~\ref{alg:ideal} with step size $\eta = \sqrt{\Delta_0/(T\,L_{\mathrm{op}}\,K)}$ and $K \coloneqq c_a^2 + \alpha_{\mathrm{kl}}^2\,\sigma_{kl}^2$ satisfies
\begin{equation}\label{eq:polar_conv}
% \vspace{-3pt}
\frac{1}{T}\sum_{t=0}^{T-1}\E\!\left[\tfrac{c_a}{C\sqrt{\Theta}}\bigl\|\nabla f(W_t)\,P_{\perp,t}\bigr\|_* + \tfrac{\alpha_{\mathrm{kl}}}{\Theta}\bigl\|\nabla f(W_t)\,U_t\bigr\|_F^2\right]
\;\leq\; 2\sqrt{\frac{\Delta_0\,L_{\mathrm{op}}\,K}{T}} + 2\,c_a\sqrt{k}\,\sigma_F,
% \vspace{-2pt}
\end{equation}
where $\Delta_0 \coloneqq f(W_0) - f^*$, $c_a = \sqrt{\max(1, m/n)}$, $k \coloneqq \min(m, n-r)$, and $\Theta, C > 0$ are the spectral upper bound and clip threshold (lower bound) from Lemma~\ref{lem:clamp_lb}.
\end{theorem}
The right-hand side of~\eqref{eq:polar_conv} gives $O(1/\sqrt T)$ convergence to a noise floor of $2c_a\sqrt{k}\,\sigma_F$, the standard SGD-style rate for nonconvex stochastic optimization. The left-hand side mirrors the two update geometries: the orthogonalized complement contributes a \emph{nuclear-norm} term (dual of spectral-norm), and the subspace update contributes a \emph{squared Frobenius} term. The asymmetry is because orthogonalization normalizes the gradient's magnitude away in the complement update, so the descent is linear in $\nabla f$; the subspace update preserves the gradient's magnitude, so $\nabla f$ contributes quadratically. The measure is also \emph{state-dependent}: it splits the gradient $\nabla f(W_t)$ via the algorithm's running $(U_t, P_{\perp,t})$. Lemma~\ref{lem:equiv} below shows the left-hand side vanishes if and only if $\nabla f(W_t) = 0$.
\begin{lemma}\label{lem:equiv}
For any $W\in\R^{m\times n}$, any $U\in\St(n,r)$, and $P_\perp \coloneqq I_n - UU^\top$,
\[
\tfrac{c_a}{C\sqrt{\Theta}}\|\nabla f(W)\,P_\perp\|_* + \tfrac{\alpha_{\mathrm{kl}}}{\Theta}\|\nabla f(W)\,U\|_F^2 \;=\; 0
\quad\text{if and only if}\quad
\nabla f(W) = 0.
\]
\end{lemma}

\label{sec:naive_conv}

% ============================================================================
\section{Experiments}
\label{sec:experiments}
% ============================================================================
We evaluate Pro-KLShampoo at four model scales and two architectures: GPT-2~\citep{radford2019language} (124M, 350M) trained on FineWeb-10B~\citep{penedo2024fineweb} and LLaMA~\citep{touvron2023llama} (134M, 450M) trained on C4~\citep{raffel2020exploring}, all on A100 GPUs. Our primary baseline is KL-Shampoo~\citep{lin2025understanding}; we also compare against AdamW and Muon~\citep{modded_nanogpt_2024}, and provide a head-to-head comparison with COSMOS~\citep{liu2025cosmos} in Appendix~\ref{app:cosmos}. For Muon, KL-Shampoo, and Pro-KLShampoo, the embedding and output weights are trained by AdamW, following~\citet{eschenhagen2026clarifying, lin2025understanding}.

We tune AdamW first, then fix its learning rate and weight decay for the embedding and output layers across all optimizers; the remaining hyperparameters of each method are tuned independently. For Pro-KLShampoo, we additionally sweep $\alpha_{\mathrm{kl}}\in\{0.005, 0.01, 0.015\}$ within the principled range derived in \S\ref{sec:polar}, and report results at $r\in\{32, 64, 128\}$. Other hyperparameters (momentum, $\beta_2$, preconditioner refresh frequency, Newton--Schulz step count) follow the source-release defaults of~\citet{lin2025understanding,modded_nanogpt_2024}; full configuration is in Appendix~\ref{app:sweeps}. KL-Shampoo and Pro-KLShampoo maintain the optimizer state in half precision following~\citet{lin2025understanding}; AdamW and Muon use the full-precision optimizer state of their official releases. Reported memory therefore reflects each method's reference configuration; we did not re-validate other methods at half precision.

\begin{table}[h]
\centering
\setlength{\tabcolsep}{5pt}
\caption{Validation loss (\emph{loss}) and peak per-GPU memory in GiB (\emph{mem}) across pretraining configurations. Best loss across hyperparameter sweeps; memory is from that run.}
\label{tab:main}
\begin{tabular}{cccccccccc}
\toprule
\multirow{2}{*}{Optimizer} & \multirow{2}{*}{$r$} & \multicolumn{2}{c}{GPT-2 124M} & \multicolumn{2}{c}{GPT-2 350M} & \multicolumn{2}{c}{LLaMA 134M} & \multicolumn{2}{c}{LLaMA 450M} \\
\cmidrule(lr){3-4}\cmidrule(lr){5-6}\cmidrule(lr){7-8}\cmidrule(lr){9-10}
& & loss & mem & loss & mem & loss & mem & loss & mem \\
\midrule
AdamW         & --  & 3.4944 & 30.62 & 3.2063 & 74.20 & 3.0871 & 32.19 & 2.8248          & 73.91 \\
Muon          & --  & 3.2841 & 30.14 & 3.0407 & 72.23 & 3.0536 & 31.71 & 2.7578          & 71.75 \\
KL-Shampoo    & --  & 3.2796 & 30.63 & 3.0397 & 73.07 & 3.0504 & 32.09 & 2.7650          & 73.42 \\
\midrule
\multirow{3}{*}{Pro-KLShampoo} & 32  & 3.2758 & 30.16 & 3.0364 & 72.23 & 3.0358 & 31.74 & 2.7461          & 71.85 \\
                               & 64  & 3.2754 & 30.16 & 3.0348 & 72.24 & 3.0345 & 31.75 & \textbf{2.7453} & 71.86 \\
                               & 128 & \textbf{3.2745} & 30.17 & \textbf{3.0344} & 72.25 & \textbf{3.0335} & 31.76 & 2.7469          & 71.89 \\
\bottomrule
\end{tabular}
\end{table}

\subsection{Pretraining experiments}
\label{sec:pretraining}

We pretrain two GPT-2 models using the modded-nanogpt framework~\citep{modded_nanogpt_2024}: a 124M model for $5{,}100$ iterations and a 350M model for $13{,}000$ iterations, corresponding to $2.7$B and $6.8$B tokens respectively. We additionally pretrain two LLaMA models on C4 following the GaLore~\citep{zhao2024galore} convention with a linear warmup--linear decay schedule, $10\%$ warmup, and zero weight decay: a 134M model on a 5B subset for $5{,}000$ iterations and a 450M model on a 10B subset for $6{,}250$ iterations. Architecture details, sweep ranges, and selected hyperparameters are in Appendix~\ref{app:sweeps}.

\paragraph{Main results.}
Across all configurations, Pro-KLShampoo attains the lowest validation loss at every tested rank, with peak per-GPU memory below KL-Shampoo's and on par with Muon's (Table~\ref{tab:main}). Pro-KLShampoo also reaches validation loss level of KL-Shampoo in less total wallclock time at every rank (Figure~\ref{fig:gpt-wallclock}); we call this reduction \emph{wallclock saving}. The best wallclock savings are $4.12\%$ (GPT-2 350M), $2.28\%$ (GPT-2 124M), $13.43\%$ (LLaMA 450M), and $10.96\%$ (LLaMA 134M).
\begin{figure}[h]
\centering
\includegraphics[width=\linewidth]{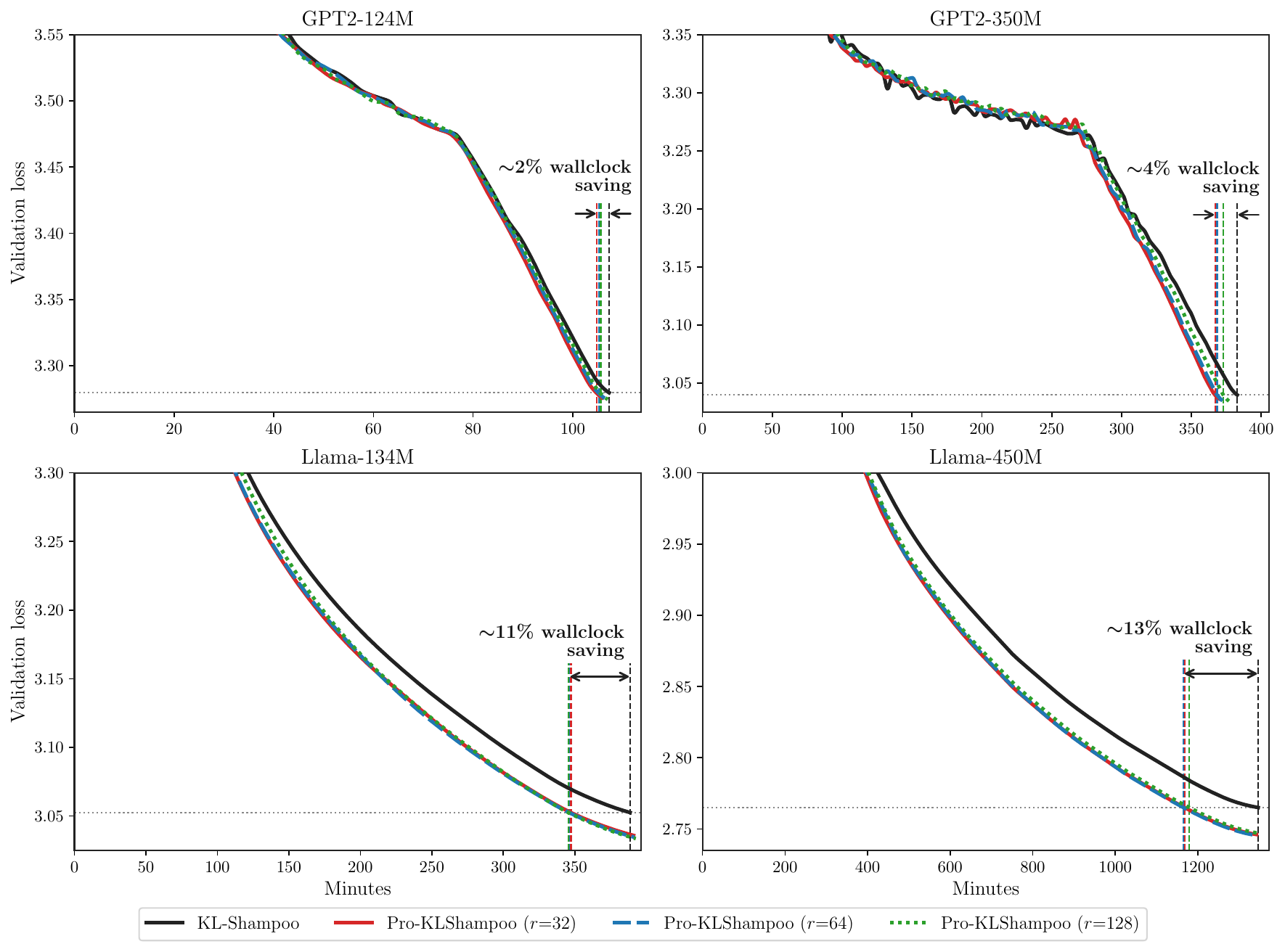}
\caption{Validation loss versus time for KL-Shampoo and Pro-KLShampoo with $r \in \{32, 64, 128\}$ across all four pretraining configurations. Pro-KLShampoo reaches each loss level faster than KL-Shampoo at every rank in every panel. On GPT-2, its curves also terminate earlier at every rank; under the fixed iteration count of each panel, this implies shorter per-step time.}
\label{fig:gpt-wallclock}
\end{figure}

\paragraph{Mechanism of the wallclock saving.}
Wallclock saving has three potential sources: reduced per-step time (each iteration is faster), faster convergence (fewer iterations to reach a given training loss), and a generalization effect (lower validation loss at matched training loss). On GPT-2, per-step time is the primary source: Pro-KLShampoo replaces the QR decomposition on the larger Kronecker factor with one of size $r$, and GPT-2's largest weight dimension ($3072$ at 124M, $4096$ at 350M) is where this reduction outweighs the overhead of subspace tracking (Appendix~\ref{app:cost} compares per-step cost and memory). Convergence does not contribute: Pro-KLShampoo's training loss is comparable to or slightly above KL-Shampoo's(Figure~\ref{fig:gpt-trainloss}, Appendix~\ref{app:trainloss}). A small generalization effect may also contribute: on GPT-2, Pro-KLShampoo's validation loss is lower than KL-Shampoo's despite comparable training loss; we leave its mechanism to future work. On LLaMA, the largest weight dimension is smaller ($2048$ at 134M, $2816$ at 450M), so per-step time is comparable between the two methods. The wallclock saving is instead driven by faster convergence: Pro-KLShampoo's training loss lies visibly below KL-Shampoo's throughout training (Figure~\ref{fig:llama-trainloss}, Appendix~\ref{app:trainloss}).

At smaller ranks such as $r{=}32$, the tracked subspace may not capture all dominant eigenvalues of the preconditioner (Figure~\ref{fig:spike_flat}); orthogonalization on the complement (\S\ref{sec:polar}) compensates by recovering per-direction whitening regardless of the scalar tail's accuracy, and the consistent improvement at $r{=}32$ across all four scales (Table~\ref{tab:main}) confirms this robustness.

\subsection{Ablation studies}
\label{sec:ablation}

We ablate two ingredients of Pro-KLShampoo: orthogonalization on the complement subspace (\S\ref{sec:polar}) and the spike-and-flat decomposition \eqref{eq:update}. All variants run at $r{=}128$. Validation losses are reported in Table~\ref{tab:ablation}; loss curves on GPT-2 124M and LLaMA 134M are shown in Figure~\ref{fig:ablation}.
\begin{figure}[h]
\centering
\includegraphics[width=\linewidth]{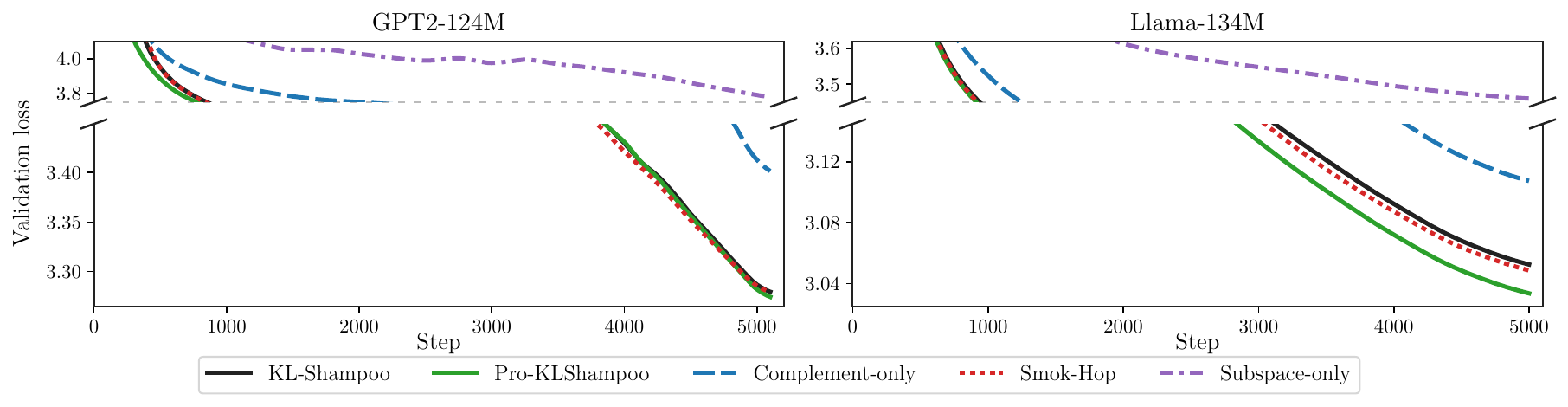}
\caption{Validation loss versus training step on GPT-2 124M (left) and LLaMA 134M (right) for Pro-KLShampoo, KL-Shampoo, and three ablation variants. Subspace-only and complement-only correspond to the spike-and-flat decomposition ablation; Smok-Hop corresponds to the orthogonalization ablation. All Pro-KLShampoo variants run at $r{=}128$.}
\label{fig:ablation}
\end{figure}

\paragraph{Orthogonalization on the complement subspace.}
We compare Pro-KLShampoo against Smok-Hop, the algorithm corresponding to the preconditioned gradient~\eqref{eq:restr_update} without orthogonalization on the complement subspace, isolating orthogonalization's contribution. Table~\ref{tab:ablation} shows Smok-Hop is competitive with KL-Shampoo at the smaller scales but degrades at LLaMA-450M, where the scalar approximation on the complement appears insufficient. Pro-KLShampoo improves over Smok-Hop at every scale; the gap is larger at LLaMA-450M ($0.025$) than at LLaMA-134M ($0.014$), suggesting orthogonalization is more beneficial at the larger LLaMA scale.
\paragraph{Spike-and-flat decomposition.}
We further isolate the two update components: \emph{subspace-only} zeros out the complement update, and \emph{complement-only} zeros out the subspace update. Neither alone matches Pro-KLShampoo or KL-Shampoo (Table~\ref{tab:ablation}, Figure~\ref{fig:ablation}). Complement-only is consistently stronger than subspace-only at both scales. Combining the two recovers the full restricted-KL stationary update, which neither alone can express.

\begin{table}[h]
\centering
\small
\caption{Final validation loss at $r{=}128$. \emph{Subspace-only} and \emph{Complement-only} ablate the spike-and-flat decomposition; \emph{Smok-Hop} ablates orthogonalization. Spike-and-flat ablations on 450M were not run.}
\label{tab:ablation}
\begin{tabular}{lccccc}
\toprule
& Subspace-only & Complement-only & Smok-Hop & Pro-KLShampoo & KL-Shampoo \\
\midrule
GPT-124M    & 3.7803 & 3.4015 & 3.2791 & \textbf{3.2745} & 3.2796 \\
LLaMA-134M  & 3.4598 & 3.1075 & 3.0476 & \textbf{3.0335} & 3.0504 \\
LLaMA-450M  & ---    & ---    & 2.7704 & \textbf{2.7453} & 2.7650 \\
\bottomrule
\end{tabular}
\end{table}

% ============================================================================
\section{Conclusion}
\label{sec:conclusion}

We proposed Pro-KLShampoo, which restricts KL-Shampoo's Kronecker preconditioner to a spike-and-flat structure: full spectral structure on a projected subspace and a single scalar on the complement subspace, with orthogonalization on the complement subspace to recover per-direction whitening. On different pre-training architectures and scales, Pro-KLShampoo improves over KL-Shampoo at every tested rank in validation loss, peak per-GPU memory, and wallclock to matched final loss.

\textbf{Limitations.}
The rank $r$ is set empirically; per-layer or adaptive online calibration of $\alpha_{\mathrm{kl}}$, and a convergence/stability analysis of the EMA-based subspace tracking and preconditioner estimation, are left to future work. Extending the convergence guarantee to the practical variant with Newton--Schulz orthogonalization and Nesterov momentum (Appendix~\ref{app:practical}) remains open.

% \begin{ack}
% \unpolishedstart
% Acknowledgments to fill in
% \unpolishedend
% \end{ack}

% ============================================================================
{
\small
\bibliographystyle{plainnat}
\bibliography{references}
}

% ============================================================================
\appendix
% ============================================================================

\newpage

\section{Related work}
\label{app:related}
 
\paragraph{Kronecker-factored preconditioning.}
Shampoo~\citep{gupta2018shampoo} introduced Kronecker-factored preconditioning as a tractable approximation to full-matrix Adagrad~\citep{duchi2011adaptive}; distributed implementations~\citep{anil2020scalable,shi2023distributed} scaled it to large models, with the Shampoo submission winning the AlgoPerf training-algorithm benchmark. SOAP~\citep{vyas2024soap} improved its practical efficiency by running Adam in Shampoo's eigenbasis. Recent work analyzes Shampoo from multiple angles: convergence theory of one-sided variants~\citep{xie2025structured}, investigation of its heuristics by preconditioner decomposition~\citep{eschenhagen2025purifying}, connection to optimal Kronecker approximation~\citep{morwani2024new}, and broader theory for structured preconditioners and adaptive structured optimization~\citep{an2025asgo}. KL-Shampoo~\citep{lin2025understanding} recasts Shampoo- and SOAP-style estimation through KL minimization, while Clarifying Shampoo~\citep{eschenhagen2026clarifying} sharpens the relationship between Shampoo and Muon-style spectral descent. An earlier line of work, K-FAC~\citep{martens2015optimizing} and its eigenbasis variant E-KFAC~\citep{george2018fast}, predates Shampoo and develops Kronecker-factored curvature approximations for natural-gradient training, though it has seen comparatively slower adoption at LLM scale. Classical Kronecker covariance MLE~\citep{dutilleul1999mle,ros2016existence} provides the estimation-theoretic background.
 
\paragraph{Orthogonalization and norm geometry.}
A complementary line of work interprets matrix-valued updates through norm geometry. Muon~\citep{modded_nanogpt_2024} applies a Newton--Schulz approximation to an orthogonalized momentum update, and subsequent work interprets orthogonalization-style updates through spectral- or modular-norm steepest-descent viewpoints~\citep{bernstein2024old,large2024scalable}; the orthogonalization argument also inherits structure from matrix sign-SGD~\citep{bernstein2018signsgd}, with early roots in spectral descent for deep networks~\citep{carlson2015preconditioned}. We use these geometric perspectives for interpreting the complement step, not for the restricted-KL derivation itself. Vyas et al.~\citep{vyas2025improving} analyze the interaction between SOAP and Muon-style updates through an iterative-whitening lens; Mousse~\citep{zhang2026mousse} reformulates spectral-norm steepest descent in a Kronecker-whitened geometry.
 
\paragraph{Subspace structure in optimization.}
Several recent methods exploit low-rank or subspace structure in gradient updates: GaLore~\citep{zhao2024galore} projects gradients into a low-rank subspace; other approaches include online subspace descent~\citep{liang2024memory}, SubTrack++~\citep{rajabi2025subtrack++}, and GoLore~\citep{he2024subspace}. COSMOS~\citep{liu2025cosmos} decomposes updates into a SOAP-preconditioned subspace and a Muon-treated complement. In our work, the subspace/complement decomposition emerges from restricting the KL objective to a spike-and-flat parametric family (Claims~\ref{claim:stat} and~\ref{claim:opt_sub}), rather than being imposed as an architectural choice.
 
\paragraph{Spectral structure and subspace tracking.}
Our spike-and-flat motivation is related to prior observations that deep-network optimization often concentrates in low-dimensional spectral structure. Gradient Descent Happens in a Tiny Subspace~\citep{gur2018gradient} and related observations on low-rank gradient structure~\citep{jaiswal2024low} support the relevance of low-dimensional gradient subspaces, while Ghorbani et al.~\citep{ghorbani2019investigation} study spectral concentration in Hessian eigenvalues during training. Our use of a spike-and-flat model is also related to classical spiked-covariance theory~\citep{johnstone2001distribution, paul2007asymptotics}. For the subspace-tracking component, we build on the classical online PCA update of Oja~\citep{oja1982simplified}, and finite-sample analyses of streaming PCA~\citep{jain2016streaming,huang2021streaming}, and Davis--Kahan-type eigenspace perturbation tools~\citep{yu2015useful}. Nonconvex stochastic $O(1/\sqrt T)$ convergence is classical~\citep{ghadimi2013stochastic}.

\newpage
\section{Training loss curves}
\label{app:trainloss}

These training loss curves support the wallclock saving mechanism analysis in \S\ref{sec:pretraining}. On GPT-2 (Figure~\ref{fig:gpt-trainloss}), Pro-KLShampoo and KL-Shampoo training losses are close throughout training at every rank, indicating that convergence speed is not the source of wallclock saving. On LLaMA (Figure~\ref{fig:llama-trainloss}), Pro-KLShampoo's training loss lies visibly below KL-Shampoo's throughout training at every rank, indicating that faster convergence is the dominant source.

\begin{figure}[h]
\centering
\includegraphics[width=\linewidth]{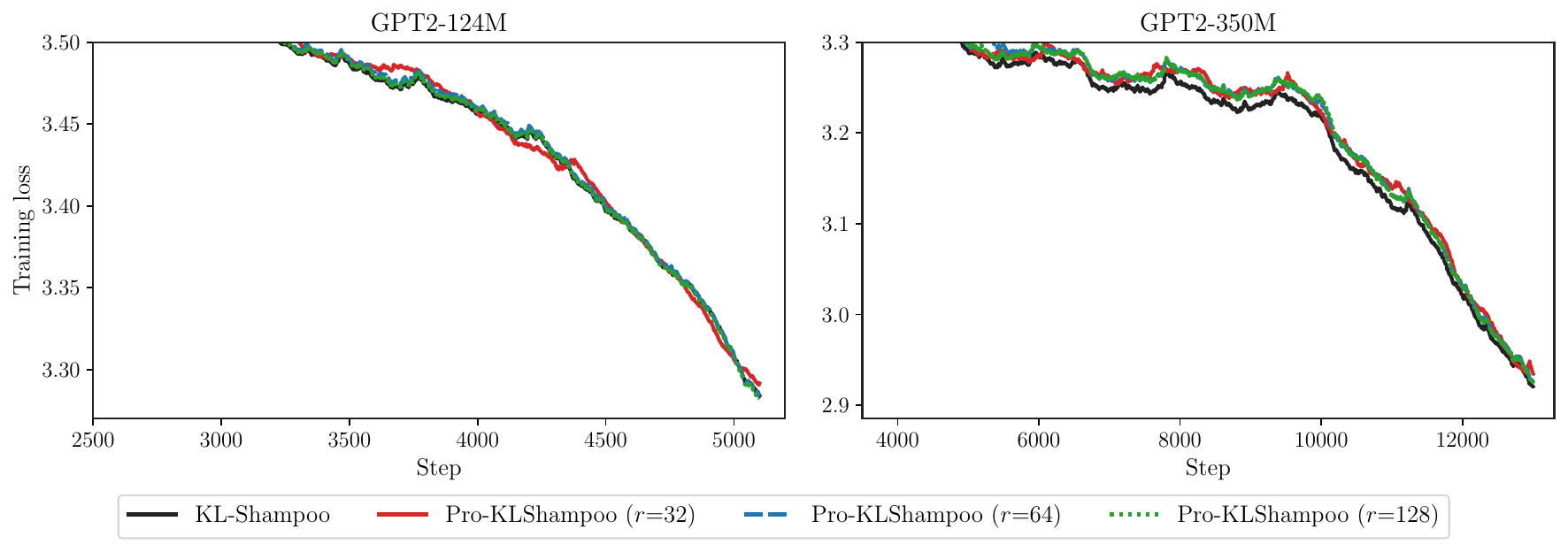}
\caption{Training loss versus training step on GPT-2 (124M, left; 350M, right) for KL-Shampoo and Pro-KLShampoo at $r \in \{32, 64, 128\}$.}
\label{fig:gpt-trainloss}
\end{figure}

\begin{figure}[h]
\centering
\includegraphics[width=\linewidth]{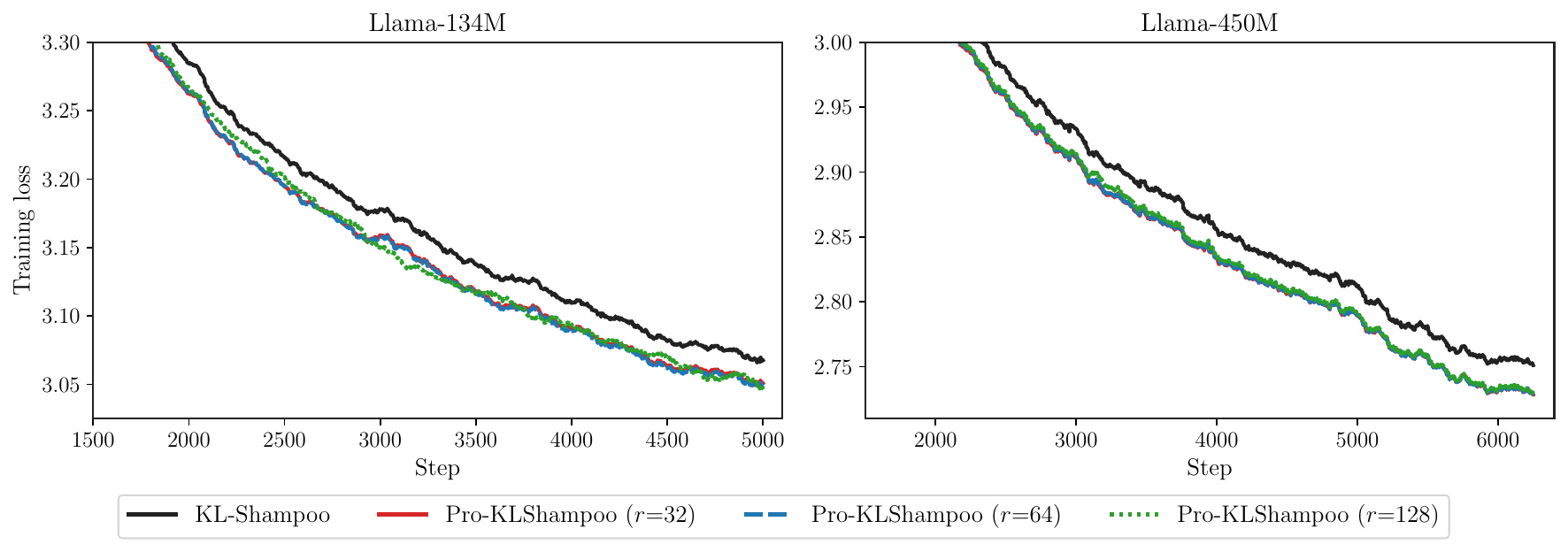}
\caption{Training loss versus training step on LLaMA (134M, left; 450M, right) for KL-Shampoo and Pro-KLShampoo at $r \in \{32, 64, 128\}$.}
\label{fig:llama-trainloss}
\end{figure}

\newpage
\section{Spike-and-flat structure of $\Phi_{\Lr}$}
\label{app:phi_L}

The conjecture in \S\ref{sec:opt_sub} concerns $\Phi_{\Lr}$, the matrix at the restricted-KL stationary point. Since $\Lr$ is unobtainable in practice, we substitute the algorithm's running EMA estimate $L_t$ and approximate $\Phi_{L_t} \coloneqq \E[G^\top L_t^{-1} G]$ by the sample mean over a window of $W = 10$ consecutive minibatches: $\Phi_{L_t} \approx \frac{1}{W}\sum_{i=0}^{W-1} G_{t+i}^\top L_{t+i}^{-1} G_{t+i}$. We repeat this measurement every $250$ training steps. Figure~\ref{fig:phi_L_spike_flat} shows the resulting eigenvalue spectra on GPT-2 (124M).

\begin{figure}[h]
\centering
\includegraphics[width=\linewidth]{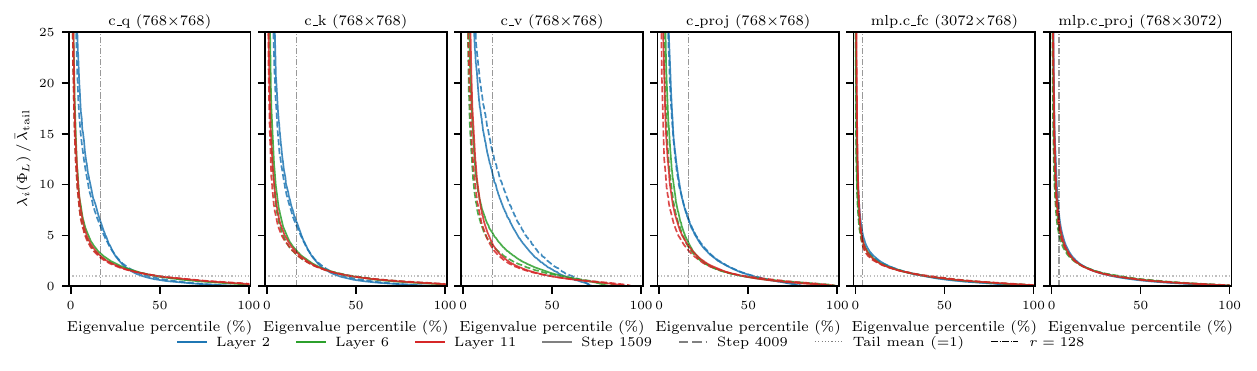}
\caption{Eigenvalue spectra of $\Phi_{L_t}$ on GPT-2 (124M), normalized by the tail mean (vertical dashed line at $r = 128$). The spike-and-flat shape is present across all layer types, depths, and training stages, supporting the conjecture in \S\ref{sec:opt_sub}, though less pronounced than the corresponding spectra of full KL-Shampoo's preconditioner (Figure~\ref{fig:spike_flat}).}
\label{fig:phi_L_spike_flat}
\end{figure}

\section{Spike-and-flat structure on LLaMA}
\label{app:llama_spike}

We confirm the spike-and-flat observation on LLaMA. Figure~\ref{fig:llama_spike_flat} shows the eigenvalue spectra of KL-Shampoo's Kronecker preconditioners during LLaMA training, in the same format as Figure~\ref{fig:spike_flat}.

\begin{figure}[h]
\centering
\includegraphics[width=\linewidth]{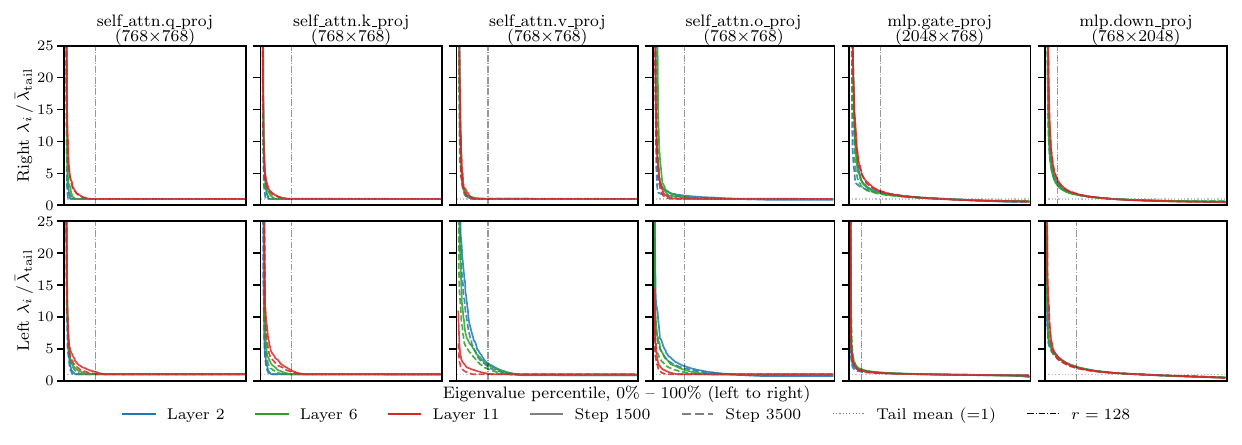}
\caption{% TODO: confirm scale (134M / 450M) and any layer-naming details
Eigenvalue spectra of KL-Shampoo's Kronecker preconditioners on LLaMA, normalized by the tail mean (vertical dashed line at $r = 128$). Both the right-side preconditioner $R$ (top row) and the left-side preconditioner $L$ (bottom row) show a spike-and-flat pattern across all layer types, depths, and training stages, mirroring the structure observed on GPT-2 (Figure~\ref{fig:spike_flat}).}
\label{fig:llama_spike_flat}
\end{figure}

\newpage
\section{Proof of Claim~\ref{claim:gap} (approximation gap)}
\label{app:proof_gap}

\begin{proof}

For the lower bound: for any $U \in \St(n,r)$, $S \in \Spp^r$, and $\mu_\perp > 0$, the matrix $USU^\top + \mu_\perp P_\perp$ belongs to $\Spp^n$.
Therefore the feasible set of~\eqref{eq:restr} is contained in that of~\eqref{eq:full_kl}, which gives $\mathcal{J}^{\mathrm{restr}} \geq \mathcal{J}^{\mathrm{full}}$.

For the upper bound, we exhibit a feasible candidate.
Let $R^* = Q\,\Diag(\mu_1^*,\ldots,\mu_n^*)\,Q^\top$ be the eigendecomposition of the full KL optimum, with eigenvalues ordered $\mu_1^* \geq \cdots \geq \mu_n^* > 0$ and $Q \in \R^{n \times n}$ orthogonal.
Define:
\begin{align}
    U_0 &\coloneqq Q_{:,\,1:r} \in \St(n,r), \label{eq:U0_def} \\
    S_0 &\coloneqq \Diag(\mu_1^*,\ldots,\mu_r^*) \in \Spp^r, \label{eq:S0_def} \\
    \bar\mu &\coloneqq \frac{1}{n-r}\sum_{i=r+1}^n \mu_i^* > 0, \label{eq:mubar_def}
\end{align}
and set
\begin{equation}\label{eq:Rhat_candidate}
    \hat{R} \coloneqq U_0\, S_0\, U_0^\top + \bar\mu\,(I_n - U_0 U_0^\top).
\end{equation}
Since $U_0 \in \St(n,r)$, $S_0 \in \Spp^r$, and $\bar\mu > 0$, the tuple $(L^*, U_0, S_0, \bar\mu)$ is feasible for~\eqref{eq:restr}.
Let $\mathcal{J}(L, R)$ denote the KL objective in~\eqref{eq:full_kl} as a function of the pair $(L, R)$, so that $\mathcal{J}^{\mathrm{full}} = \mathcal{J}(L^*, R^*)$.
Hence
\begin{equation}\label{eq:gap_ub_start}
    \mathcal{J}^{\mathrm{restr}} - \mathcal{J}^{\mathrm{full}} \;\leq\; \mathcal{J}(L^*, \hat{R}) - \mathcal{J}(L^*, R^*).
\end{equation}

\medskip\noindent
We next expand the difference $\mathcal{J}(L^*, \hat R) - \mathcal{J}(L^*, R^*)$.
Using standard Kronecker identities (see, e.g., \citealt{lin2025understanding}, \S2.2), the KL objective expands to
\begin{equation}\label{eq:J_expand}
    \mathcal{J}(L, R) = \tfrac{1}{2}\!\left[n\log\det L + m\log\det R + \Tr\!\bigl(R^{-1}\,\Phi_L\bigr) - \log\det\Sigma - mn\right],
\end{equation}
where $\Phi_L \coloneqq \E[G^\top L^{-1} G] \in \Spp^n$ is the $L$-whitened gradient column second moment.
Since both $\mathcal{J}(L^*, \hat{R})$ and $\mathcal{J}(L^*, R^*)$ share $L = L^*$, the terms $\tfrac{n}{2}\log\det L^*$, $\tfrac{1}{2}\log\det\Sigma$, and $\tfrac{mn}{2}$ in~\eqref{eq:J_expand} cancel, and the difference reduces to
\begin{equation}\label{eq:gap_expand}
    \mathcal{J}(L^*, \hat{R}) - \mathcal{J}(L^*, R^*)
    = \underbrace{\frac{m}{2}\bigl(\log\det\hat{R} - \log\det R^*\bigr)}_{\text{(I)}}
    \;+\; \underbrace{\frac{1}{2}\bigl(\Tr(\hat{R}^{-1}\Phi_{L^*}) - \Tr(R^{*-1}\Phi_{L^*})\bigr)}_{\text{(II)}}.
\end{equation}
We show that (II)~$= 0$ and (I)~$= \frac{m(n-r)}{2}\log\frac{\AM(\mu_{r+1}^*,\ldots,\mu_n^*)}{\GM(\mu_{r+1}^*,\ldots,\mu_n^*)}$, which gives the desired bound.

\medskip\noindent
We compute term (II), the trace difference, first.
By the full stationarity condition~\eqref{eq:full_stat}, $R^* = \frac{1}{m}\,\Phi_{L^*}$, so $\Phi_{L^*} = m\,R^*$.
Substituting:
\begin{equation}\label{eq:trace_sub}
    \Tr(\hat{R}^{-1}\Phi_{L^*}) = m\,\Tr(\hat{R}^{-1}R^*), \qquad
    \Tr(R^{*-1}\Phi_{L^*}) = m\,\Tr(R^{*-1}R^*) = m\,\Tr(I_n) = mn.
\end{equation}
It remains to compute $\Tr(\hat{R}^{-1}R^*)$.
By construction~\eqref{eq:U0_def}--\eqref{eq:Rhat_candidate}, $U_0$ consists of the first $r$ columns of $Q$.
Since $\hat{R}$ is constructed from $R^*$'s eigenbasis, both $Q^\top R^*Q = \Diag(\mu_1^*,\ldots,\mu_n^*)$ and $Q^\top\hat{R}\,Q = \Diag(\mu_1^*,\ldots,\mu_r^*,\bar\mu,\ldots,\bar\mu)$ are diagonal.
Hence $\hat{R}^{-1}R^*$ has eigenvalues $\mu_i^*/\mu_i^* = 1$ for $i \leq r$ and $\mu_i^*/\bar\mu$ for $i > r$, and:
\begin{equation}\label{eq:trace_split}
    \Tr(\hat{R}^{-1}R^*) = \Tr(Q^\top\hat{R}^{-1}R^*\,Q) = r + \frac{1}{\bar\mu}\sum_{i=r+1}^n \mu_i^* = r + \frac{(n-r)\bar\mu}{\bar\mu} = n,
\end{equation}
where we use the definition of $\bar\mu$ in~\eqref{eq:mubar_def}.
Substituting into~\eqref{eq:trace_sub}: $\Tr(\hat{R}^{-1}\Phi_{L^*}) = mn$.
Combined with $\Tr(R^{*-1}\Phi_{L^*}) = mn$, the trace difference (II) equals zero.

\medskip\noindent
We turn to term (I), the log-det difference. By the same diagonalization,
\begin{align}
    \log\det\hat{R} - \log\det R^*
    &= \left[\sum_{i=1}^r \log\mu_i^* + (n{-}r)\log\bar\mu\right] - \sum_{i=1}^n \log\mu_i^* \notag \\
    &= (n{-}r)\log\bar\mu - \sum_{i=r+1}^n \log\mu_i^* \notag \\
    &= (n{-}r)\log\bar\mu - (n{-}r)\log\GM(\mu_{r+1}^*,\ldots,\mu_n^*) \notag \\
    &= (n{-}r)\log\frac{\AM(\mu_{r+1}^*,\ldots,\mu_n^*)}{\GM(\mu_{r+1}^*,\ldots,\mu_n^*)}, \label{eq:logdet_result}
\end{align}
where we use the definition of $\bar\mu$ in~\eqref{eq:mubar_def} and the geometric mean.

\medskip\noindent
Substituting (I) from~\eqref{eq:logdet_result} and (II) $= 0$ into~\eqref{eq:gap_expand}:
\begin{equation}\label{eq:gap_final}
    \mathcal{J}(L^*, \hat{R}) - \mathcal{J}(L^*, R^*) = \frac{m}{2}\cdot(n{-}r)\log\frac{\AM}{\GM} + 0 = \frac{m(n{-}r)}{2}\log\frac{\AM}{\GM}.
\end{equation}
By the AM--GM inequality, $\AM(\mu_{r+1}^*,\ldots,\mu_n^*) \geq \GM(\mu_{r+1}^*,\ldots,\mu_n^*)$ with equality if and only if $\mu_{r+1}^* = \cdots = \mu_n^*$.
Therefore~\eqref{eq:gap_final} is non-negative, and combined with the lower bound shown above this yields~\eqref{eq:gap_bound}.
\end{proof}

\section{Proof of Claim~\ref{claim:stat} (restricted stationarity)}
\label{app:proof_stat}

\begin{proof}
We first compute $\hat R^{-1}$.
Since $\mathrm{range}(U)$ and $\mathrm{range}(P_\perp)$ are orthogonal and span $\R^n$, the inverse of $\hat R = USU^\top + \mu_\perp P_\perp$ can be written blockwise:
\begin{equation}\label{eq:Rhat_inv}
    \hat{R}^{-1} = US^{-1}U^\top + \mu_\perp^{-1}P_\perp.
\end{equation}

\medskip\noindent
We now expand the KL objective.
Substituting~\eqref{eq:Rhat_inv} into the trace term of~\eqref{eq:J_expand} and using trace cyclicity, with $\Phi_L \coloneqq \E[G^\top L^{-1} G] \in \Spp^n$:
\begin{align}
    \Tr\bigl(\hat{R}^{-1}\,\Phi_L\bigr)
    &= \Tr\bigl(US^{-1}U^\top\,\Phi_L\bigr) + \mu_\perp^{-1}\,\Tr\bigl(P_\perp\,\Phi_L\bigr) \notag \\
    &= \Tr\bigl(S^{-1}\,U^\top\Phi_L\,U\bigr) + \mu_\perp^{-1}\,\Tr\bigl(P_\perp\,\Phi_L\bigr). \label{eq:trace_block}
\end{align}

For the log-det term, since $\hat{R}$ has eigenvalues $\{$eigenvalues of $S\} \cup \{\mu_\perp$ with multiplicity $n{-}r\}$:
\begin{equation}\label{eq:logdet_block}
    \log\det\hat{R} = \log\det S + (n{-}r)\log\mu_\perp.
\end{equation}

\medskip\noindent
Isolating from~\eqref{eq:J_expand} the terms involving $(S, \mu_\perp)$:
\begin{equation}\label{eq:g_func}
    g(S, \mu_\perp) = \frac{m}{2}\bigl[\log\det S + (n{-}r)\log\mu_\perp\bigr] + \frac{1}{2}\bigl[\Tr(S^{-1}\,U^\top\Phi_L\,U) + \mu_\perp^{-1}\,\Tr(P_\perp\,\Phi_L)\bigr].
\end{equation}
The remaining terms ($n\log\det L$, $\log\det\Sigma$, $mn$) do not depend on $S$ or $\mu_\perp$.
The objective separates into a matrix part in $S$ and a scalar part in $\mu_\perp$. Setting derivatives to zero:
\begin{align}
    \frac{\partial g}{\partial S^{-1}} &= -\frac{m}{2}\,S + \frac{1}{2}\,U^\top\Phi_L\,U = 0 \notag \\
    \Longrightarrow\quad S^* &= \frac{1}{m}\,U^\top\Phi_L\,U = \frac{1}{m}\,\E[\Gtil^\top L^{-1}\Gtil]. \label{eq:S_opt}
\end{align}
\begin{align}
    \frac{\partial g}{\partial \mu_\perp} &= \frac{m(n-r)}{2\mu_\perp} - \frac{\Tr(P_\perp\,\Phi_L)}{2\mu_\perp^2} = 0 \notag \\
    \Longrightarrow\quad \mu_\perp^* &= \frac{\Tr(P_\perp\,\Phi_L)}{m(n-r)} = \frac{\Tr\bigl(\E[G_\perp^\top L^{-1} G_\perp]\bigr)}{m(n-r)}. \label{eq:mu_opt}
\end{align}

\medskip\noindent
For the optimality of $L$, the terms in $\mathcal{J}$ depending on $L$ are $\frac{n}{2}\log\det L + \frac{1}{2}\Tr(\hat{R}^{-1}\E[G^\top L^{-1}G])$.
Using $\Tr(\hat{R}^{-1}\E[G^\top L^{-1}G]) = \E[\Tr(L^{-1}G\hat{R}^{-1}G^\top)]$ (cyclic property), we differentiate with respect to $L^{-1}$:
\begin{align}
    \frac{\partial\mathcal{J}}{\partial L^{-1}} &= -\frac{n}{2}\,L + \frac{1}{2}\,\E[G\hat{R}^{-1}G^\top] = 0 \notag \\
    \Longrightarrow\quad \Lr &= \frac{1}{n}\,\E[G\,\hat{R}^{-1}\,G^\top]. \label{eq:L_opt}
\end{align}
\end{proof}

\section{Proof of Claim~\ref{claim:opt_sub} (optimal subspace)}
\label{app:proof_opt_sub}

\begin{proof}
The proof has two parts: we first show that every global minimizer of the reduced objective must be an eigenspace of $\Phi_L$ (using first- and second-order optimality), then evaluate the objective on eigenspaces to obtain the AM/GM characterization. Recall $\Phi_L = \E[G^\top L^{-1}G] \in \Spp^n$.

\medskip\noindent
\emph{Reduced objective.}\;
We substitute $S^* = \frac{1}{m}\,U^\top\Phi_L\,U$ and $\mu_\perp^* = \frac{\Tr(P_\perp\Phi_L)}{m(n-r)}$ into $\mathcal{J}$, where these are the Claim~\ref{claim:stat} optima evaluated at the fixed $L$. Using the block decompositions~\eqref{eq:trace_block}--\eqref{eq:logdet_block} from the proof of Claim~\ref{claim:stat}, the trace term becomes
\[
    \Tr(\hat R^{*-1}\Phi_L) = \Tr(S^{*-1}\,U^\top\Phi_L\,U) + \mu_\perp^{*-1}\Tr(P_\perp\Phi_L) = mr + m(n{-}r) = mn,
\]
where the first equality uses $S^{*-1} U^\top\Phi_L U = m\,I_r$ and $\mu_\perp^{*-1}\Tr(P_\perp\Phi_L) = m(n{-}r)$. The log-det term $m\log\det\hat R^*$ expands via~\eqref{eq:logdet_block}:
\begin{align*}
    m\log\det\hat R^* &= m\bigl[\log\det S^* + (n{-}r)\log\mu_\perp^*\bigr].
\end{align*}
Substituting $S^* = \frac{1}{m}\,U^\top\Phi_L\,U$:
\[
    \log\det S^* = \log\det\!\bigl(\tfrac{1}{m}\,U^\top\Phi_L\,U\bigr) = \log\det(U^\top\Phi_L\,U) - r\log m.
\]
Substituting $\mu_\perp^* = \frac{\Tr(P_\perp\Phi_L)}{m(n-r)}$:
\[
    (n{-}r)\log\mu_\perp^* = (n{-}r)\log\Tr(P_\perp\Phi_L) - (n{-}r)\log m - (n{-}r)\log(n{-}r).
\]
Combining:
\[
    m\log\det\hat R^* = m\bigl[\log\det(U^\top\Phi_L U) - r\log m + (n{-}r)\log\Tr(P_\perp\Phi_L) - (n{-}r)\log m - (n{-}r)\log(n{-}r)\bigr].
\]
The $U$-dependent part of $\mathcal{J}$ (up to the positive factor $m/2$) is therefore
\begin{equation}\label{eq:reduced_obj}
    f(U) \;\coloneqq\; \log\det\bigl(U^\top\Phi_L\,U\bigr) \;+\; (n{-}r)\log\Tr(P_\perp\Phi_L).
\end{equation}
The claim reduces to solving
\begin{equation}\label{eq:opt_sub_problem}
    \min_{U \in \St(n,r)} \; f(U).
\end{equation}

\medskip\noindent
\emph{First-order condition.}\;
The Stiefel manifold $\St(n,r) = \{U \in \R^{n \times r} : U^\top U = I_r\}$ is a closed and bounded subset of $\R^{n \times r}$, hence compact. The function $f$ is well-defined and continuous on $\St(n,r)$: the matrix $U^\top\Phi_L\,U$ is positive definite (since $\Phi_L \succ 0$ and $U$ has orthonormal columns, hence full column rank), so $\det(U^\top\Phi_L\,U) > 0$; and $\Tr(P_\perp\Phi_L) > 0$ because $P_\perp$ is a nonzero projection ($r < n$) and $\Phi_L \succ 0$. Therefore the minimum in~\eqref{eq:opt_sub_problem} is attained.

We derive the first-order optimality condition via the Lagrangian for the constraint $U^\top U = I_r$:
\[
    \mathcal{L}(U, \Lambda) \;=\; f(U) \;-\; \tfrac{1}{2}\Tr\bigl(\Lambda(U^\top U - I_r)\bigr), \qquad \Lambda \in \mathrm{Sym}(r).
\]
Setting $\nabla_U \mathcal{L} = 0$ gives $\nabla_U f(U) = U\Lambda$. Computing the Euclidean gradient:
\[
    \nabla_U f(U) = 2\,\Phi_L\,U\bigl(U^\top\Phi_L\,U\bigr)^{-1} - \frac{2(n{-}r)}{\Tr(P_\perp\Phi_L)}\,\Phi_L\,U.
\]
For brevity, write $A \coloneqq U^{*\top}\Phi_L\,U^* \in \Spp^r$ and $\beta \coloneqq (n{-}r)/\Tr(P_\perp\Phi_L) > 0$ at a critical point $U^*$. The KKT system reads
\begin{equation}\label{eq:kkt_opt_sub}
    2\,\Phi_L\,U^*\bigl(A^{-1} - \beta\,I_r\bigr) \;=\; U^*\Lambda.
\end{equation}
Left-multiplying by $U^{*\top}$ gives $\Lambda = 2(I_r - \beta A)$, which is indeed symmetric.

If $A^{-1} - \beta I_r$ is invertible, right-multiplying~\eqref{eq:kkt_opt_sub} by $(A^{-1} - \beta I_r)^{-1}$ and using $\Lambda = 2(I_r - \beta A) = 2A(A^{-1} - \beta I_r)$ gives
\[
    \Phi_L\,U^* \;=\; U^*\,A.
\]
Diagonalize $A = V\Diag(\lambda_1,\ldots,\lambda_r)V^\top$ with orthogonal $V \in \R^{r\times r}$. Then
\[
    \Phi_L\,(U^* V) \;=\; U^*\,A\,V \;=\; (U^* V)\,\Diag(\lambda_1,\ldots,\lambda_r),
\]
so each column of $U^* V$ is an eigenvector of $\Phi_L$. Since $U^* V$ has the same column span as $U^*$, the minimizer $U^*$ spans an eigenspace of $\Phi_L$.

The same reasoning applies if $A^{-1} - \beta I_r$ is singular but $\mathrm{range}(U^*)$ happens to be $\Phi_L$-invariant.

It remains to rule out the possibility that a global minimizer $U^*$ has $A^{-1} - \beta I_r$ singular and $\mathrm{range}(U^*)$ \textbf{not} $\Phi_L$-invariant. We do this by exhibiting a feasible descent direction, contradicting second-order optimality.

\medskip\noindent
\emph{Second-order argument.}\;
Suppose for contradiction that $U^*$ is a global minimizer of~\eqref{eq:opt_sub_problem}, that $A^{-1} - \beta I_r$ is singular, and that $\mathrm{range}(U^*)$ is \emph{not} $\Phi_L$-invariant.

Since $A^{-1} - \beta I_r$ is singular, there exists a unit vector $w \in \R^r$ with $Aw = w/\beta$. From $\Lambda = 2(I_r - \beta A)$ we get
\[
    \Lambda w = 2(I_r - \beta A)w = 2\bigl(w - \beta \cdot w/\beta\bigr) = 0.
\]
Since $\mathrm{range}(U^*)$ is not $\Phi_L$-invariant, there exists a unit vector $u_\perp \perp \mathrm{range}(U^*)$ whose image $\Phi_L u_\perp$ has a nonzero component in $\mathrm{range}(U^*)$:
\[
    \alpha \;\coloneqq\; U^{*\top}\Phi_L\,u_\perp \;\neq\; 0 \quad \in \R^r.
\]
(If no such $u_\perp$ existed, $\Phi_L$ would map $\mathrm{range}(U^*)^\perp$ into itself, and by symmetry of $\Phi_L$ also $\mathrm{range}(U^*)$ into itself---contradicting non-invariance.)

Define $Z \coloneqq u_\perp\,w^\top \in \R^{n \times r}$. We claim $Z$ is a feasible perturbation direction at $U^*$: setting $U(t) \coloneqq U^* + tZ$ and expanding $(U^* + tZ)^\top(U^* + tZ) = I_r + t(U^{*\top}Z + Z^\top U^*) + t^2 Z^\top Z$, the constraint $U(t)^\top U(t) = I_r$ is preserved to first order iff $U^{*\top}Z + Z^\top U^* = 0$. This holds since $U^{*\top}u_\perp = 0$ (because $u_\perp \perp \mathrm{range}(U^*)$).

Recall that $f(U) = \log\det(U^\top\Phi_L\,U) + (n{-}r)\log\Tr(P_\perp\Phi_L)$ is the reduced objective from~\eqref{eq:reduced_obj}. Along the path $U(t)$, define
\[
    A(t) \;\coloneqq\; U(t)^\top\,\Phi_L\,U(t), \qquad s(t) \;\coloneqq\; \Tr(\Phi_L) - \Tr A(t),
\]
so that $f(U(t)) = \log\det A(t) + (n{-}r)\log s(t)$. At $t = 0$ we have $A(0) = A = U^{*\top}\Phi_L\,U^*$ and $s(0) = (n{-}r)/\beta$. We will show that $\frac{d^2}{dt^2}f(U(t))\big|_{t=0} < 0$, meaning $f$ has strictly negative curvature at $U^*$ along $Z$.

We compute $\frac{dA}{dt}\big|_{t=0}$ and $\frac{d^2A}{dt^2}\big|_{t=0}$. Expanding $A(t) = (U^* + tZ)^\top\Phi_L(U^* + tZ) = U^{*\top}\Phi_L U^* + t(Z^\top\Phi_L U^* + U^{*\top}\Phi_L Z) + t^2 Z^\top\Phi_L Z$:
\begin{align}
    \frac{dA}{dt}\bigg|_{t=0} &= Z^\top\Phi_L\,U^* + U^{*\top}\Phi_L\,Z \notag \\
    &= (u_\perp w^\top)^\top\Phi_L\,U^* + U^{*\top}\Phi_L\,(u_\perp w^\top) \notag \\
    &= w\,\underbrace{(u_\perp^\top\Phi_L\,U^*)}_{=\,\alpha^\top} + \underbrace{(U^{*\top}\Phi_L\,u_\perp)}_{=\,\alpha}\,w^\top \notag \\
    &= w\alpha^\top + \alpha w^\top. \label{eq:dAdt}
\end{align}
\begin{align}
    \frac{d^2A}{dt^2}\bigg|_{t=0} &= 2\,Z^\top\Phi_L\,Z = 2\,(u_\perp w^\top)^\top\Phi_L\,(u_\perp w^\top) \notag \\
    &= 2\,w\,(u_\perp^\top\Phi_L\,u_\perp)\,w^\top = 2(u_\perp^\top\Phi_L\,u_\perp)\,ww^\top. \label{eq:d2Adt2}
\end{align}
Note $u_\perp^\top\Phi_L\,u_\perp > 0$ since $\Phi_L \succ 0$.

We now compute $\frac{d^2}{dt^2}\log\det A(t)\big|_{t=0}$ and $\frac{d^2}{dt^2}(n{-}r)\log s(t)\big|_{t=0}$ separately.

\medskip
\noindent\textit{The $\log\det A$ term.}\;
The standard matrix identity gives
\[
    \frac{d^2}{dt^2}\log\det A(t)\bigg|_{t=0} = \Tr\!\left(A^{-1}\,\frac{d^2A}{dt^2}\bigg|_{t=0}\right) - \Tr\!\left(\!\left(A^{-1}\,\frac{dA}{dt}\bigg|_{t=0}\right)^{\!2}\right).
\]
For the first trace, substituting~\eqref{eq:d2Adt2} and using $A^{-1}w = \beta w$ (recall $Aw = w/\beta$):
\begin{align*}
    \Tr\!\left(A^{-1} \cdot 2(u_\perp^\top\Phi_L\,u_\perp)\,ww^\top\right)
    &= 2(u_\perp^\top\Phi_L\,u_\perp)\,\Tr(A^{-1}ww^\top) \\
    &= 2(u_\perp^\top\Phi_L\,u_\perp)\,w^\top A^{-1}w \\
    &= 2(u_\perp^\top\Phi_L\,u_\perp)\,\beta.
\end{align*}
For the second trace, substituting~\eqref{eq:dAdt}:
\[
    A^{-1}\frac{dA}{dt}\bigg|_{t=0} = A^{-1}(w\alpha^\top + \alpha w^\top) = \underbrace{(A^{-1}w)}_{=\,\beta w}\alpha^\top + (A^{-1}\alpha)\,w^\top = \beta\,w\alpha^\top + (A^{-1}\alpha)\,w^\top.
\]
We expand the square of this $r \times r$ matrix and take the trace. There are four terms:
\begin{align*}
    &(\beta\,w\alpha^\top)(\beta\,w\alpha^\top) = \beta^2\,w\,(\alpha^\top w)\,\alpha^\top, &&\Tr = \beta^2\,(\alpha^\top w)^2, \\
    &(\beta\,w\alpha^\top)(A^{-1}\alpha\,w^\top) = \beta\,w\,(\alpha^\top A^{-1}\alpha)\,w^\top, &&\Tr = \beta\,\alpha^\top A^{-1}\alpha, \\
    &(A^{-1}\alpha\,w^\top)(\beta\,w\alpha^\top) = \beta\,(A^{-1}\alpha)(w^\top w)\alpha^\top = \beta\,(A^{-1}\alpha)\alpha^\top, &&\Tr = \beta\,\alpha^\top A^{-1}\alpha, \\
    &(A^{-1}\alpha\,w^\top)(A^{-1}\alpha\,w^\top) = (A^{-1}\alpha)\underbrace{(w^\top A^{-1}\alpha)}_{=\,\beta\,w^\top\alpha}w^\top, &&\Tr = \beta^2\,(\alpha^\top w)^2,
\end{align*}
where we used $w^\top w = 1$ and $w^\top A^{-1}\alpha = (A^{-1}w)^\top\alpha = \beta\,w^\top\alpha$. Summing:
\begin{equation}\label{eq:trB2}
    \Tr\!\left(\!\left(A^{-1}\frac{dA}{dt}\bigg|_{t=0}\right)^{\!2}\right) = 2\beta^2(\alpha^\top w)^2 + 2\beta\,\alpha^\top A^{-1}\alpha.
\end{equation}
The $\log\det$ contribution to $\frac{d^2f}{dt^2}\big|_{t=0}$ is therefore
\begin{equation}\label{eq:logdet_contrib}
    2(u_\perp^\top\Phi_L\,u_\perp)\,\beta \;-\; 2\beta^2(\alpha^\top w)^2 \;-\; 2\beta\,\alpha^\top A^{-1}\alpha.
\end{equation}

\medskip
\noindent\textit{The $(n{-}r)\log s$ term.}\;
Differentiating $s(t) = \Tr(\Phi_L) - \Tr A(t)$:
\begin{align*}
    \frac{ds}{dt}\bigg|_{t=0} &= -\Tr\!\left(\frac{dA}{dt}\bigg|_{t=0}\right) = -\Tr(w\alpha^\top + \alpha w^\top) = -2\,\Tr(w\alpha^\top) = -2\,\alpha^\top w, \\[4pt]
    \frac{d^2s}{dt^2}\bigg|_{t=0} &= -\Tr\!\left(\frac{d^2A}{dt^2}\bigg|_{t=0}\right) = -2(u_\perp^\top\Phi_L\,u_\perp).
\end{align*}
For a scalar function $\frac{d^2}{dt^2}\log s = \frac{1}{s^2}\bigl(s\,\frac{d^2s}{dt^2} - \bigl(\frac{ds}{dt}\bigr)^2\bigr)$. At $t = 0$, $s(0) = (n{-}r)/\beta$, so
\begin{align}
    \frac{d^2}{dt^2}(n{-}r)\log s\bigg|_{t=0}
    &= (n{-}r) \cdot \frac{\frac{n-r}{\beta}\cdot\bigl(-2(u_\perp^\top\Phi_L\,u_\perp)\bigr) - \bigl(-2\alpha^\top w\bigr)^2}{\bigl(\frac{n-r}{\beta}\bigr)^2} \notag \\
    &= (n{-}r) \cdot \frac{-\frac{2(n-r)}{\beta}(u_\perp^\top\Phi_L\,u_\perp) - 4(\alpha^\top w)^2}{\frac{(n-r)^2}{\beta^2}} \notag \\
    &= -2(u_\perp^\top\Phi_L\,u_\perp)\,\beta \;-\; \frac{4\beta^2(\alpha^\top w)^2}{n-r}. \label{eq:logs_contrib}
\end{align}

\medskip
\noindent\textit{Combining and concluding.}\;
Adding~\eqref{eq:logdet_contrib} and~\eqref{eq:logs_contrib}, the $2(u_\perp^\top\Phi_L\,u_\perp)\,\beta$ terms cancel ($+$ from log-det, $-$ from $\log s$):
\begin{equation}\label{eq:hessian_neg}
    \frac{d^2f}{dt^2}\bigg|_{t=0} \;=\; -2\beta\,\alpha^\top A^{-1}\alpha \;-\; 2\beta^2(\alpha^\top w)^2\!\left(1 + \frac{2}{n{-}r}\right) \;<\; 0,
\end{equation}
since $\alpha \neq 0$, $A^{-1} \succ 0$, and $\beta > 0$.

By the second-order necessary condition for equality-constrained optimization~\citep[Theorem~12.5]{nocedal2006numerical}, the second derivative of the Lagrangian $\mathcal{L}(U) = f(U) - \frac{1}{2}\Tr(\Lambda(U^\top U - I_r))$ along the path $U(t)$ must be non-negative at a minimizer. Differentiating $\mathcal{L}(U(t))$ twice:
\[
    \frac{d^2\mathcal{L}}{dt^2}\bigg|_{t=0} = \frac{d^2f}{dt^2}\bigg|_{t=0} - \frac{1}{2}\Tr\!\left(\Lambda\,\frac{d^2(U(t)^\top U(t))}{dt^2}\bigg|_{t=0}\right) = \frac{d^2f}{dt^2}\bigg|_{t=0} - \Tr\!\left(\Lambda\,Z^\top Z\right),
\]
where we used $\frac{d^2}{dt^2}(U^* + tZ)^\top(U^* + tZ)\big|_{t=0} = 2Z^\top Z$. The necessary condition requires $\frac{d^2\mathcal{L}}{dt^2}\big|_{t=0} \geq 0$.
The correction term vanishes: $Z^\top Z = (u_\perp w^\top)^\top(u_\perp w^\top) = ww^\top$, so $\Tr(\Lambda\,ww^\top) = w^\top\Lambda\,w = 0$ (since $\Lambda w = 0$). The condition reduces to $\frac{d^2f}{dt^2}\big|_{t=0} \geq 0$, contradicting~\eqref{eq:hessian_neg}. Hence $U^*$ cannot be a global minimizer, and every minimizer of~\eqref{eq:opt_sub_problem} is an eigenspace of $\Phi_L$.

\medskip\noindent
\emph{Evaluating $f$ on eigenspaces.}\;
Since every global minimizer is an eigenspace of $\Phi_L$, it remains to compare $f$ across the finitely many ($\binom{n}{r}$) eigenspace choices and find the best one. Denote the eigenvalues of $\Phi_L$ by $\phi_1 \geq \cdots \geq \phi_n > 0$. Let $I \subset \{1,\ldots,n\}$ with $|I| = r$ be an eigenspace choice and $J \coloneqq I^c = \{1,\ldots,n\} \setminus I$ its complement ($|J| = n{-}r$). Abusing notation, we write $f(I)$ for the value of $f$ at any $U$ whose columns span the eigenspace indexed by $I$. At such $U$, $U^\top\Phi_L U = \Diag(\phi_i : i \in I)$ and $\Tr(P_\perp\Phi_L) = \sum_{j \in J}\phi_j$, so
\[
    f(I) \;=\; \log\prod_{i \in I}\phi_i \;+\; (n{-}r)\log\sum_{j \in J}\phi_j.
\]
Recall the arithmetic and geometric means of the complement eigenvalues:
\[
    \AM_J \coloneqq \frac{1}{n{-}r}\sum_{j \in J}\phi_j, \qquad \GM_J \coloneqq \Bigl(\prod_{j \in J}\phi_j\Bigr)^{1/(n-r)}.
\]
We rewrite $f(I)$ in terms of these. For the first term, split the product over all indices:
\[
    \log\prod_{i \in I}\phi_i = \log\prod_{i=1}^n\phi_i - \log\prod_{j \in J}\phi_j = \sum_{i=1}^n\log\phi_i - (n{-}r)\log\GM_J.
\]
For the second term, substitute $\sum_{j \in J}\phi_j = (n{-}r)\,\AM_J$:
\[
    (n{-}r)\log\sum_{j \in J}\phi_j = (n{-}r)\log\bigl((n{-}r)\,\AM_J\bigr) = (n{-}r)\log(n{-}r) + (n{-}r)\log\AM_J.
\]
Combining, with $C_0 \coloneqq \sum_{i=1}^n\log\phi_i + (n{-}r)\log(n{-}r)$ independent of the choice $I$:
\begin{equation}\label{eq:f_amgm}
    f(I) \;=\; C_0 \;+\; (n{-}r)\bigl(\log\AM_J - \log\GM_J\bigr) \;=\; C_0 \;+\; (n{-}r)\log\frac{\AM_J}{\GM_J}.
\end{equation}
Since $C_0$ does not depend on $I$, the minimizer over eigenspace choices is
\[
    I^* \;\in\; \argmin_{|I|=r}\; \frac{\AM_J}{\GM_J},
\]
which is the characterization stated in Claim~\ref{claim:opt_sub}. In particular, when $\phi_{r+1} = \cdots = \phi_n$, the complement $J = \{r{+}1,\ldots,n\}$ achieves $\AM_J/\GM_J = 1$ (the minimum possible value), giving $I^* = \{1,\ldots,r\}$---the top-$r$ eigenspace. More generally, whenever the bottom-$(n{-}r)$ eigenvalues have the smallest AM/GM ratio among all $(n{-}r)$-subsets (this is the spike-and-flat condition in Claim~\ref{claim:opt_sub}), the same conclusion $I^* = \{1,\ldots,r\}$ follows directly from the $\argmin$ characterization above.
\end{proof}

\section{Spike-and-flat exactness}
\label{app:spike_proof_remarks}

\begin{lemma}[Exactness of spike-and-flat under rank-$\rho$ signal plus uncorrelated noise]\label{lem:spike_exact}
Suppose $G = AB^\top + \xi$ where $A \in \R^{m\times \rho}$, $B \in \R^{n\times \rho}$ are deterministic with $\rho < \min(m,n)$, $\xi$ has i.i.d.\ mean-zero entries with variance $\sigma^2 > 0$.
Then any KL stationary point~\eqref{eq:full_stat} satisfies $L^* = S_L^* + f_L^* I_m$ and $R^* = S_R^* + f_R^* I_n$ where $S_L^* \succeq 0$, $S_R^* \succeq 0$, $\mathrm{rank}(S_L^*) \leq \rho$, $\mathrm{rank}(S_R^*) \leq \rho$, and $f_L^*, f_R^* > 0$.
\end{lemma}

\begin{corollary}[Zero KL gap]\label{cor:zero_gap}
Under the model of Lemma~\ref{lem:spike_exact}, $\mathcal{J}^{\mathrm{restr}} = \mathcal{J}^{\mathrm{full}}$ for any $r \geq \rho$.
\end{corollary}

\begin{proof}[Proof of Lemma~\ref{lem:spike_exact}]
Substitute $G = AB^\top + \xi$ into the L stationarity~\eqref{eq:full_stat} and use $\E[\xi] = 0$ to eliminate the cross-product terms:
\begin{align*}
    \E[G\,(R^*)^{-1}\,G^\top]
    &\;=\; AB^\top (R^*)^{-1} BA^\top \;+\; \E[\xi (R^*)^{-1} \xi^\top].
\end{align*}
For the noise contribution, write the $(i,i')$ entry explicitly:
\begin{align*}
    \big(\xi (R^*)^{-1} \xi^\top\big)_{i,i'}
    &\;=\; \sum_{j,j'} \xi_{i,j}\,(R^*)^{-1}_{j,j'}\,\xi_{i',j'}.
\end{align*}
Taking expectation, the uncorrelated-noise assumption $\E[\xi_{ij}\xi_{i'j'}] = \sigma^2 \mathbf{1}_{(i,j)=(i',j')}$ collapses the double sum to its diagonal:
\begin{align*}
    \E\!\big[(\xi (R^*)^{-1} \xi^\top)_{i,i'}\big]
    &\;=\; \sigma^2\,\mathbf{1}_{i=i'}\,\sum_{j=1}^{n}(R^*)^{-1}_{j,j}
    \;=\; \sigma^2\,\Tr((R^*)^{-1})\,\mathbf{1}_{i=i'}.
\end{align*}
Hence $\E[\xi (R^*)^{-1} \xi^\top] = \sigma^2\,\Tr((R^*)^{-1})\,I_m$, and the L stationarity becomes
\begin{equation}\label{eq:L_decomp_spike}
    L^* \;=\; \underbrace{\frac{1}{n}\,A\,\big(B^\top (R^*)^{-1} B\big)\,A^\top}_{=:\,S_L^*} \;+\; \underbrace{\frac{\sigma^2\,\Tr((R^*)^{-1})}{n}}_{=:\,f_L^*}\,I_m.
\end{equation}
$S_L^*$ is PSD because $B^\top (R^*)^{-1} B \succeq 0$ (sandwich of $(R^*)^{-1} \succ 0$); its column space lies in $\mathrm{range}(A)$ by construction; and $\mathrm{rank}(S_L^*) \leq \mathrm{rank}(A) \leq \rho$.
The scalar $f_L^*$ is strictly positive because $R^* \succ 0$ implies $\Tr((R^*)^{-1}) > 0$.
This establishes the decomposition $L^* = S_L^* + f_L^*\,I_m$ with $\mathrm{rank}(S_L^*) \leq \rho$ and $f_L^* > 0$.

The same calculation applied to the R stationarity yields
\begin{equation}\label{eq:R_decomp_spike}
    R^* \;=\; \underbrace{\frac{1}{m}\,B\,\big(A^\top (L^*)^{-1} A\big)\,B^\top}_{=:\,S_R^*} \;+\; \underbrace{\frac{\sigma^2\,\Tr((L^*)^{-1})}{m}}_{=:\,f_R^*}\,I_n,
\end{equation}
which establishes the analogous decomposition $R^* = S_R^* + f_R^*\,I_n$ with $S_R^* \succeq 0$, $\mathrm{rank}(S_R^*) \leq \rho$, and $f_R^* > 0$.
\end{proof}

\begin{proof}[Proof of Corollary~\ref{cor:zero_gap}]
By Lemma~\ref{lem:spike_exact}, $R^* = S_R^* + f_R^*\,I_n$ with $S_R^* \succeq 0$ and $\mathrm{rank}(S_R^*) \leq \rho$. Since $S_R^*$ has at least $n - \rho$ zero eigenvalues, $R^*$ has at least $n - \rho$ eigenvalues equal to $f_R^*$.
For any $r \geq \rho$, choose $U \in \St(n,r)$ so that $\mathrm{range}(S_R^*) \subseteq \mathrm{range}(U)$ (possible since $\mathrm{rank}(S_R^*) \leq \rho \leq r$). Then $P_\perp S_R^* = 0$, so on the complement $P_\perp R^* P_\perp = f_R^*\,P_\perp$. Setting $S^* \coloneqq U^\top R^*\,U$ and $\mu_\perp^* \coloneqq f_R^*$, the decomposition~\eqref{eq:spike_flat} satisfies $US^*U^\top + \mu_\perp^* P_\perp = R^*$ exactly, so the restricted family contains the full KL optimum.
Hence $\mathcal{J}^{\mathrm{restr}} \leq \mathcal{J}(L^*, R^*) = \mathcal{J}^{\mathrm{full}}$, and the reverse inequality is trivial.
\end{proof}

\paragraph{Why exact, not just asymptotic.}
The decomposition is exact at any finite $m$, $n$, and $\sigma > 0$: the identity $\E[\xi (R^*)^{-1} \xi^\top] = \sigma^2\Tr((R^*)^{-1})\,I_m$ used in the proof holds at the population level. This contrasts with spiked random matrix theory~\citep{johnstone2001distribution, paul2007asymptotics}, where the spike-plus-bulk structure of the \emph{empirical} covariance is recovered only asymptotically as $m, n \to \infty$ at fixed ratio. Lemma~\ref{lem:spike_exact} therefore says: under signal-plus-uncorrelated-noise gradients, the population-level KL stationary point is \emph{precisely} of spike-and-flat structure regardless of dimension.

\paragraph{Beyond i.i.d.\ noise.}
Lemma~\ref{lem:spike_exact} establishes spike-and-flat exactly under rank-$\rho$ signal plus i.i.d.\ noise. We examine here what changes under more general noise. The flat floor in Lemma~\ref{lem:spike_exact} originates from the identity $\E[\xi R^{-1}\xi^\top] = \sigma^2\Tr(R^{-1})\,I_m$, which is a scalar multiple of identity exactly because $\E[\xi_{ij}\,\xi_{i'j'}] = \sigma^2\,\mathbf{1}_{(i,j)=(i',j')}$. Under separable noise covariance $\E[\mathrm{vec}(\xi)\mathrm{vec}(\xi)^\top] = \Sigma_L^\xi \otimes \Sigma_R^\xi$ --- equivalently, $\E[\xi_{ij}\,\xi_{i'j'}] = (\Sigma_L^\xi)_{ii'}\,(\Sigma_R^\xi)_{jj'}$ --- the same calculation gives
\begin{align*}
\E\!\big[(\xi R^{-1} \xi^\top)_{i,i'}\big]
&\;=\; \sum_{j,j'}(R^{-1})_{j,j'}\,\E[\xi_{ij}\,\xi_{i'j'}]
\;=\; (\Sigma_L^\xi)_{ii'}\,\sum_{j,j'}(R^{-1})_{j,j'}\,(\Sigma_R^\xi)_{jj'} \\
&\;=\; (\Sigma_L^\xi)_{ii'}\,\Tr(R^{-1}\,\Sigma_R^\xi),
\end{align*}
where the last step uses symmetry of $\Sigma_R^\xi$. Hence $\E[\xi R^{-1} \xi^\top] = \Tr(R^{-1}\Sigma_R^\xi)\,\Sigma_L^\xi$, and the L stationarity becomes $L^* = S_L^* + f_L^*\,\Sigma_L^\xi$ with $f_L^* \coloneqq \Tr((R^*)^{-1}\Sigma_R^\xi)/n$. The low-rank spike $S_L^*$ persists, but the floor $f_L^*\,I_m$ becomes $f_L^*\,\Sigma_L^\xi$ and inherits the spectrum of $\Sigma_L^\xi$. The flat-tail structure in Lemma~\ref{lem:spike_exact} thus requires noise isotropy; under more general noise the structure becomes spike-plus-shaped-floor.

\newpage
% ============================================================================
\section{Proofs and additional material for \S\ref{sec:naive_conv}}
\label{app:proof_naive_conv_sec}
% ============================================================================

This appendix states and proves the convergence result and the supporting lemmas referenced in the main text.

\paragraph{Notation reminder.}
Throughout~\S\ref{app:proof_naive_conv_sec}, the algorithm state at step $t$ comprises $L_t \in \Spp^m$, $S_t \in \Spp^r$, $U_t \in \St(n,r)$ (subspace basis with orthonormal columns), and $\mu_{\perp,t} > 0$. We write $P_{\perp,t} \coloneqq I_n - U_t U_t^\top$ for the orthogonal projection onto $\mathrm{range}(U_t)^\perp$, and $\hat R_t \coloneqq U_t S_t U_t^\top + \mu_{\perp,t} P_{\perp,t}$ for the right-side preconditioner. The stochastic gradient is $G_t$ with mean $\nabla f(W_t)$. The polar factor of a matrix $M = U_M \Sigma_M V_M^\top$ (SVD) is $\polar(M) \coloneqq U_M V_M^\top$; it satisfies $\|\polar(M)\|_{\mathrm{op}} \leq 1$ (it is a \emph{partial isometry}: an SVD with all nonzero singular values equal to $1$). We use $\|\cdot\|_F$ for the Frobenius norm, $\|\cdot\|_{\mathrm{op}}$ for the operator (spectral) norm, and $\|\cdot\|_* = \sum_i \sigma_i$ for the nuclear (trace) norm.

\begin{lemma}[Preconditioner bounds]\label{lem:clamp_lb}
Let $C > 0$ be the eigenvalue clip threshold.
Assume that after each EMA update, a per-eigenvalue clip enforces $\lambda_{L,i}, \lambda_{S,j}, \mu_{\perp} \geq 1/C^2$.
Under $\|G_t\|_{\mathrm{op}} \leq G_{\max}$ (Assumption~\ref{ass:bdg} below), set $\Theta \coloneqq \max(\|L_0\|_{\mathrm{op}},\|S_0\|_{\mathrm{op}},\mu_{\perp,0},C^2 G_{\max}^2)$.
Then for all $t \geq 0$:
\begin{equation}\label{eq:spectrum_bounds}
    \tfrac{1}{C^2} \;\leq\; \lambda_{L,i,t},\;\lambda_{S,j,t},\;\mu_{\perp,t} \;\leq\; \Theta
    \qquad \text{for all } i \in \{1,\ldots,m\},\; j \in \{1,\ldots,r\}.
\end{equation}
\end{lemma}

\begin{proof}
The lower bound $\geq 1/C^2$ holds by the assumption.

For the upper bound, we proceed by induction on $t$. At $t = 0$, $\Theta \geq \|L_0\|_{\mathrm{op}}, \|S_0\|_{\mathrm{op}}, \mu_{\perp,0}$ by definition. Assume the bound holds at step $t$; we show it holds at step $t+1$. Following Algorithm~\ref{alg:ideal}, we omit the step index $t$ on all quantities below.

The eigenvalue EMA updates in Step~4a read (with $\beta_2 \in (0,1)$ the EMA coefficient and $Q_L \in \R^{m\times m}$, $Q_S \in \R^{r\times r}$ the eigenvector matrices from Step~3 of Algorithm~\ref{alg:ideal}):
\begin{align*}
    \lambda_{L,i} \;&\gets\; \beta_2\,\lambda_{L,i} + (1{-}\beta_2)\,\bigl(Q_L^\top \Delta_L\, Q_L\bigr)_{ii},\\
    \lambda_{S,j} \;&\gets\; \beta_2\,\lambda_{S,j} + (1{-}\beta_2)\,\bigl(Q_S^\top \Delta_S\, Q_S\bigr)_{jj},\\
    \mu_\perp &\gets \beta_2\,\mu_\perp + (1{-}\beta_2)\,\delta_\perp,
\end{align*}
where $\Delta_L, \Delta_S, \delta_\perp$ are the covariance targets from Step~2:
\begin{align*}
    \Delta_L &= \tfrac{1}{n}\bigl(\Gtil\,Q_S\Diag(\lambda_S^{\odot -1})Q_S^\top\,\Gtil^\top \;+\; \mu_\perp^{-1}\,G_\perp\,G_\perp^\top\bigr), \\
    \Delta_S &= \tfrac{1}{m}\,\Gtil^\top\,Q_L\Diag(\lambda_L^{\odot -1})Q_L^\top\,\Gtil, \\
    \delta_\perp &= \tfrac{1}{m(n{-}r)}\,\Tr\bigl(G_\perp^\top\,Q_L\Diag(\lambda_L^{\odot -1})Q_L^\top\,G_\perp\bigr),
\end{align*}
where $\lambda_S^{\odot -1} \coloneqq (1/\lambda_{S,1},\ldots,1/\lambda_{S,r})$ denotes the vector of componentwise reciprocals (and similarly $\lambda_L^{\odot -1}$), $\Gtil = GU$ and $G_\perp = G - \Gtil U^\top = G(I_n - UU^\top)$ as in Step~1 of Algorithm~\ref{alg:ideal}. 

Since $\Delta_L$ and $\Delta_S$ are PSD, each diagonal entry of the rotated matrix $Q_L^\top \Delta_L Q_L$ is at most $\|\Delta_L\|_{\mathrm{op}}$ (because a diagonal entry of $Q_L^\top \Delta_L Q_L$ equals $q_i^\top \Delta_L\, q_i \leq \|\Delta_L\|_{\mathrm{op}}$, where $q_i$ is the $i$-th column of $Q_L$), and similarly for $Q_S^\top \Delta_S Q_S$. It therefore suffices to bound $\|\Delta_L\|_{\mathrm{op}}$, $\|\Delta_S\|_{\mathrm{op}}$, and $\delta_\perp$.

By the clip assumption, $\lambda_{S,j} \geq 1/C^2$ for all $j$, so $\lambda_{S,j}^{-1} \leq C^2$ and $\|Q_S\Diag(\lambda_S^{\odot -1})Q_S^\top\|_{\mathrm{op}} = \max_j \lambda_{S,j}^{-1} \leq C^2$. Since $U$ has orthonormal columns, $\|\Gtil\|_{\mathrm{op}} = \|GU\|_{\mathrm{op}} \leq \|G\|_{\mathrm{op}}\,\|U\|_{\mathrm{op}} \leq G_{\max}$ (Assumption~\ref{ass:bdg}). Similarly, $\|G_\perp\|_{\mathrm{op}} = \|G(I_n - UU^\top)\|_{\mathrm{op}} \leq G_{\max}$. The clip also gives $\mu_\perp^{-1} \leq C^2$ and $\lambda_{L,i}^{-1} \leq C^2$ for all $i$. Combine all the bounds:
\begin{align*}
    \|\Delta_L\|_{\mathrm{op}} &\leq \tfrac{1}{n}\bigl(G_{\max}^2\,C^2 + C^2\,G_{\max}^2\bigr) = \tfrac{2C^2 G_{\max}^2}{n} \leq C^2 G_{\max}^2 \leq \Theta, \\
    \|\Delta_S\|_{\mathrm{op}} &\leq \tfrac{1}{m}\,G_{\max}^2\,C^2 \leq C^2 G_{\max}^2 \leq \Theta, \\
    \delta_\perp &\leq \tfrac{C^2\,G_{\max}^2\,(n{-}r)}{m(n{-}r)} = \tfrac{C^2 G_{\max}^2}{m} \leq C^2 G_{\max}^2 \leq \Theta,
\end{align*}
where for $\delta_\perp$ we used $\|G_\perp\|_F^2 = \Tr(G_\perp G_\perp^\top) \leq \|G\|_{\mathrm{op}}^2\,\Tr(I_n - UU^\top) = G_{\max}^2(n{-}r)$.

Each post-EMA eigenvalue is a convex combination of two quantities bounded by $\Theta$:
\[
    \beta_2\,\lambda_{L,i} + (1{-}\beta_2)\,(Q_L^\top\Delta_L Q_L)_{ii} \;\leq\; \beta_2\,\Theta + (1{-}\beta_2)\,\Theta \;=\; \Theta,
\]
and similarly for $\lambda_{S,j}$ and $\mu_\perp$. The clip can only raise eigenvalues (from below $1/C^2$ up to $1/C^2$), and $1/C^2 \leq \Theta$ (since $\Theta \geq \|L_0\|_{\mathrm{op}} \geq \lambda_{\min}(L_0) \geq 1/C^2$ by the clip assumption at $t=0$), so the upper bound is preserved after clipping. This completes the induction.
\end{proof}

Let $\mathcal{F}_t \coloneqq \sigma(G_0, \ldots, G_{t-1})$ denote the natural filtration (all randomness up to but not including step $t$); in particular, $L_t, S_t, U_t, \mu_{\perp,t}$ are $\mathcal{F}_t$-measurable.
We assume:
\begin{enumerate}[leftmargin=*,nosep,label=(\roman*)]
\item\label{ass:lb} Lower boundedness: $f^* \coloneqq \inf_W f(W) > -\infty$.
\item\label{ass:smooth_op} $L_{\mathrm{op}}$-operator-norm smoothness: $f(W')\leq f(W)+\langle\nabla f(W),W'-W\rangle + \tfrac{L_{\mathrm{op}}}{2}\|W'-W\|_{\mathrm{op}}^2$.
\item\label{ass:noise} Unbiased stochastic gradients with bounded Frobenius variance: $\E[\|G_t-\nabla f(W_t)\|_F^2\mid\mathcal F_t]\leq\sigma_F^2$.
\item\label{ass:bdg} Bounded stochastic gradient: $\|G_t\|_{\mathrm{op}} \leq G_{\max}$ a.s.
\item\label{ass:clip} Eigenvalue clip: there is a constant $C > 0$ such that $\lambda_{L,i,t}, \lambda_{S,j,t}, \mu_{\perp,t} \geq 1/C^2$ for all $t \geq 0$, $i \in \{1,\ldots,m\}$, $j \in \{1,\ldots,r\}$.
\end{enumerate}
Operator-norm smoothness is the natural smoothness model for matrix-parameter layers, where updates act as operators on activations~\citep{bernstein2024old,large2024scalable}.

\subsection{Main convergence theorem}
\label{app:polar_conv}

This subsection contains the convergence result for Pro-KLShampoo (Algorithm~\ref{alg:ideal}). Smok-Hop (without polar orthogonalization) is analyzed separately as a technical companion in~\S\ref{app:naive_conv_sub} below.

\subsubsection{Single-step descent inequality}
\label{app:single_step}

We first state the per-step descent inequality from which Theorem~\ref{thm:polar_conv} is obtained by telescoping. The constant $\sigma_{kl}^2$, defined in Theorem~\ref{thm:polar_conv} of the main text, satisfies $\E[\|L_t^{-1/2}G_t U_t S_t^{-1/2}\|_{\mathrm{op}}^2 \mid \mathcal F_t] \leq \sigma_{kl}^2$ a.s. for all $t \geq 0$ by construction (see Remark~\ref{rem:sigma_kl} for reference values).
Throughout, we write $W_{t+1} = W_t + \eta\,\Delta W_t$ for the Pro-KLShampoo update from line~\ref{line:comp} of Algorithm~\ref{alg:ideal}. By the scale invariance $\polar(\alpha M) = \polar(M)$ for $\alpha > 0$ and $\mu_{\perp,t} > 0$, the complement term in~\eqref{eq:update} simplifies to $\polar(L_t^{-1/2}\,G_t\,P_{\perp,t})$, so $\Delta W_t$ takes the form
\[
    \Delta W_t \;=\; -\alpha_{\mathrm{kl}}\,L_t^{-1/2}\,G_t\,U_t\,S_t^{-1/2}\,U_t^\top \;-\; c_a\,\polar\bigl(L_t^{-1/2}\,G_t\,P_{\perp,t}\bigr).
\]

\begin{proposition}[Descent inequality]\label{prop:polar_descent}
Under Assumptions~\ref{ass:lb}--\ref{ass:clip}, Algorithm~\ref{alg:ideal} satisfies for every $t\geq 0$:
\begin{align}\label{eq:polar_descent}
\E\!\left[f(W_{t+1}) - f(W_t)\,\big|\,\mathcal F_t\right]
\;\le\;
&-\tfrac{\eta\,c_a}{C}\,\bigl\|L_t^{-1/2}\nabla f(W_t)\,P_{\perp,t}\bigr\|_*
\;-\; \tfrac{\eta\,\alpha_{\mathrm{kl}}}{\sqrt{\Theta}}\,\bigl\|L_t^{-1/4}\nabla f(W_t)\,U_t\bigr\|_F^2 \notag\\
&+\; 2\eta\,c_a\sqrt{k}\,\sigma_F
\;+\; \eta^2 L_{\mathrm{op}}\bigl(c_a^2 + \alpha_{\mathrm{kl}}^2\,\sigma_{kl}^2\bigr),
\end{align}
where $c_a = \sqrt{\max(1,m/n)}$ and $k = \min(m,n{-}r)$.
The complement descent is measured by the nuclear norm $\|M\|_* \coloneqq \sum_i \sigma_i(M)$ (from the identity $\langle M,\polar(M)\rangle = \|M\|_*$); the subspace descent is measured by the Frobenius norm squared.
\end{proposition}

\subsubsection{Proof of Proposition~\ref{prop:polar_descent}}
\label{app:proof_polar_descent}

\begin{proof}[Proof of Proposition~\ref{prop:polar_descent}]
We define several shorthand quantities. The noise is $N_t \coloneqq G_t - \nabla f(W_t)$, and:
\begin{alignat*}{2}
    A_t &\coloneqq L_t^{-1/2}\,\nabla f(W_t)\,P_{\perp,t} \quad&&\text{(complement gradient, with $L_t^{-1/2}$ on the left),} \\
    Z_t &\coloneqq L_t^{-1/2}\,N_t\,P_{\perp,t} \quad&&\text{(complement noise, with $L_t^{-1/2}$ on the left),} \\
    V_t &\coloneqq L_t^{-1/4}\,\nabla f(W_t)\,U_t \quad&&\text{(subspace gradient, with $L_t^{-1/4}$ on the left).}
\end{alignat*}
All three are $\mathcal F_t$-measurable. Note that $L_t^{-1/2}G_t P_{\perp,t} = A_t + Z_t$. We write $\langle X, Y\rangle \coloneqq \Tr(X^\top Y)$ for the Frobenius (trace) inner product on matrices.

\medskip
\noindent\textbf{Step 0.}\;
By operator-norm smoothness (Assumption~\ref{ass:smooth_op}) and $W_{t+1} = W_t + \eta\,\Delta W_t$:
\begin{equation}\label{eq:polar_step0}
    f(W_{t+1}) - f(W_t) \;\leq\; \eta\,\langle\nabla f(W_t),\,\Delta W_t\rangle \;+\; \tfrac{\eta^2 L_{\mathrm{op}}}{2}\,\|\Delta W_t\|_{\mathrm{op}}^2.
\end{equation}
From~\eqref{eq:update},
\[
    \Delta W_t = -\alpha_{\mathrm{kl}}\,L_t^{-1/2}\,G_t\,U_t\,S_t^{-1/2}\,U_t^\top \;-\; c_a\,\polar\!\bigl(L_t^{-1/2}\,G_t\,P_{\perp,t}\bigr),
\]
so $\langle\nabla f(W_t),\Delta W_t\rangle$ splits into a subspace piece and a complement piece. Steps~1--3 below bound the conditional expectations of the two inner-product pieces and the second-order term $\|\Delta W_t\|_{\mathrm{op}}^2$.

\medskip
\noindent\textbf{Step 1.}\;
We prove the lower bound
\begin{equation}\label{eq:polar_lowerbd}
    \E\bigl[\langle\nabla f(W_t),\,c_a\polar(L_t^{-1/2}G_t P_{\perp,t})\rangle\bigm|\mathcal F_t\bigr] \;\geq\; \tfrac{c_a}{C}\,\|A_t\|_* \;-\; 2c_a\sqrt{k}\,\sigma_F.
\end{equation}
Write $M \coloneqq A_t + Z_t = L_t^{-1/2}G_t P_{\perp,t}$. Since $P_{\perp,t}$ is a projection, the SVD of $M$ has right singular vectors in $\mathrm{range}(P_{\perp,t})$, so $\polar(M)\,P_{\perp,t} = \polar(M)$. Using $P_{\perp,t}^\top = P_{\perp,t}$ and trace cyclicity:
\begin{align*}
    \langle\nabla f(W_t),\,\polar(M)\rangle
    \;&=\; \langle\nabla f(W_t)\,P_{\perp,t},\,\polar(M)\rangle \\
    \;&=\; \langle L_t^{1/2}\,A_t,\,\polar(M)\rangle \\
    \;&=\; \underbrace{\langle L_t^{1/2}\,M,\,\polar(M)\rangle}_{=:\,\spadesuit_t} \;-\; \underbrace{\langle L_t^{1/2}\,Z_t,\,\polar(M)\rangle}_{=:\,\diamondsuit_t},
\end{align*}
where the last step uses $A_t = M - Z_t$.

\emph{Bounding $\spadesuit_t$ from below.}\;
Let $M = \sum_i \sigma_i\,u_i v_i^\top$ be the SVD, so $\polar(M) = \sum_i u_i v_i^\top$. Then
\[
    \spadesuit_t \;=\; \sum_i \sigma_i\,(u_i^\top L_t^{1/2} u_i)
    \;\geq\; \sum_i \sigma_i \cdot \tfrac{1}{C}
    \;=\; \tfrac{1}{C}\,\|M\|_*,
\]
where we used $u_i^\top L_t^{1/2} u_i \geq \sqrt{\lambda_{\min}(L_t)} \geq 1/C$ (Assumption~\ref{ass:clip}). Since the target~\eqref{eq:polar_lowerbd} is in terms of $\|A_t\|_*$ rather than $\|M\|_*$, we use the nuclear-norm triangle inequality $\|M\|_* = \|A_t + Z_t\|_* \geq \|A_t\|_* - \|Z_t\|_*$ and bound $\|Z_t\|_*$:
\begin{align*}
    \|Z_t\|_* \;&=\; \|L_t^{-1/2}\,N_t\,P_{\perp,t}\|_* \\
    \;&\leq\; \sqrt{\mathrm{rank}(Z_t)}\,\|Z_t\|_F \\
    \;&\leq\; \sqrt{\mathrm{rank}(Z_t)}\,\|L_t^{-1/2}\|_{\mathrm{op}}\,\|N_t\,P_{\perp,t}\|_F \\
    \;&\leq\; \sqrt{k}\,\|L_t^{-1/2}\|_{\mathrm{op}}\,\|N_t\|_F \\
    \;&\leq\; \sqrt{k}\,C\,\|N_t\|_F.
\end{align*}
The first inequality uses $\|X\|_* \leq \sqrt{\mathrm{rank}(X)}\,\|X\|_F$. The second splits the Frobenius norm via sub-multiplicativity $\|MN\|_F \leq \|M\|_{\mathrm{op}}\,\|N\|_F$. The third combines $\mathrm{rank}(Z_t) \leq \min(m, n{-}r) = k$ and $\|N_t P_{\perp,t}\|_F \leq \|N_t\|_F$ (since $\|P_{\perp,t}\|_{\mathrm{op}} \leq 1$). The last step uses $\|L_t^{-1/2}\|_{\mathrm{op}} = 1/\sqrt{\lambda_{\min}(L_t)} \leq C$ by Assumption~\ref{ass:clip}.
Combining,
\[
    \spadesuit_t \;\geq\; \tfrac{1}{C}\bigl(\|A_t\|_* - \|Z_t\|_*\bigr) \;\geq\; \tfrac{1}{C}\,\|A_t\|_* - \sqrt{k}\,\|N_t\|_F.
\]

\emph{Bounding $|\diamondsuit_t|$ from above.}\;
Since $\|\polar(M)\|_{\mathrm{op}} \leq 1$, the matrix H\"older inequality $|\langle X, Y\rangle| \leq \|X\|_*\,\|Y\|_{\mathrm{op}}$ gives
\[
    |\diamondsuit_t|
    \;\leq\; \|L_t^{1/2}\,Z_t\|_*
    \;=\; \|N_t\,P_{\perp,t}\|_*
    \;\leq\; \sqrt{k}\,\|N_t\,P_{\perp,t}\|_F
    \;\leq\; \sqrt{k}\,\|N_t\|_F,
\]
where the equality uses $L_t^{1/2}\,Z_t = L_t^{1/2}\cdot L_t^{-1/2}\,N_t\,P_{\perp,t} = N_t\,P_{\perp,t}$, and the next step uses $\|X\|_* \leq \sqrt{\mathrm{rank}(X)}\,\|X\|_F$ with $\mathrm{rank}(N_t P_{\perp,t}) \leq \min(m, n{-}r) = k$.

\emph{Combining.}\;
\[
    \langle\nabla f(W_t),\,\polar(M)\rangle \;=\; \spadesuit_t - \diamondsuit_t
    \;\geq\; \tfrac{1}{C}\,\|A_t\|_* - 2\sqrt{k}\,\|N_t\|_F.
\]
Multiplying by $c_a$, taking conditional expectation, and applying Jensen's inequality $\E[\|N_t\|_F\mid\mathcal F_t] \leq \sqrt{\E[\|N_t\|_F^2\mid\mathcal F_t]} \leq \sigma_F$ (Assumption~\ref{ass:noise}) yields~\eqref{eq:polar_lowerbd}.

\medskip
\noindent\textbf{Step 2.}\;
We prove the lower bound
\begin{equation}\label{eq:kl_lowerbd}
    \E\bigl[\langle\nabla f(W_t),\,\alpha_{\mathrm{kl}}\,L_t^{-1/2}G_tU_tS_t^{-1/2}U_t^\top\rangle\bigm|\mathcal F_t\bigr]
    \;\geq\; \tfrac{\alpha_{\mathrm{kl}}}{\sqrt{\Theta}}\,\|V_t\|_F^2.
\end{equation}
Since $L_t, S_t, U_t$ are $\mathcal F_t$-measurable and $\E[G_t\mid\mathcal F_t] = \nabla f(W_t)$, the conditional expectation of the left-hand side equals
\[
    \alpha_{\mathrm{kl}}\,\Tr\bigl(U_t^\top \nabla f(W_t)^\top\, L_t^{-1/2}\,\nabla f(W_t)\,U_t \cdot S_t^{-1/2}\bigr)
    \;=\; \alpha_{\mathrm{kl}}\,\Tr\bigl(V_t^\top V_t \cdot S_t^{-1/2}\bigr),
\]
where we used $V_t^\top V_t = U_t^\top \nabla f(W_t)^\top L_t^{-1/2}\,\nabla f(W_t)\,U_t$ (since $L_t^{-1/4}\cdot L_t^{-1/4} = L_t^{-1/2}$). Since $S_t^{-1/2} \succeq \Theta^{-1/2}\,I_r$ (from $\lambda_{\max}(S_t) \leq \Theta$, Lemma~\ref{lem:clamp_lb}) and $V_t^\top V_t \succeq 0$:
\[
    \Tr(V_t^\top V_t \cdot S_t^{-1/2})
    \;\geq\; \tfrac{1}{\sqrt{\Theta}}\,\Tr(V_t^\top V_t)
    \;=\; \tfrac{1}{\sqrt{\Theta}}\,\|V_t\|_F^2,
\]
proving~\eqref{eq:kl_lowerbd}.

\medskip
\noindent\textbf{Step 3.}\;
We prove the upper bound
\begin{equation}\label{eq:second_order_bound}
    \E\bigl[\|\Delta W_t\|_{\mathrm{op}}^2 \bigm| \mathcal F_t\bigr]
    \;\leq\; 2\alpha_{\mathrm{kl}}^2\,\sigma_{kl}^2 + 2c_a^2.
\end{equation}
By the triangle inequality for operator norm and $(a+b)^2 \leq 2a^2 + 2b^2$:
\[
    \|\Delta W_t\|_{\mathrm{op}}^2
    \;\leq\; 2\,\|\alpha_{\mathrm{kl}}\,L_t^{-1/2}G_tU_tS_t^{-1/2}U_t^\top\|_{\mathrm{op}}^2
    \;+\; 2\,\|c_a\,\polar(L_t^{-1/2}G_t P_{\perp,t})\|_{\mathrm{op}}^2.
\]
For the first term, $\|XU_t^\top\|_{\mathrm{op}} = \|X\|_{\mathrm{op}}$ since $U_t$ has orthonormal columns, so it equals $2\alpha_{\mathrm{kl}}^2\,\|L_t^{-1/2}G_tU_tS_t^{-1/2}\|_{\mathrm{op}}^2$. For the second term, $\|\polar(\cdot)\|_{\mathrm{op}} \leq 1$ deterministically gives $2c_a^2$. Taking conditional expectation and applying $\E[\|L_t^{-1/2}G_tU_tS_t^{-1/2}\|_{\mathrm{op}}^2 \mid \mathcal F_t] \leq \sigma_{kl}^2$ a.s. proves~\eqref{eq:second_order_bound}.

\medskip
\noindent\textbf{Step 4.}\;
Take conditional expectation of both sides of~\eqref{eq:polar_step0}. The inner product $\langle\nabla f(W_t),\Delta W_t\rangle$ splits according to~\eqref{eq:update}; applying~\eqref{eq:polar_lowerbd} (with sign flipped) and~\eqref{eq:kl_lowerbd} (with sign flipped):
\[
    \eta\,\E\bigl[\langle\nabla f(W_t),\,\Delta W_t\rangle\bigm|\mathcal F_t\bigr]
    \;\leq\;
    -\tfrac{\eta\,\alpha_{\mathrm{kl}}}{\sqrt{\Theta}}\,\|V_t\|_F^2 \;-\; \tfrac{\eta\,c_a}{C}\,\|A_t\|_* \;+\; 2\eta\,c_a\sqrt{k}\,\sigma_F.
\]
Applying~\eqref{eq:second_order_bound} to the second-order piece:
\[
    \tfrac{\eta^2 L_{\mathrm{op}}}{2}\,\E\bigl[\|\Delta W_t\|_{\mathrm{op}}^2\bigm|\mathcal F_t\bigr]
    \;\leq\;
    \eta^2 L_{\mathrm{op}}\,(\alpha_{\mathrm{kl}}^2\,\sigma_{kl}^2 + c_a^2).
\]
Substituting into~\eqref{eq:polar_step0} and recalling $A_t = L_t^{-1/2}\nabla f(W_t)\,P_{\perp,t}$ and $V_t = L_t^{-1/4}\nabla f(W_t)\,U_t$ yields~\eqref{eq:polar_descent}.
\end{proof}
\subsubsection{Scaling inequalities}

\begin{lemma}[Scaling inequalities]\label{lem:scaling}
Let $L\in\R^{m\times m}$ be SPD with $\lambda_{\max}(L)\leq\Theta$. For any $X\in\R^{m\times n}$:
\begin{enumerate}[label=(\roman*),leftmargin=2em,nosep]
\item $\|L^{-1/2}X\|_* \geq \Theta^{-1/2}\|X\|_*$.
\item $\|L^{-1/4}X\|_F^2 \geq \Theta^{-1/2}\|X\|_F^2$.
\end{enumerate}
\end{lemma}

\begin{proof}
(i) The key step is sub-multiplicativity of the nuclear norm against the operator norm: $\|MN\|_* \leq \|M\|_{\mathrm{op}}\,\|N\|_*$ for any matrices $M, N$.
Apply this to $X = L^{1/2}\cdot L^{-1/2}X$:
\[
    \|X\|_* \;\leq\; \|L^{1/2}\|_{\mathrm{op}}\,\|L^{-1/2}X\|_* \;=\; \sqrt{\lambda_{\max}(L)}\,\|L^{-1/2}X\|_* \;\leq\; \sqrt{\Theta}\,\|L^{-1/2}X\|_*.
\]
Dividing by $\sqrt{\Theta}$ gives (i).

\medskip
(ii) The key step is sub-multiplicativity of the Frobenius norm: $\|MN\|_F \leq \|M\|_{\mathrm{op}}\,\|N\|_F$ for any matrices $M, N$. Squaring and applying to $X = L^{1/4}\cdot L^{-1/4}X$:
\[
    \|X\|_F^2 \;\leq\; \|L^{1/4}\|_{\mathrm{op}}^2\,\|L^{-1/4}X\|_F^2 \;=\; \sqrt{\lambda_{\max}(L)}\,\|L^{-1/4}X\|_F^2 \;\leq\; \sqrt{\Theta}\,\|L^{-1/4}X\|_F^2,
\]
using $\|L^{1/4}\|_{\mathrm{op}}^2 = \lambda_{\max}(L)^{1/2}$. Dividing by $\sqrt{\Theta}$ gives (ii).
\end{proof}

\subsubsection{Proof of Theorem~\ref{thm:polar_conv}}
\label{app:proof_polar_conv}

\begin{proof}[Proof of Theorem~\ref{thm:polar_conv}]
From Proposition~\ref{prop:polar_descent}, taking total expectation and writing $K \coloneqq c_a^2 + \alpha_{\mathrm{kl}}^2\,\sigma_{kl}^2$:
\begin{align}\label{eq:polar_desc_total}
\E[f(W_{t+1}) - f(W_t)]
\;\leq\;
&-\tfrac{\eta c_a}{C}\E\|L_t^{-1/2}\nabla f(W_t)\,P_{\perp,t}\|_*
- \tfrac{\eta\alpha_{\mathrm{kl}}}{\sqrt{\Theta}}\E\|L_t^{-1/4}\nabla f(W_t)\,U_t\|_F^2 \notag\\
&+ 2\eta c_a\sqrt k\sigma_F + \eta^2 L_{\mathrm{op}} K.
\end{align}

\medskip
\noindent\textbf{Removing the preconditioner.}\;
Apply Lemma~\ref{lem:scaling}(i) with $X = \nabla f(W_t)\,P_{\perp,t}$ and $L = L_t$:
\[
    \|L_t^{-1/2}\nabla f(W_t)\,P_{\perp,t}\|_* \;\geq\; \Theta^{-1/2}\,\|\nabla f(W_t)\,P_{\perp,t}\|_*.
\]
Apply Lemma~\ref{lem:scaling}(ii) with $X = \nabla f(W_t)\,U_t$:
\[
    \|L_t^{-1/4}\nabla f(W_t)\,U_t\|_F^2 \;\geq\; \Theta^{-1/2}\,\|\nabla f(W_t)\,U_t\|_F^2.
\]
Substituting into~\eqref{eq:polar_desc_total}:
\begin{align*}
\E[f(W_{t+1}) - f(W_t)]
\;\leq\;
&-\eta\,\E\!\left[\tfrac{c_a}{C\sqrt{\Theta}}\|\nabla f(W_t)\,P_{\perp,t}\|_* + \tfrac{\alpha_{\mathrm{kl}}}{\Theta}\|\nabla f(W_t)\,U_t\|_F^2\right] \\
&+ 2\eta c_a\sqrt k\sigma_F + \eta^2 L_{\mathrm{op}} K.
\end{align*}

\medskip
\noindent\textbf{Telescoping.}\;
Summing over $t=0,\ldots,T-1$, the left-hand side telescopes to $\E[f(W_T)-f(W_0)]\geq -\Delta_0$ (Assumption~\ref{ass:lb}). Rearranging and dividing by $\eta T$:
\begin{align*}
\frac{1}{T}\sum_{t=0}^{T-1}\E\!\left[\tfrac{c_a}{C\sqrt{\Theta}}\|\nabla f(W_t)\,P_{\perp,t}\|_* + \tfrac{\alpha_{\mathrm{kl}}}{\Theta}\|\nabla f(W_t)\,U_t\|_F^2\right]
\;\leq\;
\tfrac{\Delta_0}{\eta T} + 2c_a\sqrt k\sigma_F + \eta L_{\mathrm{op}} K.
\end{align*}

\medskip
\noindent\textbf{Optimizing the step size.}\;
The $\eta$-dependent terms on the right are $\Delta_0/(\eta T) + \eta L_{\mathrm{op}} K$. Their product $\Delta_0 L_{\mathrm{op}} K / T$ is independent of $\eta$. The sum is minimized at $\eta = \sqrt{\Delta_0/(T L_{\mathrm{op}} K)}$, giving $\Delta_0/(\eta T) = \eta L_{\mathrm{op}} K = \sqrt{\Delta_0 L_{\mathrm{op}} K / T}$. Substituting yields~\eqref{eq:polar_conv}.
\end{proof}

\subsubsection{Proof of Lemma~\ref{lem:equiv}}

\begin{proof}[Proof of Lemma~\ref{lem:equiv}]
$(\Leftarrow)$ trivial.
$(\Rightarrow)$ both summands are non-negative, so both vanish.
$\|\nabla f(W)\,P_\perp\|_* = 0$ gives $\nabla f(W)\,P_\perp = 0$;
$\|\nabla f(W)\,U\|_F = 0$ gives $\nabla f(W)\,U = 0$.
Therefore $\nabla f(W) = \nabla f(W)(UU^\top + P_\perp) = \nabla f(W)\,U\,U^\top + \nabla f(W)\,P_\perp = 0$.
\end{proof}

\subsubsection{Interpreting \texorpdfstring{$\sigma_{kl}^2$}{sigma\_kl^2}}
\label{rem:sigma_kl}

$\sigma_{kl}^2$ admits a worst-case upper bound and a stationary reference value for the integrand:
\begin{itemize}[leftmargin=*,itemsep=1pt,topsep=2pt]
\item \textbf{Worst case (via Assumption~\ref{ass:clip} and Assumption~\ref{ass:bdg}).} Submultiplicativity gives
\[
    \|L^{-1/2}GUS^{-1/2}\|_{\mathrm{op}} \;\leq\; \|L^{-1/2}\|_{\mathrm{op}}\,\|G\|_{\mathrm{op}}\,\|S^{-1/2}\|_{\mathrm{op}} \;\leq\; C\cdot G_{\max}\cdot C \;=\; C^2 G_{\max},
\]
hence $\sigma_{kl}^2\leq C^4 G_{\max}^2$.
\item \textbf{At restricted KL stationarity (Claim~\ref{claim:stat}).} For any subspace $U$, the $S$-stationarity condition~\eqref{eq:restr_S} reads $S^* = \tfrac{1}{m}\E[U^\top G^\top \Lr^{-1}G\,U]$, where $S^*$ and $L^*_{\mathrm{restr}}$ are the corresponding optimal preconditioners for that $U$. Evaluating $\sigma_{kl}^2$ at $(L_t, S_t) = (\Lr, S^*)$ for the algorithm's current $U_t = U$:
\[
    \E\!\left[\|\Lr^{-1/2}GU\,S^{*-1/2}\|_{\mathrm{op}}^2\right] \;\leq\; \E\!\left[\|\Lr^{-1/2}GU\,S^{*-1/2}\|_F^2\right] \;=\; \Tr(S^{*-1}\cdot mS^*) \;=\; mr,
\]
where the inequality uses $\|X\|_{\mathrm{op}}^2 \leq \|X\|_F^2$, and the equality uses trace cyclicity together with the $S$-stationarity. Thus $mr$ is a \emph{stationary reference value} for the integrand---independent of $C$---rather than a bound on the global constant $\sigma_{kl}^2$, which by definition takes an essential supremum over all $t$. In our practical algorithm (\S\ref{app:practical}), each preconditioner's eigenvalues are clipped at a dimension-aware ceiling of the form $\max(10, \min(\mathrm{dim}, 4000))$; uniformly across all preconditioners, $C$ is therefore of order $\max(m, n)$ for the layer dimensions of interest. The worst-case bound $C^4 G_{\max}^2$ thus scales as $\max(m,n)^4$, while the stationary reference value $mr$ is linear in $m$.
\end{itemize}
Because the algorithm's EMA updates in Algorithm~\ref{alg:practical} track Claim~\ref{claim:stat}'s fixed point, we expect $\sigma_{kl}^2 = O(mr)$ whenever the algorithm's state is close to the restricted KL stationary point; a quantitative bound along the trajectory requires a tracking analysis of the EMA and is left to future work.

\subsubsection{Relation to Frobenius-norm stationarity}
\label{rem:frob_conversion}

The left-hand side of~\eqref{eq:polar_conv} is \emph{state-dependent}: it reads the gradient through the algorithm's own subspace decomposition $(U_t, P_{\perp,t})$. It is nevertheless a genuine stationarity measure by Lemma~\ref{lem:equiv}. The irreducible noise floor $2c_a\sqrt k\,\sigma_F$ is linear in $\sigma_F$, a feature inherited from the sign-SGD--style argument~\citep{bernstein2018signsgd} underlying the orthogonalization identity; this differs from the quadratic-in-$\sigma_F$ floor in analyses where the update is a positive-definite linear function of the gradient (cf.~Theorem~\ref{thm:naive_conv} below, which has no floor after step-size balancing).

A crude conversion to Frobenius stationarity uses $\|X\|_F^2\leq\|X\|_{\mathrm{op}}\|X\|_*\leq G_{\max}\|X\|_*$ on the complement piece, and divides through by the weight $\alpha_{\mathrm{kl}}/\Theta$ on the subspace piece, giving
\[
\frac{1}{T}\sum_t\E\|\nabla f(W_t)\|_F^2 \;\leq\; 2\max\!\left(G_{\max}\tfrac{C\sqrt{\Theta}}{c_a},\,\tfrac{\Theta}{\alpha_{\mathrm{kl}}}\right)\!\left[2\sqrt{\tfrac{\Delta_0 L_{\mathrm{op}}K}{T}} + 2c_a\sqrt k\,\sigma_F\right].
\]
The constant in front of the bracket is loose---each of its two arguments scales at least as $C^2$ since $\Theta$ is of order $C^2$ and $\sqrt\Theta$ of order $C$---so the Frobenius-converted bound is reported only for comparison; the mixed-norm bound~\eqref{eq:polar_conv} is the result we establish.

\paragraph{Rate asymmetry and noise floor.}
The complement's nuclear-norm stationarity converges at $O(T^{-1/2})$, while the subspace's Frobenius-norm stationarity converges at $O(T^{-1/4})$ (taking the square root of the squared-Frobenius term). This does not imply smaller $r$ is preferable: the noise floor $2c_a\sqrt{k}\,\sigma_F$ grows as $r$ shrinks ($k = \min(m, n{-}r)$), offsetting the faster complement rate. The floor is intrinsic to orthogonalization: as a nonlinear operation, it creates an irreducible bias from gradient noise that does not vanish with the step size. The subspace update, being linear in the gradient, has no such floor (cf.\ Theorem~\ref{thm:naive_conv}).

\subsection{Technical companion: Smok-Hop}
\label{app:naive_conv_sub}

Smok-Hop (\S\ref{sec:stationarity}) replaces the polar orthogonalization in line~\ref{line:comp} of Algorithm~\ref{alg:ideal} with the scalar scaling $\mu_{\perp}^{-1/2}\,L^{-1/2}\,G_\perp$, giving the update direction $\Delta W = L^{-1/2}\,G\,\hat R^{-1/2}$. This is a linear function of $G$ and admits a simpler analysis than Pro-KLShampoo's polar update. We include the convergence result here for comparison; it is \emph{not} the algorithm used in practice.

This subsection replaces operator-norm smoothness with the stronger Frobenius smoothness:

\noindent\textbf{(ii)$'$}\label{ass:smooth_F}~~$L_F$-Frobenius smoothness: $f(W')\leq f(W)+\langle\nabla f(W),W'-W\rangle + \tfrac{L_F}{2}\|W'-W\|_F^2$.

\begin{theorem}[Convergence of Smok-Hop]\label{thm:naive_conv}
Define $\sigma_P^2$ as the conditional second moment (in Frobenius norm) of the preconditioned stochastic gradient:
\begin{equation}\label{eq:sigma_P_def}
    \sigma_P^2 \;\coloneqq\; \sup_{t \geq 0}\,\operatorname{ess\,sup}\, \E\!\left[\bigl\|L_t^{-1/2}\,G_t\,\hat R_t^{-1/2}\bigr\|_F^2 \,\bigm|\, \mathcal F_t\right],
\end{equation}
where $\hat R_t = U_t S_t U_t^\top + \mu_{\perp,t}\,P_{\perp,t}$ is the spike-and-flat right-side preconditioner at step $t$.
Under Assumptions~\ref{ass:lb}, \ref{ass:noise}, \ref{ass:bdg}, \ref{ass:clip} and~(ii)$'$ above, for any $T \geq 1$ and step size $\eta = \eta_0/\sqrt T$ with $\eta_0 > 0$:
\begin{equation}\label{eq:naive_conv}
    \frac{1}{T}\sum_{t=0}^{T-1}\E\|\nabla f(W_t)\|_F^2
    \;\leq\;
    \frac{\Theta\,\Delta_0}{\eta_0\,\sqrt T}
    \;+\;
    \frac{\eta_0\,L_F\,\Theta\,\sigma_P^2}{2\sqrt T},
\end{equation}
where $\Delta_0 \coloneqq f(W_0) - f^*$.
The balanced choice $\eta_0^\star = \sqrt{2\Delta_0/(L_F \sigma_P^2)}$ gives the rate
\[
    \frac{1}{T}\sum_{t=0}^{T-1}\E\|\nabla f(W_t)\|_F^2 \;=\; O\!\left(\Theta \sqrt{\frac{L_F \Delta_0 \sigma_P^2}{T}}\right).
\]
The clip threshold $C$ enters the rate only through $\sigma_P^2$ (see~\S\ref{rem:sigma_P} for reference values).
\end{theorem}

\subsubsection{Proof of Theorem~\ref{thm:naive_conv}}
\label{app:proof_naive_conv}

Throughout this subsubsection, the scalar-complement update direction at step $t$ is
\[
    \Delta W_t \;\coloneqq\; L_t^{-1/2}\,G_t\,\hat R_t^{-1/2},
    \qquad \hat R_t \coloneqq U_t S_t U_t^\top + \mu_{\perp,t} P_{\perp,t},
\]
so that $W_{t+1} = W_t - \eta\,\Delta W_t$. (This is the scalar-complement update; the Pro-KLShampoo update direction in~\S\ref{app:polar_conv} is different.)

\medskip
\noindent\textbf{Linear structure.}\;
The map $G_t \mapsto \Delta W_t$ is linear in $G_t$; its $\mathrm{vec}$ form is multiplication by the matrix $P_t \coloneqq \hat R_t^{-1/2}\otimes L_t^{-1/2}$, which is symmetric positive definite with eigenvalues in $[\Theta^{-1}, C^2]$ (the lower bound is from Lemma~\ref{lem:clamp_lb}, the upper bound from Assumption~\ref{ass:clip}). The lower bound on $\lambda_{\min}(P_t)$ implies, for any matrix $M$ of the same shape as $G_t$,
\begin{equation}\label{eq:D_lb}
    \langle M,\,L_t^{-1/2}\,M\,\hat R_t^{-1/2}\rangle
    \;=\; \langle \mathrm{vec}(M),\,P_t\,\mathrm{vec}(M)\rangle
    \;\geq\; \Theta^{-1}\|M\|_F^2.
\end{equation}

\medskip
\noindent\textbf{Single-step descent.}\;
By Frobenius smoothness~(ii)$'$, with $W_{t+1} - W_t = -\eta\,\Delta W_t$:
\[
    f(W_{t+1}) \;\leq\; f(W_t) - \eta\,\langle\nabla f(W_t),\,\Delta W_t\rangle + \tfrac{\eta^2 L_F}{2}\,\|\Delta W_t\|_F^2.
\]
Take conditional expectation given $\mathcal F_t$ (so $L_t, S_t, U_t, \mu_{\perp,t}$ are measurable; only $G_t$ is random). Since $\Delta W_t$ is linear in $G_t$ and $\E[G_t\mid\mathcal F_t] = \nabla f(W_t)$ (Assumption~\ref{ass:noise}),
\[
    \E[\Delta W_t \mid \mathcal F_t] \;=\; L_t^{-1/2}\,\nabla f(W_t)\,\hat R_t^{-1/2}.
\]
Apply~\eqref{eq:D_lb} with $M = \nabla f(W_t)$ to the first-order term, and bound the second-order term by $\E[\|\Delta W_t\|_F^2\mid\mathcal F_t] \leq \sigma_P^2$ from~\eqref{eq:sigma_P_def}:
\begin{equation}\label{eq:onestep_descent}
    \E[f(W_{t+1})\mid\mathcal F_t]
    \;\leq\; f(W_t) - \tfrac{\eta}{\Theta}\,\|\nabla f(W_t)\|_F^2 + \tfrac{\eta^2 L_F}{2}\,\sigma_P^2.
\end{equation}

\medskip
\noindent\textbf{Telescoping.}\;
Taking total expectation, summing over $t = 0,\ldots,T-1$, and using $\E[f(W_T)] \geq f^*$:
\[
    \tfrac{\eta}{\Theta}\sum_{t=0}^{T-1}\E\|\nabla f(W_t)\|_F^2 \;\leq\; \Delta_0 + \tfrac{T\eta^2 L_F\sigma_P^2}{2}.
\]
Substituting $\eta = \eta_0/\sqrt T$ and dividing by $T\eta/\Theta = \eta_0\sqrt T/\Theta$:
\[
    \frac{1}{T}\sum_{t=0}^{T-1}\E\|\nabla f(W_t)\|_F^2
    \;\leq\; \frac{\Theta\Delta_0}{\eta_0\sqrt T} + \frac{\eta_0 L_F\,\Theta\,\sigma_P^2}{2\sqrt T},
\]
which is~\eqref{eq:naive_conv}. The balanced choice $\eta_0^\star = \sqrt{2\Delta_0/(L_F\sigma_P^2)}$ minimizes the right-hand side, giving the rate $\Theta\sqrt{2 L_F\Delta_0\sigma_P^2/T}$.\qed

\subsubsection{Interpreting \texorpdfstring{$\sigma_P^2$}{sigma\_P^2}}
\label{rem:sigma_P}

From its definition~\eqref{eq:sigma_P_def}, $\sigma_P^2$ is the conditional second moment (in Frobenius norm) of the preconditioned stochastic gradient. We give two reference values:
\begin{itemize}[leftmargin=*,itemsep=1pt,topsep=2pt]
    \item \textbf{Worst case (Assumption~\ref{ass:clip}).} The spectral bounds $\|L_t^{-1/2}\|_{\mathrm{op}}, \|\hat R_t^{-1/2}\|_{\mathrm{op}} \leq C$ give $\|L_t^{-1/2}G_t\hat R_t^{-1/2}\|_F^2 \leq C^4\|G_t\|_F^2$, and Assumptions~\ref{ass:bdg} and~\ref{ass:noise} imply $\E\|G_t\|_F^2 \leq \sigma_F^2 + \min(m,n)\,G_{\max}^2$. Hence $\sigma_P^2 \leq C^4(\sigma_F^2 + \min(m,n)G_{\max}^2)$.
    \item \textbf{Restricted KL stationarity (Claim~\ref{claim:stat}).} At the restricted KL fixed point $(\Lr, \hat R^*)$, the L-stationarity condition~\eqref{eq:restr_L} reads $\Lr = \tfrac{1}{n}\,\E[G\,\hat R^{*-1}\,G^\top]$. By trace cyclicity,
    \begin{equation}\label{eq:sigmaP_stationary}
        \E\!\left[\bigl\|\Lr^{-1/2}\,G\,\hat R^{*-1/2}\bigr\|_F^2\right]
        \;=\; \Tr\!\bigl(\Lr^{-1}\,\E[G\,\hat R^{*-1}\,G^\top]\bigr)
        \;=\; n\,\Tr(I_m)
        \;=\; mn.
    \end{equation}
    Hence $\sigma_P^2 = mn$ \emph{exactly} at stationarity---independent of $C$.
\end{itemize}
Because the algorithm's EMA updates in Algorithm~\ref{alg:practical} track Claim~\ref{claim:stat}'s fixed point, we expect $\sigma_P^2 = O(mn)$ whenever the algorithm's state is close to the restricted KL stationary point. A quantitative bound on $\sigma_P^2$ along the entire trajectory requires a tracking analysis of the EMA and is left to future work.

\subsection{Calibration of the mixing weight $\alpha_{\mathrm{kl}}$}
\label{app:lambda_star}

This subsection derives the closed-form magnitude estimate of $\alpha_{\mathrm{kl}}$ used in~\S\ref{sec:polar}. The principle is to choose $\alpha_{\mathrm{kl}}$ so that the relative scale between the subspace and complement components of the Pro-KLShampoo update~\eqref{eq:update} matches that of the scalar-complement update~\eqref{eq:restr_update} at the restricted KL stationary state. This way, replacing the scalar-complement whitening on the complement with polar does not, in expectation, alter the relative scale between subspace and complement prescribed by the restricted-KL solution. Operator norm is the natural matching norm: the polar map yields a partial isometry, so its output has a deterministic operator norm. Claim~\ref{claim:stat} determines $\E\|\Delta_{\mathrm{res}}^{\mathrm{naive}}\|_F^2$ but not $\E\|\Delta_{\mathrm{res}}^{\mathrm{naive}}\|_{\mathrm{op}}^2$; combining this with the bound $\mathrm{rank}(\Delta_{\mathrm{res}}^{\mathrm{naive}}) \leq k$ a.s.\ gives a bracketed range on $\alpha_{\mathrm{kl}}^*$.

\paragraph{Update components at the stationary state.}
At the restricted KL stationary state, write the subspace component (shared by both updates) and the two complement components as
\begin{align*}
    \Delta_{\mathrm{kl}} &\coloneqq (\Lr)^{-1/2}\,G\,U^*(S^*)^{-1/2}\,U^{*\top},\\
    \Delta_{\mathrm{res}}^{\mathrm{naive}} &\coloneqq (\mu_\perp^*)^{-1/2}\,(\Lr)^{-1/2}\,G_\perp, &\text{(scalar-complement, scalar-whitened)}\\
    \Delta_{\mathrm{res}}^{\mathrm{polar}} &\coloneqq c_a\,\polar\!\bigl((\Lr)^{-1/2}\,G_\perp\bigr), &\text{(Pro-KLShampoo, orthogonalized)}
\end{align*}
with $c_a = \sqrt{\max(1, m/n)}$. The scalar-complement update direction is $\Delta_{\mathrm{kl}} + \Delta_{\mathrm{res}}^{\mathrm{naive}}$; the Pro-KLShampoo update direction is $\alpha_{\mathrm{kl}}\,\Delta_{\mathrm{kl}} + \Delta_{\mathrm{res}}^{\mathrm{polar}}$. Since $\Delta_{\mathrm{kl}}$ is identical in both, the calibration only needs to match the complement components in scale (under a chosen unitarily invariant norm; in expectation where applicable).

\paragraph{Polar complement: deterministic norms.}
The polar map produces a partial isometry: $\polar(X) = U_X V_X^\top$ where $X = U_X\Sigma_X V_X^\top$ is the SVD. Write $\bar m \coloneqq \min(m,n)$ and $\bar n \coloneqq \max(m,n)$. The input to polar is $(\Lr)^{-1/2}G_\perp \in \R^{m\times n}$ in the right projection case ($m \leq n$) and $G_\perp\,(R^*)^{-1/2} \in \R^{m\times n}$ in the left projection case ($m > n$, Remark~\ref{rem:left_proj}); under any absolutely continuous gradient distribution it has rank
\[
    k \;\coloneqq\; \min(\bar m,\,\bar n - r) \;=\; \begin{cases} \min(m,\,n-r), & m \leq n, \\ \min(m-r,\,n), & m > n, \end{cases}
\]
almost surely. Hence both norms of the polar update are deterministic:
\begin{equation}\label{eq:polar_norms}
    \|\Delta_{\mathrm{res}}^{\mathrm{polar}}\|_F \;=\; c_a\sqrt{k},
    \qquad
    \|\Delta_{\mathrm{res}}^{\mathrm{polar}}\|_{\mathrm{op}} \;=\; c_a,
\end{equation}
where $c_a = \sqrt{\max(1, m/n)}$ is computed from the original matrix dimensions (matching the implementation).

\paragraph{Scalar-complement: norms in expectation.}
For the scalar-complement,
\[
    \|\Delta_{\mathrm{res}}^{\mathrm{naive}}\|_F^2
    \;=\; (\mu_\perp^*)^{-1}\,\Tr\!\bigl(G_\perp^\top (\Lr)^{-1}\,G_\perp\bigr).
\]
At the restricted KL stationary state, $\mu_\perp$-stationarity~\eqref{eq:restr_mu} gives $\Tr\!\bigl(\E[G_\perp^\top (\Lr)^{-1} G_\perp]\bigr) = m(n-r)\,\mu_\perp^*$, hence $\E\|\Delta_{\mathrm{res}}^{\mathrm{naive}}\|_F^2 = m(n-r)$ for $m \leq n$. The left projection case ($m > n$) is symmetric under $L \leftrightarrow R$ and gives $n(m-r)$. Uniformly,
\begin{equation}\label{eq:naive_frob_exp}
    \sqrt{\E\|\Delta_{\mathrm{res}}^{\mathrm{naive}}\|_F^2} \;=\; \sqrt{\bar m(\bar n - r)}.
\end{equation}
For the operator norm, since $\mathrm{rank}(\Delta_{\mathrm{res}}^{\mathrm{naive}}) \leq k$ a.s., the inequalities $\|X\|_F^2/\mathrm{rank}(X) \leq \|X\|_{\mathrm{op}}^2 \leq \|X\|_F^2$ give
\begin{equation}\label{eq:naive_op_brackets}
    \frac{\bar m(\bar n - r)}{k} \;\leq\; \E\|\Delta_{\mathrm{res}}^{\mathrm{naive}}\|_{\mathrm{op}}^2 \;\leq\; \bar m(\bar n - r),
\end{equation}
with the lower (resp.\ upper) bound attained when the singular values of $\Delta_{\mathrm{res}}^{\mathrm{naive}}$ are uniformly spread (resp.\ concentrated in one direction).

\paragraph{Matching condition: bracketed range.}
We choose $\alpha_{\mathrm{kl}}^*$ so that the subspace-to-complement norm ratio is preserved across the two update rules, with $\Delta_{\mathrm{kl}}$ shared:
\[
    \frac{\|\alpha_{\mathrm{kl}}^*\,\Delta_{\mathrm{kl}}\|_\diamond}{\|\Delta_{\mathrm{res}}^{\mathrm{polar}}\|_\diamond}
    \;=\;
    \frac{\|\Delta_{\mathrm{kl}}\|_\diamond}{\sqrt{\E\|\Delta_{\mathrm{res}}^{\mathrm{naive}}\|_\diamond^2}}
    \quad\text{if and only if}\quad
    \alpha_{\mathrm{kl}}^* \;=\; \frac{\|\Delta_{\mathrm{res}}^{\mathrm{polar}}\|_\diamond}{\sqrt{\E\|\Delta_{\mathrm{res}}^{\mathrm{naive}}\|_\diamond^2}},
\]
where $\|\cdot\|_\diamond$ is a unitarily invariant norm.

The natural matching norm is the operator norm: by~\eqref{eq:polar_norms}, $\|\Delta_{\mathrm{res}}^{\mathrm{polar}}\|_{\mathrm{op}} = c_a$ deterministically, regardless of the gradient's singular value distribution. Under operator-norm matching,
\[
    \alpha_{\mathrm{kl}}^* \;=\; \frac{c_a}{\sqrt{\E\|\Delta_{\mathrm{res}}^{\mathrm{naive}}\|_{\mathrm{op}}^2}}.
\]
Claim~\ref{claim:stat} determines $\E\|\Delta_{\mathrm{res}}^{\mathrm{naive}}\|_F^2$ but not $\E\|\Delta_{\mathrm{res}}^{\mathrm{naive}}\|_{\mathrm{op}}^2$. Substituting the bracket~\eqref{eq:naive_op_brackets} into the operator-norm matching formula gives
\begin{equation}\label{eq:lambda_star_app}
    \boxed{\;
    \frac{c_a}{\sqrt{\bar m(\bar n - r)}}
    \;\leq\; \alpha_{\mathrm{kl}}^* \;\leq\;
    \frac{c_a\sqrt{k}}{\sqrt{\bar m(\bar n - r)}}.
    \;}
\end{equation}
The lower bound corresponds to a maximally concentrated scalar-complement update (Frobenius mass in a single direction; op-norm equals Frobenius); the upper bound corresponds to a uniformly spread one (op-norm equals Frobenius$/\sqrt{k}$). The upper bound coincides with the simpler \emph{Frobenius matching} formula $\alpha_{\mathrm{kl}}^* = c_a\sqrt{k}/\sqrt{\bar m(\bar n - r)}$, recovered from~\eqref{eq:polar_norms}--\eqref{eq:naive_frob_exp} by setting $\|\cdot\|_\diamond = \|\cdot\|_F$. The bracket has multiplicative width $\sqrt{k}$. For square layers ($m = n$, so $\bar m = \bar n = n$, $c_a = 1$, $k = n-r$), it reduces to $1/\sqrt{n(n-r)} \leq \alpha_{\mathrm{kl}}^* \leq 1/\sqrt{n}$.

\paragraph{Numerical values.}
Evaluating the bracket~\eqref{eq:lambda_star_app} at $r = 128$ for the four model configurations used in our experiments:

\begin{center}
\small
\begin{tabular}{lcccccc}
\toprule
Layer & $m$ & $n$ & $r$ & $c_a$ & $k$ & bracket on $\alpha_{\mathrm{kl}}^*$ \\
\midrule
\multicolumn{7}{l}{\textit{GPT-2 124M} ($d{=}768$, MLP intermediate $4d$, 12 layers)} \\
Attention QKV / out         & 768  & 768  & 128 & $1$            & $640$  & $[0.0014,\;0.036]$ \\
MLP up   ($4d \times d$)    & 3072 & 768  & 128 & $2$            & $768$  & $[0.0013,\;0.037]$ \\
MLP down ($d \times 4d$)    & 768  & 3072 & 128 & $1$            & $768$  & $[0.00067,\;0.018]$ \\
\midrule
\multicolumn{7}{l}{\textit{GPT-medium 350M} ($d{=}1024$, MLP intermediate $4d$, 24 layers)} \\
Attention QKV / out         & 1024 & 1024 & 128 & $1$            & $896$  & $[0.0010,\;0.031]$ \\
MLP up   ($4d \times d$)    & 4096 & 1024 & 128 & $2$            & $1024$ & $[0.0010,\;0.032]$ \\
MLP down ($d \times 4d$)    & 1024 & 4096 & 128 & $1$            & $1024$ & $[0.00050,\;0.016]$ \\
\midrule
\multicolumn{7}{l}{\textit{LLaMA-134M} ($d{=}768$, MLP intermediate $e{=}2048$, 12 layers, SwiGLU)} \\
Attention Q/K/V/O           & 768  & 768  & 128 & $1$            & $640$  & $[0.0014,\;0.036]$ \\
MLP gate / up ($e \times d$)& 2048 & 768  & 128 & $\sqrt{8/3}$   & $768$  & $[0.0013,\;0.037]$ \\
MLP down ($d \times e$)     & 768  & 2048 & 128 & $1$            & $768$  & $[0.00082,\;0.023]$ \\
\midrule
\multicolumn{7}{l}{\textit{LLaMA-450M} ($d{=}1024$, MLP intermediate $e{=}2816$, 30 layers, SwiGLU)} \\
Attention Q/K/V/O           & 1024 & 1024 & 128 & $1$            & $896$  & $[0.0010,\;0.031]$ \\
MLP gate / up ($e \times d$)& 2816 & 1024 & 128 & $\sqrt{11/4}$  & $1024$ & $[0.0010,\;0.032]$ \\
MLP down ($d \times e$)     & 1024 & 2816 & 128 & $1$            & $1024$ & $[0.00060,\;0.019]$ \\
\bottomrule
\end{tabular}
\end{center}

In our experiments we sweep $\alpha_{\mathrm{kl}} \in \{0.005,\,0.01,\,0.015\}$ as a single value shared across layers (\S\ref{sec:polar}). Intersecting the per-layer brackets in the table gives the all-layer-feasible interval $[\max_i \mathrm{lower}_i,\;\min_i \mathrm{upper}_i] \approx [1.4\times 10^{-3},\;1.6\times 10^{-2}]$; the swept values $\{0.005, 0.01, 0.015\}$ all sit inside this intersection. The swept optimum lies between $0.005$ and $0.01$, varying across the four model configurations. Pinning $\alpha_{\mathrm{kl}}^*$ to a point would require characterizing the singular value concentration of $\Delta_{\mathrm{res}}^{\mathrm{naive}}$ along the trajectory, which we leave to future work.

\newpage
\section{Practical algorithm}
\label{app:practical}

Algorithm~\ref{alg:practical} below is the implementation form actually used in our experiments. The three practical additions from Algorithm~\ref{alg:ideal} are (P1) Nesterov momentum, (P2) damping and clipping in the inverse-square-roots, and (P3) Newton--Schulz approximation to the polar factor; the eigenvalue EMA and subspace-tracking step mirror Algorithm~\ref{alg:ideal} and are inlined into the algorithm below. We present the case $m \leq n$; $m > n$ is symmetric (Remark~\ref{rem:left_proj}). For numerical stability under bfloat16 we store $\lambda_L^{\odot -1/2}, \lambda_S^{\odot -1/2}, \mu_\perp^{-1/2}$ rather than $\lambda_L, \lambda_S, \mu_\perp$~\citep{lin2025understanding}.

\paragraph{State.}
Per parameter $W \in \R^{m \times n}$: momentum buffer $M$, subspace basis $U \in \R^{n \times r}$ with $U^\top U = I_r$, left factor $L \in \Spp^m$ and subspace factor $S \in \Spp^r$ with eigendecompositions $L = Q_L\,\Diag(\lambda_L)\,Q_L^\top$ and $S = Q_S\,\Diag(\lambda_S)\,Q_S^\top$, complement scalar $\mu_\perp > 0$.

\paragraph{Hyperparameter defaults.}
$\mu = \beta_2 = 0.95$ (momentum and EMA weight), $\varepsilon = 10^{-8}$ (damping), $\alpha_{\mathrm{kl}} = 0.01$ (mixing weight), $\tau = 10$ (QR refresh period), $T_{\mathrm{NS}} = 5$ (Newton--Schulz iterations), $r \in \{32, 64, 128\}$ (subspace rank), $\lambda_0 = 0.1$ (initial eigenvalue scale). Learning rate $\eta$ and weight decay $\lambda_{\mathrm{wd}}$ are swept per configuration.

\paragraph{Initialization.}
$M \gets 0$; $U$ is the top-$r$ right singular vectors of the first observed gradient $G^{(0)}$; $L \gets \tfrac{1-\beta_2}{r}\Gtil\Gtil^\top$, $S \gets \tfrac{1-\beta_2}{m}\Gtil^\top\Gtil$ with $\Gtil = G^{(0)}U$; $\lambda_L, \lambda_S, \mu_\perp \gets \lambda_0$ (fixed initialization, since single-sample eigenvalues are noisy); $Q_L, Q_S$ are the eigenvectors of $L, S$ sorted by descending eigenvalue.

\begin{algorithm}[H]
\caption{Pro-KLShampoo (practical; one step, $m \leq n$)}
\label{alg:practical}
\begin{algorithmic}[1]
\Require gradient $G \in \R^{m \times n}$
\Statex
\vspace{-6pt}
\Statex \emph{(1) Momentum and projection}
\State $M \gets \mu\,M + G$;\quad $\hat{G} \gets G + \mu\,M$
\State $\hat{G}_\parallel \gets \hat{G}\,U$;\quad $\hat{G}_\perp \gets \hat{G} - \hat{G}_\parallel\,U^\top$
\Statex
\vspace{-6pt}
\Statex \emph{(2) Update direction}
\State $\Delta_{\mathrm{sub}} \gets Q_L\,\Diag(\lambda_L^{\odot -1/2})\,Q_L^\top\,\hat{G}_\parallel\,Q_S\,\Diag(\lambda_S^{\odot -1/2})\,Q_S^\top\,U^\top$
\State $\Delta_{\mathrm{res}} \gets \mathrm{NS}\!\big(Q_L\,\Diag(\lambda_L^{\odot -1/2})\,Q_L^\top\,\hat{G}_\perp\big)\cdot\sqrt{\max(1,\,m/n)}$
\State $W \gets (1 - \eta\,\lambda_{\mathrm{wd}})\,W$
\State $W \gets W - \eta\,(\Delta_{\mathrm{res}} + \alpha_{\mathrm{kl}}\,\Delta_{\mathrm{sub}})$
\Statex
\vspace{-6pt}
\Statex \emph{(3) Statistics update (raw $G$)}
\State $\Gtil \gets G\,U$;\quad $G_\perp \gets G - \Gtil\,U^\top$
\State $\mathrm{ldiag} \gets \mathrm{mean}\!\big([Q_L^\top\,\Gtil\,Q_S\,\Diag(\lambda_S^{\odot -1/2})]^{\odot 2},\; 1\big)$
\State $\mathrm{rdiag} \gets \mathrm{mean}\!\big([\Diag(\lambda_L^{\odot -1/2})\,Q_L^\top\,\Gtil\,Q_S]^{\odot 2},\; 0\big)$
\State $e^{\mathrm{res}}_i \gets (Q_L^\top\,G_\perp\,G_\perp^\top\,Q_L)_{ii}$ for all $i$
\State $\lambda_{L,i} \gets \beta_2\,\lambda_{L,i} + (1{-}\beta_2)\,(r\,\mathrm{ldiag}_i + \mu_\perp^{-1} e^{\mathrm{res}}_i)/n$ for all $i$
\State $\lambda_{S,j} \gets \beta_2\,\lambda_{S,j} + (1{-}\beta_2)\,\mathrm{rdiag}_j$ for all $j$
\State $\mu_\perp \gets \beta_2\,\mu_\perp + (1{-}\beta_2)\,\dfrac{1}{m(n-r)}\sum_{i=1}^{m} e^{\mathrm{res}}_i\,(\lambda_{L,i}^{-1/2})^2$
\State $L \gets \beta_2\,L + \dfrac{1-\beta_2}{n}\Big([\Gtil\,Q_S\,\Diag(\lambda_S^{\odot -1/2})][\Gtil\,Q_S\,\Diag(\lambda_S^{\odot -1/2})]^\top + \mu_\perp^{-1}\,G_\perp G_\perp^\top\Big)$
\State $S \gets \beta_2\,S + \dfrac{1-\beta_2}{m}\,[\Diag(\lambda_L^{\odot -1/2})\,Q_L^\top\,\Gtil]^\top[\Diag(\lambda_L^{\odot -1/2})\,Q_L^\top\,\Gtil]$
\Statex
\vspace{-6pt}
\Statex \emph{(4) Subspace tracking}
\State $U_{\mathrm{new}} \gets \mathrm{qr}\!\Big(\beta_2\,U\,S + \dfrac{1-\beta_2}{m}\,G^\top\,Q_L\,\Diag(\lambda_L^{\odot -1})\,Q_L^\top\,G\,U\Big)$
\State $T_{\mathrm{rot}} \gets U^\top\,U_{\mathrm{new}}$
\State $S \gets T_{\mathrm{rot}}^\top\,S\,T_{\mathrm{rot}}$;\quad $Q_S \gets T_{\mathrm{rot}}^\top\,Q_S$;\quad $U \gets U_{\mathrm{new}}$
\Statex
\vspace{-6pt}
\Statex \emph{(5) Periodic eigenbasis refresh (every $\tau$ steps)}\label{line:qr_refresh}
\State $Q_L \gets \qr(L\,Q_L)$;\quad $Q_S \gets \qr(S\,Q_S)$
\end{algorithmic}
\end{algorithm}

\textbf{Smok-Hop}~\eqref{eq:restr_update} is recovered by replacing $\Delta_{\mathrm{res}}$ with
\[
    \mu_\perp^{-1/2}\,Q_L\,\Diag(\lambda_L^{\odot -1/2})\,Q_L^\top\,\hat G_\perp,
\]
setting $\alpha_{\mathrm{kl}} = 1$ and using EMA momentum in place of Nesterov.

\paragraph{(P1) Nesterov momentum.}
The update direction uses the Nesterov-lookahead gradient $\hat G \coloneqq G + \mu M$ in place of the raw gradient $G$. The EMAs of $L, S, \mu_\perp$ continue to use the raw $G$.

\paragraph{(P2) Damping and clipping.}
We apply $L^{-1/2} = Q_L\,\Diag(\lambda_L^{\odot -1/2})\,Q_L^\top$ (similarly $S^{-1/2}$). Each elementwise $\lambda_i^{-1/2}$ is damped to $1/(\sqrt{\lambda_i} + \varepsilon)$ and additionally clipped from above by a dimension-aware ceiling (Eigenvalue clipping below). Algorithm~\ref{alg:practical} writes the un-damped, un-clipped form.

\paragraph{Eigenvalue clipping.}
After every per-eigenvalue EMA update we clip the inverse-square-roots:
\begin{equation}\label{eq:clamp_practical}
    \lambda_{L,i}^{-1/2} \gets \min\!\bigl(\lambda_{L,i}^{-1/2},\,C_L\bigr),\quad
    \lambda_{S,j}^{-1/2} \gets \min\!\bigl(\lambda_{S,j}^{-1/2},\,C_S\bigr),\quad
    \mu_\perp^{-1/2} \gets \min\!\bigl(\mu_\perp^{-1/2},\,C_\mu\bigr),
\end{equation}
with dimension-aware ceilings $C_L = \max(10, \min(m, 4000))$, $C_S = \max(10, \min(n, 4000))$, $C_\mu = \max(10, \min(\max(m,n), 4000))$, equivalently flooring $\lambda$ at $1/C^2$ where $C \coloneqq \max(C_L, C_S, C_\mu)$. The clip enforces the lower-bound assumption of Lemma~\ref{lem:clamp_lb} and provides numerical stability for the inverse-square-roots. It also fixes a scalar gauge: $L \otimes \hat R$ is invariant under $(L, \hat R) \mapsto (cL,\, c^{-1}\hat R)$ for any $c > 0$, so the KL stationarity system determines only the Kronecker product and not the individual factors; clipping all three preconditioners simultaneously breaks this rescaling freedom. We use $C \in [768, 4000]$.

\paragraph{(P3) Newton--Schulz orthogonalization.}
The exact polar factor in Algorithm~\ref{alg:ideal} is approximated by $T_{\mathrm{NS}}$ iterations of the Muon~\citep{modded_nanogpt_2024} polynomial $X \mapsto a\,X + b\,X X^\top X + c\,(X X^\top)^2 X$ with $(a,b,c) = (3.4445,-4.7750,2.0315)$, starting from $X/\|X\|_F$.

\vspace{10pt}

\begin{remark}[Left projection, $m > n$]\label{rem:left_proj}
When $m > n$, we project the left-side preconditioner $L$ rather than the right-side $R$. This is implemented by transposing all roles: the $r$-dimensional subspace $U \in \R^{m\times r}$ acts on the left, $\Gtil \coloneqq U^\top G \in \R^{r\times n}$, complement $G_\perp \coloneqq G - U\Gtil$, and the orthogonalized update becomes $\Delta_{\mathrm{res}} = \mathrm{NS}(\hat{G}_\perp\,R^{-1/2})\cdot\sqrt{\max(1,\,m/n)}$. The stationarity conditions and stability propositions hold by symmetry.
\end{remark}

\newpage

\section{Memory and computational cost}
\label{app:cost}

We compare the memory and per-step computational cost of Pro-KLShampoo and KL-Shampoo for one weight $W \in \R^{m \times n}$ with $m \leq n$, Newton--Schulz iteration count $T_{\mathrm{NS}}$, subspace rank $r \ll n$, and eigenbasis-refresh period $\tau$ (Table~\ref{tab:cost}, in the row layout of~\citet{lin2025understanding}). Memory rows give exact element counts; compute rows give leading-order matmul cost with constants suppressed.

\begin{table}[h]
\centering
\setlength{\tabcolsep}{6pt}
\caption{Memory and per-step computational cost for one weight $W \in \R^{m \times n}$, $m \leq n$, $r \ll n$.}
\label{tab:cost}
\begin{tabular}{lcc}
\toprule
& KL-Shampoo & Pro-KLShampoo \\
\midrule
\multicolumn{3}{l}{\emph{Memory (per parameter)}} \\
Kronecker / subspace factors ($L,R$ vs $L,S$) & $m^2 + n^2$ & $m^2 + r^2$ \\
Eigenbases ($Q_L, Q_R$ vs $Q_L, Q_S$) & $m^2 + n^2$ & $m^2 + r^2$ \\
Eigenvalues ($\lambda_L, \lambda_R$ vs $\lambda_L, \lambda_S$) & $m + n$ & $m + r$ \\
Subspace basis $U$ & N/A & $nr$ \\
Complement scalar $\mu_\perp$ & N/A & $1$ \\
Momentum buffer & $mn$ & $mn$ \\
\midrule
\multicolumn{3}{l}{\emph{Compute (per step, leading order)}} \\
Matmul (preconditioning $+$ statistics) & $mn^2$ & $T_{\mathrm{NS}}\,m^2 n + mnr$ \\
Eigenbasis QR (amortized over $\tau$ steps) & $(m^3 + n^3)/\tau$ & $(m^3 + r^3)/\tau$ \\
\bottomrule
\end{tabular}
\end{table}

\paragraph{Memory.}
The right-side $n \times n$ Kronecker factor $R$ and its eigenbasis $Q_R$ are replaced by the $r \times r$ pair $(S, Q_S)$, plus the $n \times r$ subspace basis $U$ and the scalar $\mu_\perp$; for $r \ll n$ the right-side state shrinks by a factor of about $n/r$.

\paragraph{Compute.}
KL-Shampoo's per-step cost is dominated by the right-side $mn^2$ matmul (applying $R^{-1/2}$) and the periodic $n^3$ QR refresh of $Q_R$. Pro-KLShampoo replaces the full preconditioning by a rank-$r$ Kronecker step on the projection and $T_{\mathrm{NS}}$ Newton--Schulz iterations on the $m \times n$ residual, and the right-side QR shrinks to $r^3$; the subspace projection and basis tracking add $mnr$ overhead (e.g., forming $\hat G U$). Newton--Schulz is dense matrix multiplication and runs efficiently on GPU, whereas QR is sequential and does not---so the $n^3 \to r^3$ QR shrink is the main per-step wallclock saving, partly offset by the subspace overhead when $n$ is moderate.

\newpage

\section{Wallclock comparison with Muon}
\label{app:muon-wallclock}

The main results report Pro-KLShampoo's wallclock saving relative to KL-Shampoo (\S\ref{sec:pretraining}). Here we report the wallclock saving relative to Muon~\citep{modded_nanogpt_2024} on LLaMA. Validation loss and memory against Muon at all four scales are reported in Table~\ref{tab:main}.

\begin{figure}[h]
\centering
\includegraphics[width=\linewidth]{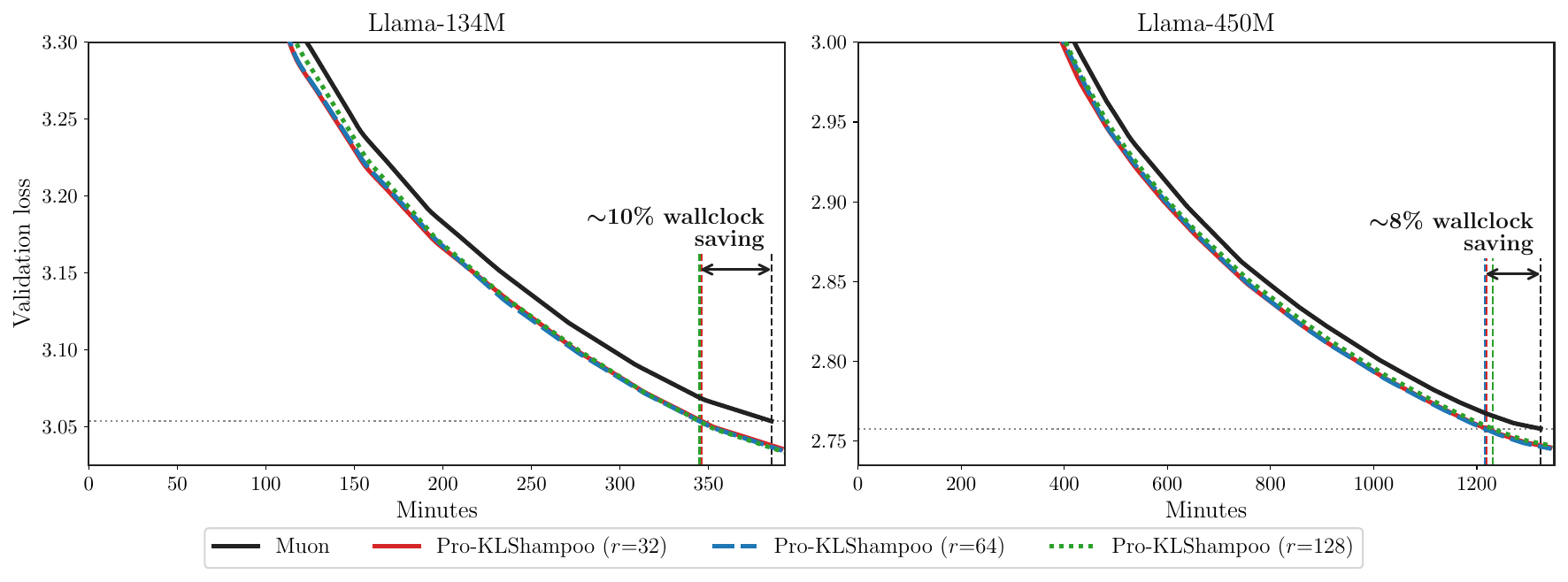}
\caption{Validation loss versus wallclock time for Muon and Pro-KLShampoo at $r \in \{32, 64, 128\}$ on LLaMA (134M, left; 450M, right). Pro-KLShampoo reaches Muon's final validation loss in approximately $10\%$ less wallclock time at 134M and $8\%$ less at 450M.}
\label{fig:llama-wallclock-muon}
\end{figure}

On LLaMA (Figure~\ref{fig:llama-wallclock-muon}), Pro-KLShampoo reaches Muon's final validation loss in less wallclock time at every rank: approximately $10\%$ less at 134M and $8\%$ less at 450M. On GPT-2, Pro-KLShampoo does not show a wallclock saving over Muon.

\section{Comparison with COSMOS}
\label{app:cosmos}

We provide a head-to-head comparison with COSMOS~\citep{liu2025cosmos}, which combines SOAP-style estimation inside a tracked subspace with Muon on the complement, at matched ranks across all four scales.

\paragraph{COSMOS configuration.}
We use the official COSMOS release with default hyperparameters (Nesterov momentum $0.95$, $5$-step Newton--Schulz, mixing weight $\gamma = 0.25$). Embedding and output weights are trained by AdamW, matching \S\ref{sec:experiments}. The subspace rank $r$ and learning rate are swept independently per scale (Appendix~\ref{app:sweeps}).

\paragraph{Results.}
Table~\ref{tab:cosmos-comparison} reports the head-to-head. At matched rank $r \in \{64, 128\}$, Pro-KLShampoo reaches lower validation loss than COSMOS at three of the four scales (GPT-2 124M, GPT-2 350M, LLaMA 134M), with margins of $0.004$--$0.026$. On LLaMA 450M the gap is small ($0.005$ in favor of COSMOS at $\mathrm{WD}{=}0$) and reverses under weight decay (below). Pro-KLShampoo also uses lower peak GPU memory at every entry (largest reduction $0.76$\,GiB on LLaMA 450M).

\begin{table}[h]
\centering
\caption{Head-to-head comparison with COSMOS at matched rank, across all four scales. Validation loss is the best across hyperparameter sweeps; memory (mem, in GiB) is the corresponding peak per-GPU memory of that run. Memory is comparable only within each scale. Pro-KLShampoo maintains the optimizer state in half precision, while COSMOS uses the full-precision optimizer state of its official release.}
\label{tab:cosmos-comparison}
\begin{tabular}{lccccccc}
\toprule
\multirow{2}{*}{Scale} & \multirow{2}{*}{$r$} & \multicolumn{2}{c}{COSMOS} & \multicolumn{2}{c}{Pro-KLShampoo} & \multicolumn{2}{c}{$\Delta$ (Pro $-$ COSMOS)} \\
\cmidrule(lr){3-4}\cmidrule(lr){5-6}\cmidrule(lr){7-8}
& & val loss & mem & val loss & mem & val loss & mem \\
\midrule
\multirow{2}{*}{GPT-2 124M}   & 64  & 3.3014 & 30.24 & 3.2754 & 30.16 & $-0.0260$ & $-0.08$ \\
                              & 128 & 3.2985 & 30.25 & 3.2745 & 30.17 & $-0.0240$ & $-0.08$ \\
\midrule
\multirow{2}{*}{GPT-2 350M}   & 64  & 3.0492 & 72.39 & 3.0348 & 72.24 & $-0.0144$ & $-0.15$ \\
                              & 128 & 3.0474 & 72.41 & 3.0344 & 72.25 & $-0.0130$ & $-0.16$ \\
\midrule
\multirow{2}{*}{LLaMA 134M}   & 64  & 3.0383 & 31.90 & 3.0345 & 31.75 & $-0.0038$ & $-0.15$ \\
                              & 128 & 3.0386 & 31.91 & 3.0335 & 31.76 & $-0.0051$ & $-0.15$ \\
\midrule
\multirow{2}{*}{LLaMA 450M}   & 64  & 2.7404 & 72.62 & 2.7453 & 71.86 & $+0.0049$ & $-0.76$ \\
                              & 128 & 2.7416 & 72.64 & 2.7469 & 71.89 & $+0.0053$ & $-0.75$ \\
\bottomrule
\end{tabular}
\end{table}

\paragraph{Weight decay ablation on LLaMA 450M.}
The LLaMA experiments above follow the GaLore convention of $\mathrm{WD}{=}0$ (\S\ref{sec:experiments}). We additionally evaluate both methods with weight decay on LLaMA 450M. COSMOS is swept over $\mathrm{WD} \in \{0.1, 0.3, 0.5\}$ at its previously-selected learning rate, reaching $2.7264$ at $r{=}64$. Pro-KLShampoo at $\mathrm{WD}{=}0.03$ (no further sweep) reaches $2.7142$ at $r{=}64$---a gap of $-0.0122$ over COSMOS.

\section{Hyperparameter sweeps}
\label{app:sweeps}

This appendix lists the hyperparameter sweep ranges and the selected values for every method and every model scale used in the experiments of \S\ref{sec:experiments}. The tuning protocol is that of \S\ref{sec:experiments}: AdamW is tuned first; its selected learning rate and weight decay are then fixed for the embedding and output layers under all matrix optimizers; the matrix-side hyperparameters of each method are then swept independently per method. For Pro-KLShampoo, the mixing weight $\alpha_{\mathrm{kl}}$ is swept jointly with the learning rate. On LLaMA, weight decay is fixed to $0$ throughout, following the GaLore convention~\citep{zhao2024galore}; on GPT-2, weight decay is swept jointly with the learning rate. Other hyperparameters follow the source releases of \citet{lin2025understanding,modded_nanogpt_2024,liu2025cosmos}: AdamW uses $(\beta_1, \beta_2) = (0.9, 0.95)$; Muon and Pro-KLShampoo use Nesterov momentum with coefficient $0.95$; KL-Shampoo uses momentum coefficient $0.95$; KL-Shampoo and Pro-KLShampoo use $\beta_2 = 0.95$ and refresh the preconditioner every $10$ steps; Muon, COSMOS, and Pro-KLShampoo use 5-step Newton--Schulz iteration for orthogonalization.

\begin{table}[H]
\centering
\small
\caption{Sweep ranges and selected values, GPT-2 124M on FineWeb-10B. ``--'' indicates the hyperparameter is not used by the method. Selected values are given in parentheses.}
\label{tab:sweep-gpt124m}
\resizebox{\textwidth}{!}{%
\begin{tabular}{lccc}
\toprule
Method & lr & wd & $\alpha_{\mathrm{kl}}$ \\
\midrule
AdamW         & $\{1\mathrm{e}{-3}, 2\mathrm{e}{-3}, 3\mathrm{e}{-3}\}$ ($2\mathrm{e}{-3}$)        & $\{2\mathrm{e}{-2}, 5\mathrm{e}{-2}, 1\mathrm{e}{-1}\}$ ($5\mathrm{e}{-2}$)              & --  \\
Muon          & $\{1\mathrm{e}{-2}, 2\mathrm{e}{-2}, 3\mathrm{e}{-2}\}$ ($2\mathrm{e}{-2}$)            & $\{0, 2\mathrm{e}{-2}, 3\mathrm{e}{-2}, 5\mathrm{e}{-2}, 1\mathrm{e}{-1}\}$ ($2\mathrm{e}{-2}$)     & --  \\
KL-Shampoo    & $\{1\mathrm{e}{-3}, 2\mathrm{e}{-3}, 3\mathrm{e}{-3}\}$ ($2\mathrm{e}{-3}$)        & $\{5\mathrm{e}{-2}, 1\mathrm{e}{-1}, 1.5\mathrm{e}{-1}, 2\mathrm{e}{-1}, 2.5\mathrm{e}{-1}, 3\mathrm{e}{-1}\}$ ($2.5\mathrm{e}{-1}$) & --  \\
COSMOS ($r{=}64$)  & $\{2\mathrm{e}{-4}, 5\mathrm{e}{-4}, 1\mathrm{e}{-3}, 2\mathrm{e}{-3}\}$ ($5\mathrm{e}{-4}$) & $\{1\mathrm{e}{-1}, 2.5\mathrm{e}{-1}, 5\mathrm{e}{-1}\}$ ($2.5\mathrm{e}{-1}$) & --  \\
COSMOS ($r{=}128$) & $\{2\mathrm{e}{-4}, 5\mathrm{e}{-4}, 1\mathrm{e}{-3}, 2\mathrm{e}{-3}\}$ ($5\mathrm{e}{-4}$) & $\{1\mathrm{e}{-1}, 2.5\mathrm{e}{-1}, 5\mathrm{e}{-1}\}$ ($2.5\mathrm{e}{-1}$) & --  \\
Pro-KLShampoo ($r{=}32$)   & $\{1\mathrm{e}{-2}, 2\mathrm{e}{-2}, 3\mathrm{e}{-2}\}$ ($2\mathrm{e}{-2}$) & $\{1\mathrm{e}{-2}, 2\mathrm{e}{-2}, 3\mathrm{e}{-2}, 5\mathrm{e}{-2}, 1\mathrm{e}{-1}\}$ ($3\mathrm{e}{-2}$) & $\{5\mathrm{e}{-3}, 1\mathrm{e}{-2}, 1.5\mathrm{e}{-2}\}$ ($1\mathrm{e}{-2}$) \\
Pro-KLShampoo ($r{=}64$)   & $\{1\mathrm{e}{-2}, 2\mathrm{e}{-2}, 3\mathrm{e}{-2}\}$ ($2\mathrm{e}{-2}$) & $\{1\mathrm{e}{-2}, 2\mathrm{e}{-2}, 3\mathrm{e}{-2}, 5\mathrm{e}{-2}, 1\mathrm{e}{-1}\}$ ($3\mathrm{e}{-2}$) & $\{5\mathrm{e}{-3}, 1\mathrm{e}{-2}, 1.5\mathrm{e}{-2}\}$ ($1\mathrm{e}{-2}$) \\
Pro-KLShampoo ($r{=}128$)  & $\{1\mathrm{e}{-2}, 2\mathrm{e}{-2}, 3\mathrm{e}{-2}\}$ ($2\mathrm{e}{-2}$) & $\{1\mathrm{e}{-2}, 2\mathrm{e}{-2}, 3\mathrm{e}{-2}, 5\mathrm{e}{-2}, 1\mathrm{e}{-1}\}$ ($3\mathrm{e}{-2}$) & $\{5\mathrm{e}{-3}, 1\mathrm{e}{-2}, 1.5\mathrm{e}{-2}\}$ ($1\mathrm{e}{-2}$) \\
Smok-Hop ($r{=}128$) & uses KL-Shampoo's selected lr ($2\mathrm{e}{-3}$)        & $2.5\mathrm{e}{-1}$  & --  \\
\bottomrule
\end{tabular}%
}
\end{table}

\begin{table}[H]
\centering
\small
\caption{Sweep ranges and selected values, GPT-2 350M on FineWeb-10B. Selected values are given in parentheses.}
\label{tab:sweep-gpt350m}
\resizebox{\textwidth}{!}{%
\begin{tabular}{lccc}
\toprule
Method & lr & wd & $\alpha_{\mathrm{kl}}$ \\
\midrule
AdamW         & $\{1\mathrm{e}{-3}, 2\mathrm{e}{-3}, 3\mathrm{e}{-3}\}$ ($2\mathrm{e}{-3}$)         & $\{5\mathrm{e}{-2}, 1\mathrm{e}{-1}, 1.5\mathrm{e}{-1}\}$ ($5\mathrm{e}{-2}$)              & --  \\
Muon          & $\{1\mathrm{e}{-2}, 2\mathrm{e}{-2}, 3\mathrm{e}{-2}, 5\mathrm{e}{-2}\}$ ($2\mathrm{e}{-2}$)       & $\{0, 2\mathrm{e}{-2}, 3\mathrm{e}{-2}, 5\mathrm{e}{-2}, 1\mathrm{e}{-1}\}$ ($3\mathrm{e}{-2}$)     & --  \\
KL-Shampoo    & $\{1\mathrm{e}{-3}, 2\mathrm{e}{-3}, 3\mathrm{e}{-3}\}$ ($2\mathrm{e}{-3}$)         & $\{5\mathrm{e}{-2}, 1\mathrm{e}{-1}, 1.5\mathrm{e}{-1}, 2\mathrm{e}{-1}, 2.5\mathrm{e}{-1}, 3\mathrm{e}{-1}\}$ ($1.5\mathrm{e}{-1}$) & --  \\
COSMOS ($r{=}64$)  & $\{2\mathrm{e}{-4}, 5\mathrm{e}{-4}, 1\mathrm{e}{-3}, 2\mathrm{e}{-3}\}$ ($1\mathrm{e}{-3}$) & $\{2\mathrm{e}{-1}, 3\mathrm{e}{-1}, 5\mathrm{e}{-1}\}$ ($5\mathrm{e}{-1}$) & --  \\
COSMOS ($r{=}128$) & $\{2\mathrm{e}{-4}, 5\mathrm{e}{-4}, 1\mathrm{e}{-3}, 2\mathrm{e}{-3}\}$ ($1\mathrm{e}{-3}$) & $\{2\mathrm{e}{-1}, 3\mathrm{e}{-1}, 5\mathrm{e}{-1}\}$ ($5\mathrm{e}{-1}$) & --  \\
Pro-KLShampoo ($r{=}32$)  & $\{1\mathrm{e}{-2}, 2\mathrm{e}{-2}, 3\mathrm{e}{-2}\}$ ($2\mathrm{e}{-2}$)  & $\{1\mathrm{e}{-2}, 2\mathrm{e}{-2}, 3\mathrm{e}{-2}, 5\mathrm{e}{-2}, 1\mathrm{e}{-1}\}$ ($3\mathrm{e}{-2}$) & $\{5\mathrm{e}{-3}, 1\mathrm{e}{-2}, 1.5\mathrm{e}{-2}\}$ ($1\mathrm{e}{-2}$) \\
Pro-KLShampoo ($r{=}64$)  & $\{1\mathrm{e}{-2}, 2\mathrm{e}{-2}, 3\mathrm{e}{-2}\}$ ($2\mathrm{e}{-2}$)  & $\{1\mathrm{e}{-2}, 2\mathrm{e}{-2}, 3\mathrm{e}{-2}, 5\mathrm{e}{-2}, 1\mathrm{e}{-1}\}$ ($3\mathrm{e}{-2}$) & $\{5\mathrm{e}{-3}, 1\mathrm{e}{-2}, 1.5\mathrm{e}{-2}\}$ ($1\mathrm{e}{-2}$) \\
Pro-KLShampoo ($r{=}128$) & $\{1\mathrm{e}{-2}, 2\mathrm{e}{-2}, 3\mathrm{e}{-2}\}$ ($2\mathrm{e}{-2}$)  & $\{1\mathrm{e}{-2}, 2\mathrm{e}{-2}, 3\mathrm{e}{-2}, 5\mathrm{e}{-2}, 1\mathrm{e}{-1}\}$ ($3\mathrm{e}{-2}$) & $\{5\mathrm{e}{-3}, 1\mathrm{e}{-2}, 1.5\mathrm{e}{-2}\}$ ($1\mathrm{e}{-2}$) \\
\bottomrule
\end{tabular}%
}
\end{table}

\begin{table}[H]
\centering
\small
\caption{Sweep ranges and selected values, LLaMA 134M on C4. Weight decay is $0$ for all methods, following the GaLore convention. Selected values are given in parentheses.}
\label{tab:sweep-llama134m}
\begin{tabular}{lcc}
\toprule
Method & lr & $\alpha_{\mathrm{kl}}$ \\
\midrule
AdamW         & $\{1\mathrm{e}{-3}, 2\mathrm{e}{-3}, 3\mathrm{e}{-3}, 5\mathrm{e}{-3}\}$ ($3\mathrm{e}{-3}$)        & --  \\
Muon          & $\{5\mathrm{e}{-3}, 1\mathrm{e}{-2}, 2\mathrm{e}{-2}, 3\mathrm{e}{-2}\}$ ($2\mathrm{e}{-2}$)            & --  \\
KL-Shampoo    & $\{2\mathrm{e}{-3}, 3\mathrm{e}{-3}, 4\mathrm{e}{-3}, 5\mathrm{e}{-3}, 8\mathrm{e}{-3}\}$ ($5\mathrm{e}{-3}$) & --  \\
COSMOS ($r{=}64$)             & $\{2\mathrm{e}{-4}, 5\mathrm{e}{-4}, 1\mathrm{e}{-3}, 2\mathrm{e}{-3}\}$ ($1\mathrm{e}{-3}$) & -- \\
COSMOS ($r{=}128$)            & $\{2\mathrm{e}{-4}, 5\mathrm{e}{-4}, 1\mathrm{e}{-3}, 2\mathrm{e}{-3}\}$ ($1\mathrm{e}{-3}$) & -- \\
Pro-KLShampoo ($r{=}32$)      & $\{1\mathrm{e}{-2}, 2\mathrm{e}{-2}, 3\mathrm{e}{-2}\}$ ($2\mathrm{e}{-2}$)               & $\{5\mathrm{e}{-3}, 1\mathrm{e}{-2}, 1.5\mathrm{e}{-2}\}$ ($1\mathrm{e}{-2}$) \\
Pro-KLShampoo ($r{=}64$)      & $\{1\mathrm{e}{-2}, 2\mathrm{e}{-2}, 3\mathrm{e}{-2}\}$ ($2\mathrm{e}{-2}$)               & $\{5\mathrm{e}{-3}, 1\mathrm{e}{-2}, 1.5\mathrm{e}{-2}\}$ ($1\mathrm{e}{-2}$) \\
Pro-KLShampoo ($r{=}128$)     & $\{1\mathrm{e}{-2}, 2\mathrm{e}{-2}, 3\mathrm{e}{-2}\}$ ($3\mathrm{e}{-2}$)               & $\{5\mathrm{e}{-3}, 1\mathrm{e}{-2}, 1.5\mathrm{e}{-2}\}$ ($1\mathrm{e}{-2}$) \\
Smok-Hop ($r{=}128$) & $\{2\mathrm{e}{-3}, 3\mathrm{e}{-3}, 5\mathrm{e}{-3}\}$ ($3\mathrm{e}{-3}$)               & --  \\
\bottomrule
\end{tabular}
\end{table}

The subspace-only and complement-only variants on LLaMA 134M use the selected hyperparameters of Pro-KLShampoo at $r{=}128$ with the learning rate scaled by $1.5$ as on GPT-2 124M (same rationale).

\begin{table}[H]
\centering
\small
\caption{Sweep ranges and selected values, LLaMA 450M on C4. Weight decay is $0$ for all methods. Selected values are given in parentheses.}
\label{tab:sweep-llama450m}
\begin{tabular}{lcc}
\toprule
Method & lr & $\alpha_{\mathrm{kl}}$ \\
\midrule
AdamW         & $\{2\mathrm{e}{-3}, 3\mathrm{e}{-3}, 5\mathrm{e}{-3}\}$ ($2\mathrm{e}{-3}$)        & --  \\
Muon          & $\{2\mathrm{e}{-2}, 3\mathrm{e}{-2}, 5\mathrm{e}{-2}\}$ ($2\mathrm{e}{-2}$)            & --  \\
KL-Shampoo    & $\{2\mathrm{e}{-3}, 3\mathrm{e}{-3}, 5\mathrm{e}{-3}\}$ ($3\mathrm{e}{-3}$)        & --  \\
COSMOS ($r{=}64$)             & $\{2\mathrm{e}{-4}, 5\mathrm{e}{-4}, 1\mathrm{e}{-3}, 2\mathrm{e}{-3}\}$ ($5\mathrm{e}{-4}$) & -- \\
COSMOS ($r{=}128$)            & $\{2\mathrm{e}{-4}, 5\mathrm{e}{-4}, 1\mathrm{e}{-3}, 2\mathrm{e}{-3}\}$ ($5\mathrm{e}{-4}$) & -- \\
Pro-KLShampoo ($r{=}32$)      & $\{2\mathrm{e}{-2}, 3\mathrm{e}{-2}\}$ ($2\mathrm{e}{-2}$)                      & $\{5\mathrm{e}{-3}, 1\mathrm{e}{-2}\}$ ($5\mathrm{e}{-3}$) \\
Pro-KLShampoo ($r{=}64$)      & $\{2\mathrm{e}{-2}, 3\mathrm{e}{-2}\}$ ($2\mathrm{e}{-2}$)                      & $\{5\mathrm{e}{-3}, 1\mathrm{e}{-2}\}$ ($5\mathrm{e}{-3}$) \\
Pro-KLShampoo ($r{=}128$)     & $\{2\mathrm{e}{-2}, 3\mathrm{e}{-2}\}$ ($2\mathrm{e}{-2}$)                      & $\{5\mathrm{e}{-3}, 1\mathrm{e}{-2}\}$ ($5\mathrm{e}{-3}$) \\
Smok-Hop ($r{=}128$) & uses KL-Shampoo's selected lr ($3\mathrm{e}{-3}$)        & --  \\
\bottomrule
\end{tabular}
\end{table}

\paragraph{Ablation hyperparameters.}
The subspace-only variant and the complement-only variant reuse the selected hyperparameters from the corresponding Pro-KLShampoo $r{=}128$ runs, except with a learning rate scaled by $1.5$ to compensate for the removed component (rationale: under $\alpha_{\mathrm{kl}}$ calibration the subspace and complement update components have approximately matched operator norms; removing one halves the squared-Frobenius norm of the update, and $\sqrt{2} \approx 1.5$ rescales it back).

\end{document}